\documentclass{aims_preprint} 
\pdfoutput=1 
\usepackage{amsmath}
\usepackage{paralist}
\usepackage[misc]{ifsym}
\usepackage{epsfig} 
\usepackage{epstopdf} 

\usepackage[colorlinks,pagebackref,bookmarksdepth=3]{hyperref}
\renewcommand*{\backref}[1]{\ifx#1\relax \else Page #1 \fi}
\renewcommand*{\backrefalt}[4]{%
  \ifcase #1 \footnotesize{(not cited)}%
  \or        \footnotesize{(cited on p.~#2)}%
  \else      \footnotesize{(cited on pp.~#2)}%
  \fi
}
\hypersetup{urlcolor=blue, citecolor=red}
\allowdisplaybreaks

\def\currentvolume{X}
 \def\currentissue{X}
  \def\currentyear{200X}
   \def\currentmonth{XX}
    \def\ppages{X-XX}
     \def\DOI{10.3934/xx.xxxxxxx}

\usepackage{bm}
\usepackage{lipsum}
\usepackage{amsfonts}
\usepackage{graphicx}
\usepackage{algorithmic}
\usepackage{amsopn}

\usepackage{amssymb} 
\usepackage{subcaption} 
\usepackage{booktabs} 
\usepackage{mathtools} 
\usepackage{bbm} 
\usepackage{mathrsfs} 
\usepackage{stmaryrd} 
\usepackage{thm-restate} 
\let\originalleft\left 
\let\originalright\right
\renewcommand{\left}{\mathopen{}\mathclose\bgroup\originalleft}
\renewcommand{\right}{\aftergroup\egroup\originalright}
\usepackage{array}
\newcolumntype{M}[1]{>{\centering\arraybackslash}m{#1}} 

\newcommand{\N}{\mathbb{N}}
\newcommand{\Z}{\mathbb{Z}}
\newcommand{\R}{\mathbb{R}}
\newcommand{\T}{\mathbb{T}}
\newcommand{\C}{\mathbb{C}}

\newcommand{\cA}{\mathcal{A}}
\newcommand{\cL}{\mathcal{L}}
\newcommand{\cF}{\mathcal{F}}

\newcommand{\cH}{\mathcal{H}}
\newcommand{\cX}{\mathcal{X}}
\newcommand{\cB}{\mathcal{B}}
\newcommand{\cY}{\mathcal{Y}}
\newcommand{\cC}{\mathcal{C}}
\newcommand{\cS}{\mathcal{S}}
\newcommand{\cG}{\mathcal{G}}

\newcommand{\cR}{\mathcal{R}}
\newcommand{\cE}{\mathcal{E}}
\newcommand{\cD}{\mathcal{D}}
\newcommand{\cI}{\mathcal{I}}

\newcommand{\cK}{\mathcal{K}}
\newcommand{\cU}{\mathcal{U}}
\newcommand{\cV}{\mathcal{V}}

\newcommand{\cQ}{\mathcal{Q}}

\newcommand{\mcU}{\mathcal{U}}
\newcommand{\mcD}{\mathcal{D}}
\newcommand{\mcQ}{\mathcal{Q}}
\newcommand{\mcS}{\mathcal{S}}
\newcommand{\mcY}{\mathcal{Y}}

\newcommand{\sL}{\mathscr{L}}

\newcommand{\sR}{\mathscr{R}}

\newcommand{\sG}{\mathscr{G}}

\newcommand{\sD}{\mathscr{D}}
\newcommand{\sW}{\mathscr{W}}
\newcommand{\sfL}{\mathsf{L}}
\newcommand{\sfE}{\mathsf{E}}

\newcommand{\Ld}{L^{\dagger}} 
\newcommand{\ld}{l^{\dagger}} 
\newcommand{\fd}{f^{\dagger}}

\newcommand{\qd}{q^{\dagger}} 

\DeclarePairedDelimiterX{\iptemp}[2]{\langle}{\rangle}{#1, #2}
\newcommand{\ip}{\iptemp}
\DeclarePairedDelimiterX{\normtemp}[1]{\lVert}{\rVert}{#1}
\newcommand{\norm}{\normtemp}
\DeclarePairedDelimiterX{\abstemp}[1]{\lvert}{\rvert}{#1}
\newcommand{\abs}{\abstemp}
\DeclarePairedDelimiterX{\trtemp}[1]{(}{)}{#1}
\newcommand{\tr}{\operatorname{tr}\trtemp}
\DeclarePairedDelimiterX{\SEtemp}[2]{(}{)}{#1, #2}
\newcommand{\SE}{\operatorname{\mathsf{SE}}\SEtemp}
\DeclarePairedDelimiterX{\SGtemp}[1]{(}{)}{#1}
\newcommand{\SG}{\operatorname{\mathsf{SG}}\SGtemp}
\DeclareRobustCommand{\rchi}{{\mathpalette\irchi\relax}}
\newcommand{\irchi}[2]{\raisebox{\depth}{$#1\chi$}} 

\newcommand{\defeq}{\coloneqq} 
\newcommand{\eqdef}{\eqqcolon} 
\newcommand{\condbar}{\, \vert \,}
\newcommand{\lcondbar}{\, \big\vert \,}

\DeclarePairedDelimiterX{\floor}[1]{\lfloor}{\rfloor}{#1} 
\DeclarePairedDelimiterX{\ceil}[1]{\lceil}{\rceil}{#1} 

\newcommand{\tp}{\top} 
\DeclareMathOperator{\Id}{Id} 
\newcommand{\HS}{\mathrm{HS}} 
\DeclareMathOperator{\image}{Im} 

\renewcommand{\P}{\operatorname{\mathbb{P}}} 
\newcommand{\E}{\operatorname{\mathbb{E}}} 
\newcommand{\comp}{\textsf{c}} 
\newcommand{\Cov}{\operatorname{Cov}} 
\newcommand{\one}{\mathbbm{1}} 
\newcommand{\diid}{\stackrel{\mathrm{i.i.d.}}{\sim}} 
\newcommand{\normal}{\mathcal{N}} 

\newcommand{\iunit}{\mathrm{i}}
\newcommand{\diff}{d}

\newcommand{\dd}[1]{\,\diff{#1}}

\newcommand{\lap}{\Delta} 
\newcommand{\onebm}{\bm{1}} 

\newcommand{\al}{\alpha}
\newcommand{\ep}{\varepsilon}

\newcommand{\eps}{\epsilon}
\newcommand{\p}{\partial}
\newcommand{\Td}{\mathbb{T}^d}

\def\qfa{\quad\text{for all}\quad}

\def\qas{\quad\text{as}\quad}
\def\qa{\quad\text{and}\quad}

\def\qw{\quad\text{where}\quad}
\def\qin{\quad\text{in}\quad}
\def\qon{\quad\text{on}\quad}

\newcommand{\sfit}[1]{\textsf{\small{#1}}} 

\newcommand{\set}[2]{{\left\{ #1 \,\middle|\, #2 \right\}}}
\newcommand{\slot}{{\,\cdot\,}}

\ifpdf
\DeclareGraphicsExtensions{.eps,.pdf,.png,.jpg}
\else
\DeclareGraphicsExtensions{.eps}
\fi
\graphicspath{ {figures/} }

\usepackage{enumitem}
\setlist[enumerate]{leftmargin=.5in}
\setlist[itemize]{leftmargin=.5in}
\makeatletter
\newcommand{\myitem}[1]{%
\item[#1]\protected@edef\@currentlabel{#1}%
}
\makeatother

\usepackage[normalem]{ulem} 
\usepackage{xcolor}
\usepackage[textsize=tiny]{todonotes}

\newtheorem{theorem}{Theorem}[section]

\newtheorem{lemma}[theorem]{Lemma}
\newtheorem{proposition}[theorem]{Proposition}

\theoremstyle{definition}
\newtheorem{definition}[theorem]{Definition}
\newtheorem{remark}[theorem]{Remark}
\newtheorem{assumption}[theorem]{Assumption}

\title[Operator learning for parameter-to-observable maps]
{An operator learning perspective on parameter-to-observable maps} 

\author[Daniel Zhengyu Huang, Nicholas H. Nelsen, and Margaret Trautner]{}


\thanks{$^*$Corresponding author: Nicholas H. Nelsen (\texttt{nnelsen@caltech.edu})}

\ifpdf
\hypersetup{
  pdftitle={An operator learning perspective on parameter-to-observable maps},
  pdfauthor={D. Z. Huang, N. H. Nelsen, and M. Trautner}
}
\fi

\begin{document}
\maketitle

\centerline{\scshape
Daniel Zhengyu Huang,$^{{\href{mailto:huangdz@bicmr.pku.edu.cn}{\textrm{\Letter}}}\,1}$
Nicholas H. Nelsen,$^{{\href{mailto:nnelsen@caltech.edu}{\textrm{\Letter}}}\,*2}$
and Margaret Trautner$^{{\href{mailto:trautner@caltech.edu}{\textrm{\Letter}}}\,2}$
}

\medskip

{\footnotesize
 \centerline{$^1$Beijing International Center for Mathematical Research,
Peking University, China}
} 

\medskip

{\footnotesize
 \centerline{$^2$Computing and Mathematical Sciences, California Institute of Technology, USA}
}

\bigskip


\begin{abstract} 
Computationally efficient surrogates for parametrized physical models play a crucial role in science and engineering. Operator learning provides data-driven surrogates that map between function spaces. However, instead of full-field measurements, often the available data are only finite-dimensional parametrizations of model inputs or finite observables of model outputs. Building on Fourier Neural Operators, this paper introduces the Fourier Neural Mappings (FNMs) framework that is able to accommodate such finite-dimensional vector inputs or outputs. The paper develops universal approximation theorems for the method. Moreover, in many applications the underlying parameter-to-observable (PtO) map is defined implicitly through an infinite-dimensional operator, such as the solution operator of a partial differential equation. A natural question is whether it is more data-efficient to learn the PtO map end-to-end or first learn the solution operator and subsequently compute the observable from the full-field solution. A theoretical analysis of Bayesian nonparametric regression of linear functionals, which is of independent interest, suggests that the end-to-end approach can actually have worse sample complexity. Extending beyond the theory, numerical results for the FNM approximation of three nonlinear PtO maps demonstrate the benefits of the operator learning perspective that this paper adopts.
\end{abstract}


\section{Introduction}
Operator learning has emerged as a methodology that enables the machine learning of maps between spaces of functions. Many surrogate modeling tasks in areas such as uncertainty quantification, inverse problems, and design optimization involve a map between function spaces, such as the solution operator of a partial differential equation (PDE). However, the primary quantities of interest (QoI) in these tasks are usually just a finite number of design parameters or output observables. This may be because full-field data, such as initial conditions, boundary conditions, and solutions of PDEs, are not accessible from measurements or are too expensive to acquire. The prevailing approach then involves emulating the parameter-to-observable (PtO) map instead of the underlying solution map between function spaces. Yet, it is natural to wonder if the success of operator learning in the function-to-function setting can be brought to bear in this more realistic setting where inputs or outputs may necessarily be finite-dimensional vectors.
To this end, the present paper introduces Fourier Neural Mappings (FNMs) as a way to extend operator learning architectures such as the Fourier Neural Operator (FNO) to finite-dimensional input and output spaces in a manner compatible with the underlying operator between infinite-dimensional spaces. The admissible types of FNM models considered in this work are visualized in Figure~\ref{fig:ee_vs_ff}.

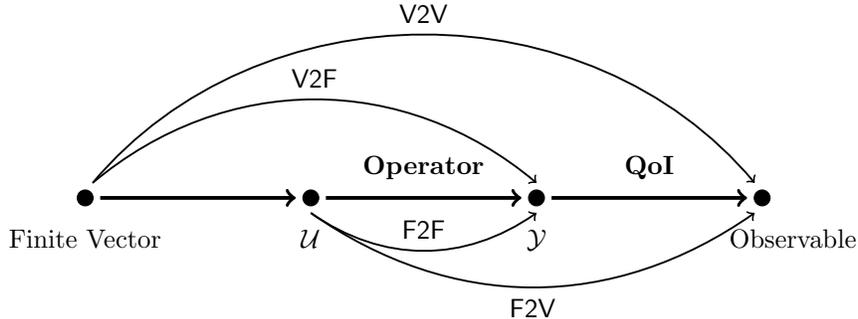
\begin{figure}[tb]
    \centering
    \begin{tikzpicture}
        \foreach \x/\label in {-3/Finite Vector, 0/$\mathcal{U}$, 3/$\mathcal{Y}$, 6/$\quad\quad$ Observable}
            \filldraw (\x, 0) circle (3pt) node[below=3mm] {\label};
            
        \draw[->,line width=0.5mm] (-2.8, 0) -- node[above=1mm] {} (-0.2, 0);
        \draw[->,line width=0.5mm] (0.2, 0) -- node[above=1mm] {\textbf{Operator}} (2.8, 0);
        \draw[->, line width=0.5mm] (3.2, 0) -- node[above=1mm] {\textbf{QoI}} (5.8, 0);
          
        \draw[->, line width=0.25mm] (-2.9, 0.2) to[bend left=50] node[above=0.25mm] {$\mathsf{V2V}$} (5.9, 0.2);
        \draw[->, line width=0.25mm] (-2.9, 0.2) to[bend left=40] node[above=0.25mm] {$\mathsf{V2F}$} (3, 0.2);
        \draw[->, line width=0.25mm] (0, -0.2)to[bend left=-35] node[below=0.25mm] {$\mathsf{F2V}$} (5.9, -0.2);
        \draw[->, line width=0.25mm] (0, -0.2) to[bend left=-35] node[above=0.25mm] {$\mathsf{F2F}$} (3, -0.2);
    \end{tikzpicture}
    \caption{Illustration of the factorization of an underlying PtO map into a QoI and an operator between function spaces. Also shown are the four variants of input and output representations considered in this work. Here, $\mathcal{U}$ is an input function space and $\mathcal{Y}$ is an intermediate function space.}
    \label{fig:ee_vs_ff}
\end{figure}

Nevertheless, it is possible to accommodate finite-dimensional inputs or outputs through other means. For instance, one could lift a finite-dimensional input vector to a function by expanding in predetermined basis functions, apply traditional operator learning architectures to the full-field function space data, and then directly compute a known finite-dimensional QoI from the output function.
In contrast, the end-to-end FNM approach in this work is fully data-driven and operates directly on finite-dimensional vector data without the need for pre- and postprocessing.
A natural question is whether one of these two approaches achieves better accuracy than the other when the goal is to predict certain QoIs. In the present paper, we address this important question both from a theoretical and a numerical perspective. Indeed, it has been empirically observed in various nonlinear problems ranging from electronic structure calculations~\cite{teh2021machine} to metamaterial design~\cite{bastek2023inverse} that data-driven methods that predict the full-field response of a system are superior to end-to-end approaches for the same downstream tasks or QoIs. The analysis in this paper takes the first steps toward a rigorous theoretical understanding of these empirical findings.

Throughout the paper, we refer to learning a function-valued map as \emph{full-field learning}. Given such a learned map, various known QoIs may be directly computed from the output of the map. On the other hand, we refer to the direct estimation of the map from an input to the observed QoI as \emph{end-to-end learning}. This terminology distinguishes between output spaces. When either the input or output is a finite vector and the other is infinite-dimensional, we label the learning approach as ``vector-to-function'' (V2F) or ``function-to-vector" (F2V), respectively, to avoid ambiguity. The abbreviations V2V and F2F for ``vector-to-vector'' and ``function-to-function'' are analogous.

\subsection{Contributions}\label{ssec:contr}
In this paper, we make the following contributions.
\begin{enumerate}[label=(C\arabic*),topsep=1.67ex,itemsep=0.5ex,partopsep=1ex,parsep=1ex]
    \item \label{item:contr_fnm} We introduce FNMs as a function space architecture that is able to accommodate finite-dimensional vector inputs, outputs, or both.
    \item \label{item:contr_ua} We prove universal approximation theorems for FNMs.
    \item \label{item:contr_linear} We establish convergence rates for Bayesian nonparametric regression of linear functionals under smoothness misspecification; as a byproduct of this analysis, we prove that full-field learning of linear functionals that are factorized into the composition of a linear QoI and a linear operator enjoys better sample complexity than end-to-end estimators in certain regimes.
    \item \label{item:contr_numerics} We perform numerical experiments with FNMs in three different examples---an advection--diffusion equation, flow over an airfoil, and an elliptic homogenization problem---that show empirical evidence that the theoretical linear insight from Contribution~\ref{item:contr_linear} remains valid for nonlinear maps.
\end{enumerate}

Next, we provide an overview of related work in the literature in Subsection \ref{ssec:lit}. Subsection \ref{ssec:notation} contains relevant notation, and Subsection \ref{ssec:outline} gives an outline of the remainder of the paper. 

\subsection{Related work}\label{ssec:lit}
Several works have established neural operators as a viable tool for scientific machine learning. The general neural operator formalism is described in \cite{kovachki2023neural} and contains several subclasses including DeepONet \cite{lu2021learning}, graph neural operator \cite{li2020neural,li2020multipole}, and FNO \cite{li2021fourier}. These architectures allow for function data evaluated at different grid points or resolutions to all be used with the same model. In particular, the FNO is primarily parametrized in Fourier space. It exploits the fact that the Fourier 
basis spans $L^2$ and uses the efficient Fast Fourier Transform (FFT) algorithm for computations. The idea of parametrizing operators in Fourier space is explored in earlier works as well \cite{nelsen2021random,patel2021physics}. The FNO has been shown to be applicable both to domains other than the torus and to nonuniform meshes \cite{li2022fourier,lingsch2023vandermonde}. These neural operators have been used in various areas of application, including climate modeling \cite{kurth2023fourcastnet}, fracture mechanics \cite{goswami2022physics}, and catheter design \cite{zhou2024ai}. In several of these applications, 
neural operators have been implemented with finite-dimensional vector inputs or outputs by using constant functions as substitutes for finite vectors, which is theoretically justified by statements of universal approximation \cite{bhattacharya2023learning}, or by using other hand-designed maps. However, learning a constant function as a representation for a constant is unnatural and computationally wasteful; it is desirable to substitute a more suitable architecture. The present paper develops FNMs that extend neural operators to this important setting while retaining desirable universal approximation properties.

The theory of scientific machine learning falls broadly into three tiers. In the first tier, 
universal approximation 
results~\cite{chen1995universal,cybenko1989approximation,hornik1991approximation,weinan2019barron} use classical approximation theory to guarantee that the architecture is capable of representing maps from within a class of interest to any desired accuracy. 
Some of the proofs of these results contain constructive arguments, but the corresponding architectures are usually not as empirically effective as those that solely come with existence results. Examples of constructive arguments for operator approximation are contained in \cite{herrmann2022constructive}, which constructs ReLU neural networks, and \cite{boulle2023elliptic}, which uses randomized numerical linear algebra to sketch Green's functions for linear elliptic PDEs. Each of these works also comes equipped with convergence rates with respect to model size and data size, respectively; these rates form the second tier of operator learning theory. Many papers in this tier prove bounds on the required model size, i.e., parameter complexity \cite{deng2022approximation,herrmann2022neural,kovachki2021universal,lanthaler2023operator,lanthaler2022nonlinear,lanthaler2023curse,lu2021deep,shen2019deep}. Some are able to obtain sample complexity bounds, although most results are restricted to linear or kernelized settings~\cite{caponnetto2007optimal,de2023convergence,jin2022minimax,lanthaler2023error,li2023towards,liu2024deep,mollenhauer2022learning,schafer2021sparse,stepaniants2023learning}.
The third tier of theory describes the likelihood of actually obtaining an accurate approximation through optimization. While some results
along these lines exist for linear models, linear maps, and constructive operators \cite{lanthaler2023error,laurent2018deep,nelsen2024operator}, they are absent for the class of neural operators optimized through variants of stochastic gradient descent (SGD). This is the class that has proven empirically most effective in applications thus far and is the class used in this work.

To provide theoretical intuition for nonlinear settings, this paper establishes convergence rates in the tractable setting of learning a linear functional---the PtO map---from noisy data. The setting of functional linear regression has a long history in statistics \cite{cai2006prediction,cardot2007clt,reimherr2015functional,yuan2010reproducing}. A zoo of different estimators exist, e.g., those based on reproducing kernel Hilbert space (RKHS) methods, principal component analysis, and wavelets. We target the frequentist convergence properties of a linear Gaussian posterior estimator that arises from reinterpreting the regression task as a Bayesian inverse problem. The closest work to ours is \cite{lian2016posterior}. There, the authors derive posterior contraction rates for Bayesian functional linear regression with Gaussian priors. However, their main results are only sharp if the prior is correctly specified to match the regularity of the true linear functional. Our work goes beyond this by proving sharp high-probability error bounds for out-of-distribution prediction error under very general smoothness misspecification. However, to do so we require that the prior and data covariance operators commute, which is a limitation of our theory. Another relevant work is \cite{cai2012minimax}, in which the authors make minimal assumptions on the prior and data covariances but still make the well-specified assumption that the truth belongs to the RKHS of the prior. They also require a data-dependent scaling of the prior. Similarly, \cite{pati2015optimal} obtains convergence guarantees for Bayesian nonparametric regression of functions with Gaussian process priors in the misspecified setting. However, their work only considers a squared exponential covariance structure and requires a careful rescaling of the prior to deal with the smoothness misspecification. Our work holds for the larger class of misspecified Mat\'ern-like Gaussian priors and delivers rates without rescaling the prior distribution.

Beyond linear functionals, recent work proposes and analyzes a kernelized deep learning method for  nonlinear functionals~\cite{shi2024nonlinear}. The idea of such neural functionals, a subclass of the FNMs proposed in this work, is not new. One appearance is in the context of a function space discriminator for generative adversarial networks~\cite{rahman2022generative}. However, that work uses only a single bounded linear functional that is appended to the output of a FNO and is parametrized by a standard neural network. This is a special case of our FNMs for F2V maps. Another paper that shares similar ideas to ours is \cite{zhang2023improving}. There, the authors also formulate a V2V neural network approach that maps through a latent 1D function space. However, their encoder and decoder maps are prescribed by hand-picked basis functions, while for FNMs the encoder and decoder maps are learned from data.

In this paper, three potential applications are highlighted. The first application is an advection--diffusion model where the input is a velocity field and the output is the state at a fixed future time. This problem is considered a benchmark for scientific machine learning \cite{takamoto2022pdebench}. Some theoretical approximation rates for it have been developed for DeepONet in the F2F setting \cite{deng2022approximation}. The second application centers on the compressible flow over an airfoil, i.e., a wing cross section. This experiment is explored for FNO in \cite{li2022fourier} and used as a shape optimization example in the F2F setting for DeepONet in \cite{shukla2023deep} and for reduced basis networks in \cite{o2021adaptive}. Several other related works devise V2V-based neural network approaches and novel training strategies for this aerodynamics problem \cite{lye2021multi,lye2020deep,lye2021iterative,mishra2021enhancing}.
The third application involves learning the homogenized elasticity coefficient for a multiscale elliptic PDE. This example is explored in detail for FNO in \cite{bhattacharya2023learning} and for other constitutive laws in \cite{liu2022learning}. For the Darcy flow---or scalar coefficient---setting of this equation, other work adopts the F2F setting to efficiently compute QoIs \cite{you2022learning}. For each of these applications, we compare the generalization error performance of all four F2F, F2V, V2F, and V2V variants of FNMs as well as standard fully-connected neural networks.

\subsection{Notation}\label{ssec:notation}
The set of continuous linear operators from a Banach space $\cU$ to a Banach space $\cV$ is denoted by $\cL(\cU;\cV)$, and if $\cU=\cV$, then we write $\cL(\cV)$. For a separable Hilbert space $(H,\ip{\cdot}{\cdot},\norm{\slot})$, the outer product operator $a \otimes b\in\cL(H)$ is defined by $(a\otimes b)c\defeq \ip{b}{c}a$ for any elements $a$, $b$, and $c$ of $H$. Let $(\Omega,\cF,\P)$ be a probability space that is sufficiently rich to support all random variables that appear in this paper. All expectations are interpreted as Bochner integrals. For a probability measure $\Pi$ supported on $H$ with two finite moments, we denote by $\Cov(\Pi)=\E^{u\sim\Pi}[(u-\E u)\otimes (u-\E u)]\in\cL(H)$ its covariance operator. Independent and identically distributed (i.i.d.) random variables $X_1, X_2,\ldots, X_n$ from $\Pi$ are denoted by $\{X_i\}_{i=1}^n\sim\Pi^{\otimes n}$. For two random variables $X$ and $Z$, the conditional expectation notation $\E^{Z\condbar X}[\slot]$ denotes integration with respect to the law of $Z\condbar X$. For $N\in\N$, we write $[N]\defeq \{1,2,\ldots, N\}$. For two nonnegative real sequences $ \{a_n\}_{n\in\N} $ and $ \{b_n\}_{n\in\N} $, we write 
$ a_n\lesssim b_n $ if there exists \( c>0 \) such that $ a_n \leq cb_n $ for all \( n\in\N \) and $a_n\simeq b_n$ if both $a_n\lesssim b_n$ and $b_n\lesssim a_n$. To denote asymptotic equivalence, we write $a_n\asymp b_n$ as $n\to\infty$ if there exist $c\geq 1$ and $n_0\in\N$ such that $c^{-1}b_n\leq a_n\leq c b_n$ for all $n\geq n_0$. The Sobolev-like sequence Hilbert spaces $\cH^s = \cH^s(\N,\R)$ are defined for $s\in\R$ by $\cH^s \defeq \{v\colon \N\to\R\,|\, \sum_{j=1}^{\infty} j^{2s}\abs{v_j}^2 < \infty\}$. We write $\Td$ for the $d$-dimensional unit torus $[0,1]_{\mathrm{per}}^d$.

\subsection{Outline}\label{ssec:outline}
The remainder of this article is organized as follows. We define the architecture of FNMs as a slight adjustment of FNOs in Section~\ref{sec:nm} (Contribution~\ref{item:contr_fnm}) and confirm that FNMs retain desirable properties of FNOs such as universal approximation in Section~\ref{sec:ua} (Contribution~\ref{item:contr_ua}). In Section~\ref{sec:linear}, we analyze end-to-end and full-field learning of linear functionals to establish a theoretical foundation that underlies the data volume requirements of the two approaches (Contribution~\ref{item:contr_linear}); simulations also support the linear theory. Going beyond the theory, Section~\ref{sec:numerics} provides numerical experiments that compare end-to-end and full-field learning with FNMs with both finite- and infinite-dimensional input space representations for predicting QoIs in several nonlinear PDE problems (Contribution~\ref{item:contr_numerics}). Concluding remarks are given in Section~\ref{sec:conclusion}. Appendix~\ref{appx:thm_full} contains additional theorems related to Section~\ref{sec:linear}. All proofs are provided in Appendices~\ref{appx:proofs_ua}~and~\ref{appx:proofs_linear}.

\section{Neural mappings for finite-dimensional vector data}\label{sec:nm}
In this section, we recall the FNO architecture (Subsection~\ref{sec:nm_no}) and describe modifications of it to form FNMs (Subsection~\ref{sec:nm_nm}). 

\subsection{A review of neural operators}\label{sec:nm_no}
Let $\mcU = \mcU(\mcD;\R^{d_u})$ and $\mcY = \mcY(\mcD;\R^{d_y})$ be Banach function spaces over Euclidean domain $\mcD\subset \R^d$. Finite-dimensional fully-connected neural networks are repeated compositions of affine mappings alternating with pointwise nonlinearities. To extend this framework to the infinite-dimensional function space setting, depth $T$ neural operators from $\mcU$ to $\mcY$ take the form 
\begin{equation}\label{eqn:no}
	\Psi^{(\mathrm{NO})}(u) \defeq \bigl(\mathcal{Q}\circ \mathscr{L}_{T}\circ \mathscr{L}_{T-1}\circ \cdots \circ \mathscr{L}_{1}\circ \mathcal{S}\bigr)(u) \qfa u\in\cU\,,
\end{equation}
where $\mcS$ is a pointwise-defined local lifting operator, $\mcQ$ is a pointwise-defined local projection operator, and for each $t \in [T]$, the layer $\sL_t \colon \cB_{t-1} \to \cB_t$ is a nonlinear map between appropriate Banach function spaces $\cB_{t-1}(\cD; \R^{d_{t-1}}) \subset L^2(\cD;\R^{d_{t-1}})$ and $\cB_t(\cD; \R^{d_t}) \subset L^2(\cD;\R^{d_t})$. The map $\sL_t $ is the composition of a local (and usually nonlinear) operator with a nonlocal affine kernel integral operator \cite{kovachki2023neural}.

The FNO is a specific instance of the class of neural operators~\eqref{eqn:no}. Let $\cD=\Td$. Then for FNO, the form of the layer $\sL_{t}\colon \cB_{t-1}(\T^d; \R^{d_{t-1}}) \to \cB_t(\T^d; \R^{d_{t}})$ is given by $v\mapsto \mathscr{L}_{t}(v)$, where for any $x\in\T^d$, it holds that
\begin{equation}\label{eqn:fno_layer_nonlinear}
    \bigl(\mathscr{L}_{t}(v)\bigr)(x) = \sigma_t\bigl(W_t v(x) + (\cK_t v)(x) + b_t(x)\bigr)\,.
\end{equation}
In \eqref{eqn:fno_layer_nonlinear}, $W_t \in \R^{d_{{t}}\times d_{t-1}}$ is a weight matrix, $b_t\colon \Td \to\R^{d_{t}}$ is a bias function, and $\cK_t$ is a convolution operator given, for $v\colon \Td \to\R^{d_{t-1}}$ and any $x\in\Td$, by the expression
\begin{equation}\label{eqn:fno_Kt}
    (\cK_tv)(x) = \left\{\sum_{k \in \Z^d}\left(\sum_{j=1}^{d_{t-1}} (P^{(k)}_t)_{\ell j}\ip{\psi_k}{v_j}_{L^2(\mathbb{T}^d;\mathbb{C})} \right)\psi_k(x)\right\}_{\ell\in[d_{{t}}]} \in\R^{d_{t}}\,.
\end{equation}
In the preceding display, the $\psi_k = e^{2\pi \iunit \ip{k}{\slot}_{\R^d}}$ are the complex Fourier basis elements of $L^2(\Td;\C)$ and $P^{(k)}_t \in \mathbb{C}^{d_{{t}} \times d_{t-1}}$ are the learnable parameters of the integral operator $\cK_t$ for each $k\in\Z^d$. The functions $\sigma_t\colon \R\to\R$ are nonlinear activations that act pointwise when applied to vectors.
Additional details of more general versions and computational implementations of the FNO may be found in \cite{kovachki2021universal,kovachki2023neural,li2022fourier}.

Though the internal FNO layers $\{\sL_{t}\}$ in \eqref{eqn:fno_layer_nonlinear} and \eqref{eqn:fno_Kt} are defined on the periodic domain $\Td$, it is possible to apply the FNO to other $d$-dimensional domains $\cD\neq \T^d$. Define Banach spaces $\cB_{\mathrm{in}}(\cD;\R^{d_0})$ and $\cB_{\mathrm{out}}(\cD;\R^{d_T})$ and introduce an operator $\cE\colon \cB_{\mathrm{in}} \to \cB_0$. Then, replace $\cS$ in \eqref{eqn:no} with $\cE\circ \cS$. Similarly, let $\cR\colon \cB_{T} \to \cB_{\mathrm{out}}$ be an operator that maps back to functions on the desired domain $\mcD$ and replace $\cQ$ in \eqref{eqn:no} with $\cQ\circ\cR$.
These modifications to the lifting and projecting components yield the final FNO architecture
\begin{equation}\label{eqn:fno}
\Psi^{(\mathrm{FNO})} = \cQ \circ \cR \circ \mathscr{L}_{T}\circ \mathscr{L}_{T-1}\circ \cdots \circ \mathscr{L}_{1}\circ \cE \circ \cS\,.
\end{equation}
In practice, the map $\cE$ is usually represented by zero padding the input domain and $\cR$ by restricting to the output domain of interest.

\subsection{The neural mappings framework}\label{sec:nm_nm}
The neural operator architecture described in Section~\ref{sec:nm_no} only accepts inputs, outputs, and intermediate states that are elements of function spaces. Finite-dimensional vector inputs, outputs, and states are not directly compatible with neural operators. We propose \emph{neural mappings}, which lift this restriction through two fundamental building blocks. The first, linear functional layers, map from function space to finite dimensions. The second, linear decoder layers, map from finite dimensions to function space. We combine these two building blocks with standard iterative neural operator layers to form several classes of nonlinear and function space consistent architectures.

Instating the neural operator notation from Section~\ref{sec:nm_no}, we define a \emph{linear functional layer} $\sG\colon \cB_{T-1}\to \R^{d_{T}}$ and a \emph{linear decoder layer} $ \sD\colon \R^{d_0} \to \cB_1$ to be maps of the form
\begin{align}\label{eqn:sG-sD}
\begin{split}
    h&\mapsto \sG h\defeq \int_\cD \kappa(x)h(x) \dd{x}\,,\qw \kappa\colon\cD\to \R^{d_{T}\times d_{T-1}}\,, \qa\\
    z&\mapsto\sD z\defeq \kappa(\cdot)z\,,\qw \kappa\colon\cD\to \R^{d_1\times d_0}\,,
\end{split}
\end{align}
respectively. The linear functional layer $\sG$ takes a vector-valued function $h$ and integrates it against a fixed matrix-valued function $\kappa$ to produce a finite vector output. In duality to $\sG$, the linear decoder layer $\sD$ takes as input a finite vector $z$ and multiplies it by a fixed matrix-valued function $\kappa$ to produce an output function. The functions $\kappa$ are the sole learnable parameters of these two layers.

Although $\sG$ and $\sD$ may be incorporated into general neural operators~\eqref{eqn:no}, we will specialize our method to the FNO. In anticipation of this periodic setting, we view $\sG$ as a Fourier linear functional layer by replacing $\cD$ in \eqref{eqn:sG-sD} by the torus $\Td$ and using Fourier series to expand $\sG$ as
\begin{equation}\label{eqn:fnf_kernel_fourier_integral}
	 h\mapsto \sG h = \left\{\sum_{k\in\Z^d} \left(\sum_{j=1}^{d_{T-1}}P^{(k)}_{\ell j}\ip{\psi_k}{h_j}_{L^2(\Td;\C)}\right)\right\}_{\ell\in [d_T]}\in\R^{d_T}\, ,
\end{equation}
where we recall that $\{\psi_k\}$ is the Fourier basis of $L^2(\Td;\C)$. In \eqref{eqn:fnf_kernel_fourier_integral}, the entries of the matrices $\{P^{(k)}\}\subset \C^{d_T\times d_{T-1}}$ correspond to the Fourier coefficients of the function $\kappa$ in \eqref{eqn:sG-sD}. 
The calculation leading to the convergent series formula~\eqref{eqn:fnf_kernel_fourier_integral} uses Parseval's theorem to equate the $L^2$~\eqref{eqn:sG-sD} and $\ell^2$~\eqref{eqn:fnf_kernel_fourier_integral} inner products.
Similar calculations show that, on the torus $\Td$, the map $\sD$ takes the form 
\begin{equation}\label{eqn:fnd_kernel}
 	z\mapsto \sD z = \left\{\sum_{k\in\Z^d} \bigl(P^{(k)} z\bigr)_j \psi_k\right\}_{j\in [d_1]}, \qw P^{(k)} \in \C^{d_1 \times d_0}\,.
\end{equation}
Just like for the FNO kernel integral layers \eqref{eqn:fno_Kt}, the expressions \eqref{eqn:fnf_kernel_fourier_integral} and \eqref{eqn:fnd_kernel} are efficiently implemented and learned in Fourier space. 

We are now in a position to define the general FNMs architecture.
    
\begin{definition}[Fourier Neural Mappings]\label{def:FNM}
    Let $Q\colon \R^{d_T} \to \R^{d_y}$ and $S\colon \R^{d_u} \to \R^{d_0}$ be finite-dimensional maps. For $\{\sL_t\}$ defined as in \eqref{eqn:fno} and $\sG$ and $\sD$ defined as in \eqref{eqn:fnf_kernel_fourier_integral} and \eqref{eqn:fnd_kernel}, let
    \begin{equation}\label{eqn:FNM}
    \Psi^{(\text{FNM})} \defeq Q \circ \sG \circ \sL_{T-1} \circ \dots \circ \sL_2 \circ \sD \circ S\,.
    \end{equation}
    be the base level map. The \emph{Fourier Neural Mappings} architecture is comprised of the following four main models that are obtained by modifying the base map:
    \begin{itemize}[topsep=1.67ex,itemsep=0.5ex,partopsep=1ex,parsep=1ex,leftmargin=9.5ex]
        \myitem{(M-V2V)} \label{enum:V2V} \sfit{vector-to-vector (V2V)}: $\Psi^{(\text{FNM})}$ in \eqref{eqn:FNM} as written, thus mapping finite vector inputs to finite vector outputs;
        
        \myitem{(M-V2F)} \label{enum:V2F}  \sfit{vector-to-function (V2F)}: $\Psi^{(\text{FNM})}$ with operator $\sG$ in \eqref{eqn:FNM} replaced by $\cR\circ \sL_T$, where $\cR$ and $\sL_T$ are as in \eqref{eqn:fno} and \eqref{eqn:fno_layer_nonlinear}, respectively, and $Q$ in \eqref{eqn:FNM} now viewed as a pointwise-defined operator acting on vector-valued functions;
        
        \myitem{(M-F2V)} \label{enum:F2V} \sfit{function-to-vector (F2V)}: $\Psi^{(\text{FNM})}$ with operator $\sD$ in \eqref{eqn:FNM} replaced by $\sL_1\circ\cE$, where $\sL_1$ and $\cE$ are as in \eqref{eqn:fno_layer_nonlinear} and \eqref{eqn:fno}, respectively, and $S$ in \eqref{eqn:FNM} now viewed as a pointwise-defined operator acting on vector-valued functions;
        
        \myitem{(M-F2F)} \label{enum:F2F} \sfit{function-to-function (F2F)}: $\Psi^{(\text{FNM})}$ with modifications \ref{enum:V2F} and \ref{enum:F2V}, thus the resulting architecture is the standard FNO $\Psi^{(\text{FNO})}$~\eqref{eqn:fno}.\footnote{Notice that yet another function-to-function FNM architecture is possible by exchanging the roles of $\sG$ and $\sD$ in \eqref{eqn:FNM}; this is a nonlinear Fourier neural autoencoder.}
    \end{itemize}
\end{definition}
When the \ref{enum:F2V} FNM is of primary interest, we sometimes call this architecture \emph{Fourier Neural Functionals}. Similarly, we may also call the \ref{enum:V2F} FNM a \emph{Fourier Neural Decoder}.

\section{Universal approximation theory for Fourier Neural Mappings}\label{sec:ua}
In this section, we establish universal approximation theorems for FNMs; this is a confirmation that the architectures maintain this desirable property of neural operators. The results are stated for the cases of the F2V and V2F architectures; the case of V2V trivially follows. Similar results also hold for general neural mappings by invoking the appropriate universal approximation theorems for general neural operators from \cite[Section 9.3]{kovachki2021universal} and for the topology induced by Lebesgue--Bochner norms, i.e., average error with respect to a probability measure supported on the input space. For more details regarding these extensions, see \cite[Theorems~11--14, Section 9.3, pp. 55--57]{kovachki2023neural} and \cite[pp. 12--14 and Theorem 18]{kovachki2021universal}. Our proofs, which are collected in Appendix~\ref{appx:proofs_ua}, use arguments based on constant functions that are similar to those used to prove universal approximation theorems at the level of operators.

The approximation theory in this section relies on the following assumption.
\begin{assumption}[activation function]\label{ass:activation}
All nonlinear layers $\{\sL_{t}\}_{t=1}^{T}$ from \eqref{eqn:fno_layer_nonlinear} have the same non-polynomial and globally Lipschitz activation function $\sigma\in C^{\infty}(\R;\R)$.
\end{assumption}

We note that in practice, the final Fourier layer activation function is often set to be the identity. Moreover, the bias functions $b_t$ in $\sL_t$ are typically chosen to be constant functions. The universal approximation theory does not distinguish these differences. 
Additionally, to align with the existing theory developed in \cite{kovachki2021universal}, our existence proofs rely on a reduction to the setting that
\begin{enumerate}[label=(\textit{\roman*})]
    \item the channel dimension $d_t$ is constant across all layers, say $d_t=d_v\in\N$ for all $t\in[T]$, and\label{item:channel_const_dim}

    \item the maps $S$ and $Q$ in \eqref{eqn:FNM} are linear and act pointwise on functions.\label{item:lift_project_affine}
\end{enumerate}
These conditions are certainly special cases of nonconstant channel dimension and nonlinear lifting and projection maps, respectively. Hence, the forthcoming universality properties still hold for more sophisticated architectures that deviate from conditions (\textit{i}) and (\textit{ii}), such as those used in Section~\ref{sec:numerics} in this paper.

Our first result delivers a universal approximation result for Fourier Neural Functionals, i.e., the F2V setting. Appendix~\ref{appx:proofs_ua_fnm} contains the proof.

\begin{restatable}[universal approximation: function-to-vector mappings]{theorem}{thmuafnf}\label{thm:ua_fnf}
    Let $s\geq 0$, $\mcD\subset \R^d$ be an open Lipschitz domain such that $\overline{\mcD}\subset (0,1)^d$, and $\cU=H^s(\mcD;\R^{d_u})$. Let $\Psi^\dagger\colon \cU\to\R^{d_y}$ be a continuous mapping. Let $K\subset \cU$ be compact in $\cU$. Under Assumption~\ref{ass:activation}, for any $\ep>0$, there exist Fourier Neural Functionals $\Psi\colon \cU\to\R^{d_y}$ of the form \eqref{eqn:FNM} with modification~\ref{enum:F2V} such that
    \begin{align}\label{eqn:ua_fnf_ep}
    \sup_{u\in K}\norm[\big]{{\Psi}^\dagger(u)-\Psi(u)}_{\R^{d_y}}<\ep\,.
    \end{align}
\end{restatable}

The approximation theorem for the Fourier Neural Decoder, i.e., the V2F case, is analogous.
\begin{restatable}[universal approximation: vector-to-function mappings]{theorem}{thmuafnd}\label{thm:ua_fnd}
    Let $t\geq 0$, $\mcD\subset \R^d$ be an open Lipschitz domain such that $\overline{\mcD}\subset (0,1)^d$, and $\cY=H^t(\mcD;\R^{d_y})$. Let $\Psi^\dagger\colon \R^{d_u}\to\cY$ be a continuous mapping. Let $\mathcal{Z}\subset \R^{d_u}$ be compact. Under Assumption~\ref{ass:activation}, for any $\ep>0$, there exists a Fourier Neural Decoder $\Psi\colon \R^{d_u}\to\cY$ of the form \eqref{eqn:FNM} with modification \ref{enum:V2F} such that
    \begin{align}\label{eqn:ua_fnd_ep}
    \sup_{z\in\mathcal{Z}}\norm[\big]{{\Psi}^\dagger(z)-\Psi(z)}_{\cY}<\ep\,.
    \end{align}
\end{restatable}
The proof may also be found in Appendix~\ref{appx:proofs_ua_fnm}. While perhaps not surprising, the results in Theorems~\ref{thm:ua_fnf}~and~\ref{thm:ua_fnd} nonetheless show that the proposed FNM architectures are sensible for the tasks of approximating continuous function-to-vector or vector-to-function mappings.

\section{Statistical theory for regression of linear function-to-scalar maps}\label{sec:linear}
The previous two sections propose and justify a general nonlinear framework for approximating PtO maps with finite-dimensional input or output spaces. Although universality properties of the proposed architectures are established, the efficiency of statistically estimating the underlying PtO map from a finite dataset remains to be addressed. Such sample complexity results are crucial to understand the expected performance of learning algorithms in scenarios where experimental or computational resources for data generation are limited. However, it is an open challenge to develop such a theory in the general \emph{nonlinear} setting previously considered.

To still shed some light on this issue, we provide a detailed theoretical analysis of learning a \emph{linear} PtO map $f$ that admits a factorization into a linear functional $q$, the QoI, composed with a self-adjoint linear operator $L$, the forward map. 
To set the stage, we adopt the following setup in the remainder of Section~\ref{sec:linear}. Let $(H,\ip{\cdot}{\cdot}, \norm{\slot})$ be an infinite-dimensional real separable Hilbert space. We view the PtO map $f=q\circ L\colon H\to\R$ as a linear functional on $H$. Hence, the \emph{input space is always infinite-dimensional} in this section. Let $\nu$ be the data-generating Borel probability measure on the input space $H$. We work in the setting that
\begin{align}\label{eqn:linear_data_center_and_cov}
    \E^{u\sim\nu} u = 0\qa \Sigma\defeq \Cov(\nu)=\sum_{j=1}^\infty\sigma_j\varphi_j\otimes \varphi_j
\end{align}
for some orthonormal basis $\{\varphi_j\}_{j\in\N}$ of $H$ and eigenvalue sequence $\{\sigma_j\}_{j\in\N}\subset \R_{\geq 0}$, and that $N$ i.i.d. input data samples $\{u_n\}_{n=1}^N\sim\nu^{\otimes N}$ are available.

We consider two different supervised learning approaches, visualized in Figure~\ref{fig:ee_vs_ff} as F2V and F2F, that correspond to the given labeled output training data $\{y_n\}_{n=1}^N$ being either noise-perturbed versions of 
\begin{itemize}[topsep=1.67ex,itemsep=1ex,partopsep=1ex,parsep=1ex]
    \myitem{(EE)}\label{item:ee} \sfit{(end-to-end learning)} \emph{the entire PtO map $f$ applied to the input data}, or
    \myitem{(FF)}\label{item:ff} \sfit{(full-field learning)} \emph{the forward map $L$ applied to the input data}. \emph{Moreover, in this latter case the linear functional $q$ is assumed to be known}.
\end{itemize}
In both cases, the primary goal is to estimate $f$ given input--output data pairs. Notice that in the \ref{item:ee} approach, the responses are scalar-valued, while for the \ref{item:ff} approach the responses are function-valued.\footnote{Although the Hilbert space $H$ is general and not necessarily comprised of functions, we still refer to its elements as ``functions'' to avoid confusion caused by attempts to distinguish between finite-dimensional vectors and infinite-dimensional vectors.}
In some cases, the problem itself specifies the approach that can be used. For example, in physical experiments, it is often impossible to experimentally acquire the full output of $L$. Instead, only a finite number of possibly indirect measurements are available, which represent the QoI $q$ in the \ref{item:ee} framework. Even if the data generation procedure is algorithmic,  the cost of generating full-field data may be much higher than simply measuring finite-dimensional QoIs. Nevertheless, for a large class of forward maps $L$ and \emph{continuous} QoIs $q$, we demonstrate that \ref{item:ff} is more data-efficient than \ref{item:ee} in this linear setup; see insights~\ref{item:insight_rough} and \ref{item:insight_smooth} and Figure~\ref{fig:rates_plot} for more details.

To this end, Subsection~\ref{sec:linear_ee} 
sets up the framework for statistical error analysis of learning general linear functionals on $H$ with a particular Bayesian posterior estimator; this is the end-to-end setting. Subsection~\ref{sec:linear_ff} then continues the setup for the setting of 
factorized linear functionals in the full-field learning setting. Subsection~\ref{sec:linear_main} contains the main results. Here, two theorems provide convergence rates for the end-to-end and full-field settings, respectively. In particular, the end-to-end result may be of independent interest to the functional data analysis and Bayesian nonparametric statistics communities. The theorems are followed by a discussion and a corollary that directly compares the data efficiency of end-to-end versus full-field learning of factorized linear functionals. In Subsection~\ref{sec:linear_numeric}, synthetic data experiments empirically confirm the theoretical results.

\subsection{End-to-end learning}\label{sec:linear_ee}
Consider a general linear functional $f\colon H\to\R$. Working in a nonparametric functional regression framework, we adopt a linear Gaussian posterior mean estimator that is obtained by conditioning a Gaussian process prior on the training dataset under the access model~\ref{item:ee}. Our primary concern is the development of large sample convergence rates for the average squared prediction error of the estimator with respect to some centered test distribution $\nu'$. We allow $\nu'$ to be different from the input training distribution $\nu$~\eqref{eqn:linear_data_center_and_cov}; that is, our error bounds hold \emph{out-of-distribution} or under \emph{covariate shift}. We describe our statistical model and Bayesian inference approach in Subsection~\ref{sec:linear_ee_setup}. In Subsection~\ref{sec:linear_ee_ass}, we list and interpret the main assumptions that underlie the theory.

\subsubsection{Setup and estimator}\label{sec:linear_ee_setup}
We adopt a Bayesian inverse problems perspective on the linear functional regression task. To simplify the analysis, suppose that $f\in H^*$ is continuous on $H$.\footnote{It is possible to handle \emph{unbounded} linear functionals $f$ using the weighted Hilbert--Schmidt approach from \cite[Section 2.2]{de2023convergence} or the framework of measurable linear functionals from \cite[Section 3 and 5]{knapik2011bayesian}, possibly at the expense of stronger assumptions on the data and prior covariances.}
As a mild abuse of notation, we use the same symbol $f$ for both the linear functional itself and its Riesz representer. We thus view $f$ as an element of $H$ in all that follows. 

Next, for noise level $\gamma>0$ and for $n\in[N]$, we consider the statistical model
\begin{equation}\label{eqn:model_data_qoi}
    y_n=\ip{f}{u_n} + \gamma\xi_n\,, \qw u_n\diid \nu \qa \xi_n\diid \normal(0,1)\,.
\end{equation}
We get to observe $y_n$ and $u_n$, but not the noise $\xi_n$. Concatenate the data and noise as $U=\{u_n\}_{n=1}^N$, $Y=\{y_n\}_{n=1}^N$, and $\Xi = \{\xi_n\}_{n=1}^N$. This allows \eqref{eqn:model_data_qoi} to be recast as the linear inverse problem of finding $f$ from inputs $U$ and outputs
\begin{equation}\label{eqn:model_data_qoi_cat}
    Y=S_Nf+\gamma \Xi\,,
\end{equation}
where $S_N\colon H\to \R^N$ is the (random) sampling operator $h\mapsto \{\ip{h}{u_n}\}_{n=1}^N$. 

Proceeding with the Bayesian approach, we endow $f$ with a Gaussian prior distribution $f\sim \normal(0,\Lambda)$. Here $\Lambda\in\cL(H)$ is a trace-class covariance operator on $H$. The advantage of working with the preceding linear Gaussian model is that the posterior distribution for $f$ is also a Gaussian measure supported on $H$ with closed form expressions for its mean and covariance. To this end, let
\begin{align}\label{eqn:sigmahat}
    \widehat{\Sigma}\defeq \frac{S_N^*S_N^{\phantom{*}}}{N}=\frac{1}{N}\sum_{n=1}^Nu_n\otimes u_n
\end{align}
denote the empirical covariance operator of the centered input distribution $\nu$~\eqref{eqn:linear_data_center_and_cov}. Additionally, suppose that $U$, $\Xi$, and $f$ are independent as random variables.
Although the randomness of the operator $S_N$ is a slight deviation from the usual Bayesian setting, application of \cite[Proposition 3.1, pp. 2630--2631]{knapik2011bayesian} and \cite[Theorems 32, 13, and 37]{dashti2017} still imply that the posterior distribution obtained by conditioning the prior $f$ on the training dataset $(U,Y)$ is given by
\begin{equation}\label{eqn:posterior_ee}
    f\condbar (U,Y)\sim\normal(\bar{f}^{(N)}, \Lambda^{(N)}),\qw \bar{f}^{(N)}= A_N Y
\end{equation}
and the operator $A_N\colon \R^N\to H$ and posterior covariance $\Lambda^{(N)}\in\cL(H)$ satisfy
\begin{equation}\label{eqn:AN_and_postcov}
A_N\defeq \gamma^{-2} \Lambda^{(N)} S_N^*
\qa \Lambda^{(N)} = \frac{\gamma^2}{N}\Lambda^{1/2}\biggl(\Lambda^{1/2}\widehat{\Sigma}\Lambda^{1/2} + \frac{\gamma^2}{N}\Id_H\biggr)^{-1}\Lambda^{1/2}\,.
\end{equation}

We take a frequentist consistency perspective by assuming there exists a fixed ground truth $\fd\in H$ that generates the data $Y$ in \eqref{eqn:model_data_qoi_cat}. Abusing notation by using the same symbol for $Y$, this means that we observe the output response data\footnote{The actual noise process $\Xi$ in the observed data \eqref{eqn:model_data_qoi_cat_dagger} need not match the assumed Gaussian likelihood model implied by \eqref{eqn:model_data_qoi}. Indeed, the main results in Subsection~\ref{sec:linear_main} for the posterior mean remain valid for any centered square integrable random vector with isotropic covariance.}
\begin{equation}\label{eqn:model_data_qoi_cat_dagger}
    Y=S_N\fd+\gamma \Xi\,.
\end{equation}
Our estimator of $\fd$ is then taken to be the posterior mean $\bar{f}^{(N)}=A_NY$ from \eqref{eqn:posterior_ee} with $Y$ as in \eqref{eqn:model_data_qoi_cat_dagger}. However, the particular estimator $\bar{f}^{(N)}$ is chosen to simplify the exposition. As explained in Remark~\ref{rmk:contraction}, the analysis remains valid if the full posterior distribution $\normal(\bar{f}^{(N)}, \Lambda^{(N)})$ from \eqref{eqn:posterior_ee} is used to estimate $\fd$ instead of just its mean. Posterior contraction rates also follow as a consequence.

Instead of measuring the quality of our estimate of $\fd$ in the $H$-norm, we consider a weaker weighted norm induced by the average squared prediction error of the estimator. This weighted norm is more common in statistical learning than in statistical inverse problems; it includes the $H$-norm as a special case whenever $\Cov(\nu')=\Id_H$. To this end, let $\nu'$ be a centered Borel probability measure supported on a sufficiently large space containing $H$. We are interested in the out-of-distribution test squared error of $\bar{f}^{(N)}$ with respect to the test distribution $\nu'$, which is given by
\begin{equation}\label{eqn:test_error_weighted}
    \E^{u'\sim\nu'}\abs[\big]{\ip{\fd}{u'}-\ip{\bar{f}^{(N)}}{u'}}^2 = \norm[\big]{\Cov(\nu')^{1/2}(\fd-\bar{f}^{(N)})}^2\,.
\end{equation}
The derivation of this identity is explained in Appendix~\ref{appx:proofs_linear_ee}. Equation~\eqref{eqn:test_error_weighted} is equivalent to the squared $L^2_{\nu'}(H;\R)$ Bochner norm error between the functionals.

Finally, we explain how our posterior mean estimator relates to standard regularized least squares minimizers from machine learning.
\begin{remark}[equivalence to regularized empirical risk minimization]
    Under the setting described in this section, the posterior mean is equivalent to a generalized-Tikhonov regularized least squares estimator. See \cite[Section 2.3]{de2023convergence} and \cite[pp. 3631--2632]{knapik2011bayesian} for more details and relations to RKHS methods. This connects the Bayesian approach taken here back to traditional supervised learning frameworks.
\end{remark}

\subsubsection{Assumptions}\label{sec:linear_ee_ass}
In the end-to-end setting, we work under three primary assumptions that involve the data, the prior, and the truth.
The first assumption concerns the covariance operators that characterize the end-to-end learning framework~\ref{item:ee}.
\begin{assumption}[end-to-end learning: data and prior]\label{ass:data_end_to_end}
    Instate the setup and hypotheses developed in Subsection~\ref{sec:linear_ee_setup}. The following hold true.
    \begin{enumerate}[label=(A\arabic*)]
        \item\label{item:ass_ee_diag} \sfit{(simultaneous diagonalization)} The covariance operator $\Sigma=\Cov(\nu)$ of centered input training data distribution $\nu$ satisfies \eqref{eqn:linear_data_center_and_cov}. The covariance operators $\Sigma'\defeq\Cov(\nu')$ of the centered input test data distribution $\nu'$ and $\Lambda$ of the centered Gaussian prior are diagonalized in the $H$-orthonormal eigenbasis $\{\varphi_j\}_{j\in\N}$ of $\Sigma$ and have the representations
        \begin{align}\label{eqn:linear_ass_diag}
            \Sigma'=\sum_{j=1}^\infty \sigma_j'\varphi_j\otimes\varphi_j\qa
            \Lambda=\sum_{j=1}^\infty \lambda_j\varphi_j\otimes\varphi_j
        \end{align}
        for some sequences $\{\sigma_j'\}_{j\in\N}\subset \R_{\geq 0}$ and $\{\lambda_j\}_{j\in\N}\subset \R_{\geq 0}$.
        \item\label{item:ass_ee_data} \sfit{(data decay)} For some $\alpha > \frac{1}{2}$ and $\al'\geq 0$, the eigenvalues of $\Sigma$ and $\Sigma'$ satisfy
        \begin{align}\label{eqn:linear_ass_data_eig}
            \sigma_j \asymp j^{-2\alpha} \qas j\to\infty \qa \sigma'_j\lesssim j^{-2\alpha'} \qas j\to\infty \,.
        \end{align}
        \item\label{item:ass_ee_prior} \sfit{(prior decay)} For some $p>\frac{1}{2}$, the eigenvalues of $\Lambda$ satisfy       \begin{align}\label{eqn:linear_ass_prior_eig}
            \lambda_j \asymp j^{-2p} \qas j\to\infty\,.
        \end{align}
    \end{enumerate}
\end{assumption}

The power law eigenvalue decay conditions \ref{item:ass_ee_data} and \ref{item:ass_ee_prior} are standard in learning theory and help to facilitate explicit algebraic convergence rates. However, other eigenvalue assumptions such as exponential decay~\cite{agapiou2014bayesian} or convex eigenvalues~\cite{cardot2007clt} are also possible and interesting to study. The exponent lower bounds $ \al>1/2 $ and $ p>1/2 $ ensure that $\Sigma$ and $\Lambda$ are trace-class on $H$, and thus the posterior~\eqref{eqn:AN_and_postcov} is a proper Gaussian measure on $H$.

The simultaneous diagonalization condition~\ref{item:ass_ee_diag} is the strongest assumption in the paper and is the main limitation of our theory in Section~\ref{sec:linear}. It implies that the prior covariance operator $\Lambda$, which we choose \emph{a priori}, is perfectly aligned with the eigenspaces of the training and test data covariance operators $\Sigma$ and $\Sigma'$, which are unknown in general. Nonetheless, in certain settings, these eigenspaces may be known from problem-specific knowledge such as periodicity, translation invariance, or rotation invariance.
Also, if the data generation procedure or experimental design is controlled by the practitioner, then the data and prior covariance operators can freely be chosen to share the same eigenfunctions, e.g., Mat\'ern-like covariances with different lengthscales and regularity exponents.

In our analysis and in related work~\cite{de2023convergence,knapik2011bayesian}, the simultaneous diagonalization implied by Assumption~\ref{item:ass_ee_diag} enables the derivation of tight error bounds that remain valid for misspecified Gaussian priors. The resulting convergence rates clearly isolate the influence that problem regularity, data quality, and prior information have on sample complexity. In contrast, papers in statistical inverse problems that do not require the assumption of simultaneous diagonalization usually obtain less precise and less interpretable results compared to ours \cite{agapiou2013posterior,gugushvili2020bayesian,knapik2018general,ray2013bayesian,trabs2018bayesian,yuan2010reproducing}. Moreover, many of these works are limited to the well-specified setting in which the implied smoothness of the prior perfectly matches the unknown regularity of the ground truth. Our analysis overcomes this limitation at the expense of simultaneous diagonalization.
In further contrast to traditional linear inverse problems, in our finite data linear functional regression setting the actual normal operator corresponding to the inverse problem~\eqref{eqn:model_data_qoi_cat_dagger} is the empirical covariance $S_N^*S_N^{\phantom{*}}/N=\widehat{\Sigma}$~\eqref{eqn:sigmahat}, which \emph{does not} commute with $\Lambda$ in general for any finite $N$. Our theory accommodates this challenge and only assumes that $\Lambda$ commutes with $\Sigma$~\eqref{eqn:linear_data_center_and_cov}, which is the infinite data limit of $\widehat{\Sigma}$.
However, the assumption that test data covariance $\Sigma'$ and prior covariance $\Lambda$ have the same eigenbasis is less realistic because the test distribution $\nu'$ may be outside of the user's control and differ substantially from the training distribution $\nu$~\eqref{eqn:linear_data_center_and_cov} in some applications. Nevertheless, the present work is only one of a handful~\cite{de2023convergence,jin2022minimax,mollenhauer2022learning} to provide out-of-distribution test error bounds at all. The commuting assumption is satisfied for two popular choices of test data covariance, $\Sigma'=\Sigma$ (if $\Sigma$ and $\Lambda$ commute) and $\Sigma'=\Id_H$, which reflect the prediction error and the strong estimation error, respectively.
It is clear that weakening or removing Assumption~\ref{item:ass_ee_diag} while maintaining sharp, interpretable, and widely applicable theoretical insights for PtO map learning is an important direction for future research.

Assumption~\ref{ass:data_end_to_end} provides fine-grained control of the second moments of the data and prior distributions. Next, we impose a slightly strong assumption about the tails of the Karhunen--Lo\`eve (KL) expansion of the training data distribution $\nu$~\eqref{eqn:linear_data_center_and_cov} in order to obtain high probability error bounds.
\begin{assumption}[strongly-subgaussian training data]\label{ass:data_input_kl_expand}
The input training data distribution $\nu$ is a centered Borel probability measure on $H$ with KL expansion
\begin{align}\label{eqn:data_input_kl_expand}
    u=\sum_{j=1}^\infty\sqrt{\sigma_j}z_j\varphi_j\sim\nu\,,
\end{align}
where the eigenvalues $\{\sigma_j\}_{j\in\N}$ of $\Cov(\nu)=\Sigma$ are nonincreasing and the $\{z_j\}_{j\in\N}$ are zero mean, unit
variance, independent random variables that satisfy\footnote{See Appendix~\ref{appx:proofs_linear_sub} for the definition of subgaussian random variables and their norms $\norm{\slot}_{\psi_2}$.}
\begin{align}\label{eqn:data_input_subg_uniform_bound}
    m\defeq\sup_{j\in\N}\norm{z_j}_{\psi_2}=\sup_{j\in\N}\left(\sup_{\ell\geq 1}\ell^{-1/2}\bigl(\E\abs{z_j}^\ell\bigr)^{1/\ell}\right)<\infty\,.
\end{align}
\end{assumption}

Equation~\eqref{eqn:data_input_subg_uniform_bound} in Assumption~\ref{ass:data_input_kl_expand} implies that the training data distribution $\nu$~\eqref{eqn:linear_data_center_and_cov} is subgaussian in a relatively strong sense. This enables the use of exponential concentration inequalities in the proofs of forthcoming results. We further insist that the KL expansion coefficients $\{z_j\}_{j\in\N}$~\eqref{eqn:data_input_kl_expand} are an independent family in order to align with a similar assumption made for full-field learning in the next subsection. However, the following remark explains how this undesirable independence condition can be eliminated for end-to-end learning at the expense of worse tail bounds.
\begin{remark}[independence of the KL coefficients]\label{rmk:kl_wo}
    The requirement of independence of the KL expansion coefficients $\{z_j\}_{j\in\N}$~\eqref{eqn:data_input_kl_expand}, while commonly found in the literature~\cite{bartlett2020benign,hucker2022note}, is undesirable because it excludes many interesting statistical models. This condition is only used in Lemma~\ref{lem:good_set}, which shows that a particular event has high probability. This lemma requires a strong notion of subgaussianity~\eqref{eqn:app_subg_need_indep}; the assumed independence of the KL coefficients suffices to satisfy this condition.
    However, we show in Lemma~\ref{lem:good_set_wo} that the requirement of independence can be replaced by the condition that the $\{z_j\}_{j\in\N}$ are pairwise uncorrelated. At the cost of a smaller probability for the event of interest, this improvement extends the applicability of our end-to-end learning theory. For example, using the kernel trick, the main theoretical results of this section imply convergence rates for Gaussian process regression of scalar-valued \emph{nonlinear} functions as a special case.
\end{remark}

Last, we require a Sobolev-like regularity assumption on the true linear functional $\fd$ in order to derive explicit convergence rates.
\begin{assumption}[regularity of ground truth linear functional]\label{ass:truth_regularity}
    For each $j\in\N$, denote by $\fd_j=\ip{\fd}{\varphi_j}$ the coefficients of $\fd$. For some $s\geq 0$, it holds that
    \begin{align}\label{eqn:truth_regularity_functional}
        \norm*{\fd}_{\cH^s}^2\defeq \sum_{j=1}^\infty j^{2s}\abs{\fd_j}^2<\infty\,.
    \end{align}
    Additionally, $\al+s>1$, where $\al$ is as in \eqref{eqn:linear_ass_data_eig}.
\end{assumption}
Notice that Assumption~\ref{ass:truth_regularity} \emph{implies} that $\fd\in H^*$ is a continuous linear functional on $H$ because $s\geq 0$ and hence the coefficients of $\fd$ belong to $\ell^2(\N;\R)$. The regularity constraint linking $\al$ with $s$ is a technical condition that ensures a certain event has vanishing probability in the large sample limit (see Lemma~\ref{lem:app_series_uniform_bound}); it may be possible to weaken this constraint with alternative proof techniques.

\subsection{Full-field learning}\label{sec:linear_ff}
In this subsection, we provide the additional assumptions and framework required for the setting in which the true linear PtO map $\fd$ is factorized as $\fd=\qd \circ \Ld$, where the QoI $\qd$ is a linear functional on $H$ and $\Ld$ is a self-adjoint linear operator mapping $H$ to $H$. We allow $\qd$ or $\Ld$ to potentially be unbounded with respect to the topology of $H$. With the full-field learning data access model~\ref{item:ff}, we fully observe noisy versions of the \emph{function-valued} output of $\Ld$ at the training input functions drawn from $\nu$~\eqref{eqn:linear_data_center_and_cov}. We adopt a Bayesian posterior mean estimator $\bar{L}^{(N)}$ for $\Ld$ based on these data. The final estimator of the true linear functional $\fd$ is obtained by composing $\qd$ with the learned operator $\bar{L}^{(N)}$. Subsection \ref{sec:linear_ff_setup} contains the setup, while Subsection \ref{sec:linear_ff_ass} contains the assumptions we make and gives examples of some QoIs that satisfy these assumptions.

\subsubsection{Setup and estimator}\label{sec:linear_ff_setup}
Our full-field training data observations are given by 
\begin{equation}\label{eqn:model_data_ff}
    Y_n=\Ld u_n + \eta_n\,, \qw u_n\diid \nu \qa \eta_n\diid \normal(0,\Id_H)
\end{equation}
for $n\in[N]$. Equation~\eqref{eqn:model_data_ff} should be interpreted in a weak sense, i.e., as $H$-indexed stochastic processes, because $\eta_n\not\in H$ almost surely~\cite[Section 2.2.2., p. 11]{de2023convergence}. Without loss of generality, we assume a unit noise level because this does not affect the asymptotic results.
To make the analysis tractable, we work in the setting that $\Ld$ is diagonalized in the eigenbasis $\{\varphi_j\}_{j\in\N}$ of $\Sigma$ from~\eqref{eqn:linear_data_center_and_cov}. Thus, we write
\begin{align}\label{eqn:linear_ff_Ldiag}
    \Ld=\sum_{j=1}^\infty \ld_j\varphi_j\otimes\varphi_j
\end{align}
and develop an estimator for the eigenvalue sequence $\ld=\{\ld_j\}_{j\in\N}$. Details about the domain of $\Ld$ and the topology in which \eqref{eqn:linear_ff_Ldiag} converges may be found in \cite{de2023convergence}.

Our Bayesian approach follows \cite{de2023convergence} by modeling the eigenvalue sequence $\ld$ with an independent scalar Gaussian prior $l_j\sim\normal(0,\mu_j)$ on each eigenvalue. Write $\Upsilon=\{Y_n\}_{n=1}^N$ and $U=\{u_n\}_{n=1}^N$ and assume that $l=\{l_j\}_{n\in\N}$, $\Upsilon$, and $U$ are independent. Then \cite[Fact 2.4, p. 12]{de2023convergence} furnishes the posterior distribution
\begin{align}\label{eqn:posterior_ff}
    l\condbar (U,\Upsilon) \sim \bigotimes_{j=1}^\infty \normal\left(\bar{l}_j^{(N)}, c_j^{(N)}\right)\,.
\end{align}
In \eqref{eqn:posterior_ff}, $\{\bar{l}_j^{(N)}\}_{j\in\N}$ are the posterior mean eigenvalues and $\{c_j^{(N)}\}_{j\in\N}$ are the posterior variances.
The plug-in estimator for $\fd=\qd\circ \Ld$ is then given by
\begin{align}\label{eqn:linear_ff_estimator}
    \qd\circ \bar{L}^{(N)}\,,
    \qw
    \bar{L}^{(N)}\defeq \sum_{j=1}^\infty \bar{l}_j^{(N)}\varphi_j\otimes \varphi_j\,.
\end{align}
The precise formulas for the mean and variance in \eqref{eqn:posterior_ff} are given in \cite[Equation (2.4), p. 12]{de2023convergence}. As in Subsection~\ref{sec:linear_ee_setup}, we are interested in the out-of-distribution test squared error of $\qd\circ \bar{L}^{(N)}$ with respect to a centered input test measure $\nu'$.

\subsubsection{Assumptions}\label{sec:linear_ff_ass}
We now collect the main assumptions for the full-field learning approach; these are primarily drawn from~\cite[Assumption 3.1, pp. 14--15]{de2023convergence}.
\begin{assumption}[full-field learning: main assumptions]\label{ass:ff}
    Instate the setup and hypotheses of Subsection~\ref{sec:linear_ff_setup}. The following hold true.
    \begin{enumerate}[label=(A-\Roman*), ,topsep=1.67ex,itemsep=0.5ex,partopsep=1ex,parsep=1ex]
        \item \sfit{(simultaneous diagonalization, data decay, and tail condition)} \label{item:ass_ff_data} The input distributions $\nu$ and $\nu'$ and their covariance operators $\Sigma$ and $\Sigma'$ satisfy Assumptions \ref{item:ass_ee_diag} and \ref{item:ass_ee_data}.\footnote{Inspection of the proof of \cite[Theorem 3.9, pp.~18--19]{de2023convergence} shows that the high probability upper bound (Equation 3.10) there remains valid if it is assumed that $\sigma'_j$ is only bounded above asymptotically by $j^{-2\al'}$ and not necessarily from below.} Moreover, $\nu$ satisfies Assumption~\ref{ass:data_input_kl_expand}.

        \item\label{item:ass_ff_L} \sfit{(forward operator regularity)} The eigenvalues $\ld$ of operator $\Ld$ in \eqref{eqn:linear_ff_Ldiag} satisfy $\ld \in \cH^\beta$ for some $\beta\in\R$.\footnote{Recall that the Sobolev-like sequence Hilbert spaces $\cH^s$ for $s\in\R$ are defined in Subsection~\ref{ssec:notation}.}
        
        \item \sfit{(prior regularity)} \label{item:ass_ff_prior} The prior $l_j\sim\normal(0,\mu_j)$ on eigenvalues has variances satisfying $\mu_j\asymp j^{-2\beta - 1}$ as $j\to\infty$.
        
        \item \sfit{(QoI regularity)} \label{item:ass_ff_qoi} The QoI $\qd$ satisfies $\abs{\qd(\varphi_j)}^2 \lesssim j^{-2r-1}$ as $j\to\infty$ for some $r\in\R$ such that $\min(\alpha, \alpha' + r +1/2) + \beta > 0$, with $\al$ and $\al'$ as in \ref{item:ass_ff_data}.

        \item \sfit{(PtO map continuity)} \label{item:ass_ff_pto} The PtO map $\qd\circ\Ld$ is continuous, i.e., 
        \begin{align}
            \sum_{j=1}^\infty\abs{\qd(\varphi_j)}^2 \abs{\ld_j}^2<\infty\,.
        \end{align}
    \end{enumerate}
\end{assumption}

The conditions in Assumption~\ref{ass:ff} have similar interpretations to those in Subsection~\ref{sec:linear_ee_ass} and \cite[Assumption 3.1, pp. 14--15]{de2023convergence}. For simplicity, in \ref{item:ass_ff_prior} we have already chosen the optimal prior smoothness exponent $\beta + 1/2$. The condition $\alpha' + r + 1/2+\beta > 0$ in \ref{item:ass_ff_qoi} ensures that the PtO map $\qd\circ\Ld$ has finite $L^2_{\nu'}(H;\R)$ Bochner norm; thus, the test error is well-defined. The continuity of the PtO map enforced by \ref{item:ass_ff_pto} aligns with the \ref{item:ee} setting.
The power law decay of the coefficients of $\qd$ in \ref{item:ass_ff_qoi} allows for a sharp convergence analysis. Several common, concrete linear QoIs satisfy this condition, as the next remark demonstrates.
\begin{remark}[examples of linear QoIs]\label{rmk:linear_qoi_examples}
    Several simple QoIs $\qd$ satisfy the power law decay in \ref{item:ass_ff_qoi}. Let $H=L^2((0,1);\R)$. An orthonormal basis for $H$ is $x\mapsto \varphi_j(x) = \sqrt{2}\sin(j\pi x)$ for each $j\in\N$. For convenience, denote $\qd_j\defeq \qd(\varphi_j)$.
    \begin{itemize}
        \item \sfit{(mean on an interval)} The map $\qd\colon h\mapsto \int_0^1h(x)\dd{x}=\ip{\onebm}{h}_{L^2((0,1))}$ is continuous and has Riesz representer $\onebm\colon x\mapsto 1$. The coefficients of $\qd$ satisfy
        \begin{align}\label{eqn:qoi_mean}
            \abs{\qd_j}^2 =\abs[\bigg]{\frac{\sqrt{2}(1-\cos(j\pi))}{j\pi}}^2 = \frac{8\one_{\{j \ \text{odd}\}}}{j^2\pi^2}\lesssim j^{-2}\,.
        \end{align}
        Hence, $r=1/2$ is a valid decay exponent in Assumption~\ref{item:ass_ff_qoi}.
        
        \item \sfit{(point evaluation)} The map $\qd\colon h\mapsto h(x_0)$ for a fixed $x_0\in (0,1)$ is not continuous on $H$. It holds that
        \begin{align}\label{eqn:qoi_point}
            \abs{\qd_j}^2=2\abs{\sin(j \pi x_0)}^2\lesssim 1
        \end{align}
        and hence $r=-1/2$ is a valid decay exponent in \ref{item:ass_ff_qoi}.

        \item \sfit{(point evaluation of derivative)} The map $\qd\colon h\mapsto (dh/dx)(x_0)$ for a fixed $x_0\in (0,1)$ is not continuous on $H$. Its coefficients satisfy
        \begin{align}\label{eqn:qoi_pointd}
            \abs{\qd_j}^2=2\pi^2j^2\abs{\cos( j\pi x_0)}^2\lesssim j^2
        \end{align}
        and hence $r=-3/2$ is a valid decay exponent. This QoI is not smooth.
    \end{itemize}
\end{remark}

\subsection{Main results}\label{sec:linear_main}
Building upon the setup from the previous two subsections, this subsection establishes convergence rates for end-to-end learning in Theorem \ref{thm:linear_ee_opt} and full-field learning in Theorem \ref{thm:linear_ff_main_power}. Both results are stated for the expectation of the out-of-distribution test error conditioned on the input data $U$. Intuitively, this averages out the noise in the data; full expectation bounds are also possible to obtain, as in \cite[Corollary D.2, p.~226]{nelsen2024statistical}. The two theorems are interpreted in Subsection~\ref{sec:linear_main_diss}. This discussion is followed by Subsection~\ref{sec:linear_compare}, which directly compares the end-to-end and full-field methods in Corollary \ref{cor:linear_compare}.

The first theorem describes convergence rates in the end-to-end learning setting.

\begin{theorem}[end-to-end learning: optimized convergence rate]\label{thm:linear_ee_opt}
    Let the input training data distribution $\nu$, the test data distribution $\nu'$, and the Gaussian prior $\normal(0,\Lambda)$ satisfy Assumptions~\ref{ass:data_end_to_end} and \ref{ass:data_input_kl_expand}. Let the ground truth linear functional $\fd\in\cH^s$ satisfy Assumption~\ref{ass:truth_regularity} with $s>0$. Let $\al$ and $\al'$ be as in \eqref{eqn:linear_ass_data_eig} and $p=s+1/2$ be as in \eqref{eqn:linear_ass_prior_eig}.
    Then there exists $c\in (0,1/4)$ and $N_0\geq 1$ such that for any $N\geq N_0$, the mean $\bar{f}^{(N)}$ of the Gaussian posterior distribution~\eqref{eqn:posterior_ee} arising from the $N$ pairs of observed training data $(U,Y)$ in \eqref{eqn:model_data_qoi_cat_dagger} satisfies the error bound
    \begin{align}\label{eqn:linear_ee_main_optimal}
    \E^{Y\condbar U}\E^{u'\sim\nu'}\abs[\big]{\ip{\fd}{u'}-\ip{\bar{f}^{(N)}}{u'}}^2 \lesssim
        \bigl(1 + \norm{\fd}_{\cH^s}^2 \bigr)\ep_N^2
    \end{align}
    with probability at least $1-2\exp(-c N^{\min(1,\frac{\al+s-1}{\al+s+1/2})})$ over $U\sim\nu^{\otimes N}$, where
    \begin{align}\label{eqn:linear_ee_main_rate_optimal}
        \ep_N^2=
        \begin{cases}
            N^{-\bigl(\frac{2\al'+2s}{1+2\al + 2s}\bigr)}\, , & \textit{if }\, \al'<\al+1/2\,,\\
            N^{-1}\log 2N\, , & \textit{if }\, \al'=\al+1/2\,,\\
            N^{-1}\, , & \textit{if }\, \al'>\al+1/2\,.
        \end{cases}
    \end{align}
    The constants $c$, $N_0$, and the implied constant in \eqref{eqn:linear_ee_main_optimal} do not depend on $N$ or $\fd$.
\end{theorem}

Appendix~\ref{appx:thm_full} contains a more general version of the preceding theorem that is valid \emph{for any} $p>1/2$ (Theorem~\ref{thm:linear_ee_main}). The assertion of this version is optimized for the choice $p = s + 1/2$ made in Theorem~\ref{thm:linear_ee_opt}. The proof of Theorem~\ref{thm:linear_ee_main}, from which Theorem~\ref{thm:linear_ee_opt} follows immediately, may be found in Appendix~\ref{appx:proofs_linear_ee}.

The second main theorem in this subsection describes the expected squared test error in the QoI after learning an approximate forward map from full-field data.

\begin{restatable}[full-field learning: convergence rate for power law QoI]{theorem}{thmlinearffmainpower}\label{thm:linear_ff_main_power}
    Let the input training data distribution $\nu$, the test data distribution $\nu'$, the true forward map $\Ld$, and the QoI $\qd$ satisfy Assumption~\ref{ass:ff}. Let $\al$ and $\al'$ be as in \eqref{eqn:linear_ass_data_eig} and $\beta$ and $r$ be as in \ref{item:ass_ff_L} and \ref{item:ass_ff_qoi}.
    Then there exist constants $c>0$ and $C>0$ such that for all sufficiently large $N$, the plug-in estimator $\qd\circ \bar{L}^{(N)}$ in \eqref{eqn:linear_ff_estimator} based on the Gaussian posterior distribution~\eqref{eqn:posterior_ff} arising from the $N$ pairs of observed full-field training data $(U,\Upsilon)$ in \eqref{eqn:model_data_ff} satisfies the error bound
    \begin{equation}\label{eqn:thmerror_power}
        \E^{\Upsilon \condbar U}\E^{u'\sim{\nu'}}\abs[\big]{\qd(\Ld u') - \qd(\bar{L}^{(N)}u')}^2 \lesssim \ep_N^2
    \end{equation}
    with probability at least $1-Ce^{-cN}$ over $U\sim\nu^{\otimes N}$, where
    \begin{equation}\label{eqn:thmcases_power}
        \ep_N^2\defeq 
        \begin{cases}
            N^{-\bigl(\frac{1+2\al'+2\beta+2r}{1+2\al + 2\beta}\bigr)}\, , & \textit{if }\, \al'+r<\al\,,\\
            N^{-1}\log N\, , & \textit{if }\, \al'+r=\al\,,\\
            N^{-1}\, , & \textit{if }\, \al'+r>\al\,.
        \end{cases}
    \end{equation}
    The constants $c$, $C$, and the implied constant in \eqref{eqn:thmerror_power} do not depend on $N$.
\end{restatable}

Appendix~\ref{appx:thm_full} also contains a similar convergence result for QoIs with an assumed Sobolev-like regularity instead of power law regularity. Proofs of both this result and Theorem~\ref{thm:linear_ff_main_power} are collected in Appendix~\ref{appx:proofs_linear_ff}.

\subsubsection{Discussion}\label{sec:linear_main_diss}
The proof of Theorem~\ref{thm:linear_ee_opt} is based on a bias--variance decomposition argument that is detailed in Appendix \ref{appx:proofs_linear_ee}. Notice that the error bound in the theorem is uniform over $\cH^s$-balls because the implied constant in the inequality~\eqref{eqn:linear_ee_main_optimal} does not depend on $\fd$. The probability that the bound fails to hold decays to zero faster than any power law as a function of the sample size $N$ because $\al+s-1>0$ by hypothesis. The convergence rate~\eqref{eqn:linear_ee_main_optimal} depends on the three smoothness exponents $\al$, $\al'$, and $s$. The consequences are the same as those identified in \cite[Section 3]{de2023convergence} because the error bound takes the same form. Indeed, rougher training data points $\{u_n\}_{n=1}^N$ (smaller $\al$), smoother test data points $u'\sim\nu'$ (larger $\al'$), and smoother target functionals $\fd$ (larger $s$) all serve to reduce the test error (up to saturation). 

Moreover, the optimal choice of $p$ made in Theorem \ref{thm:linear_ee_opt} depends on the regularity exponent $s$ of the ground truth $\fd$, which is unknown. However, there exist more sophisticated estimators that adapt to the unknown regularity and achieve the optimal convergence rate~\cite{agapiou2023heavy,knapik2016bayes}. Nonetheless, the fact that we are even able to choose $p=s+1/2$ in the first place is one of the novelties of our result. Most existing theoretical work on functional linear regression requires some constraint linking the regularity $p$ of the prior to the regularity $s$ of $\fd$. The most common assumption corresponds to the \emph{well-specified} setting \cite{cai2012minimax,lian2016posterior,yuan2010reproducing,zhang2020faster}, which means that $\fd$ belongs to the RKHS $\image(\Lambda^{1/2})\subset H$ of the prior $\normal(0,\Lambda)$. In terms of coefficients, this is equivalent to assuming that $\sum_{j=1}^\infty \lambda_j^{-1}\abs{\fd_j}^2\lesssim \norm{\fd}^2_{\cH^p}<\infty$ because $\lambda_j\asymp j^{-2p}$. However, we only have that $\fd\in\cH^s$. With the optimal choice $p=s+1/2$, it is possible that $\fd\not\in \cH^p=\cH^{s+1/2}$. Our theory allows for such RKHS misspecification. In the full-field learning setting, similar notions of robustness to misspecification are guaranteed by the error bounds in \cite{de2023convergence}.

The consequences of Theorem~\ref{thm:linear_ff_main_power} for full-field learning are similar to those of Theorem~\ref{thm:linear_ee_opt} for end-to-end learning with regard to the smoothness exponents that define the estimation problem. However, Theorem~\ref{thm:linear_ff_main_power} is only valid for QoIs $\qd$ with asymptotic power law decay of the form $\abs{\qd(\varphi_j)}^2 \lesssim j^{-2r-1}$ as $j\to\infty$. While many QoIs in practice satisfy this condition, it still corresponds to a relatively small set within the class of all linear functionals. For example, if the asymptotic power law decay condition holds, then $\{\qd(\varphi_j)\}_{j\in\N}\in\cH^{r-\ep}$ for every $\ep>0$. It is natural to wonder whether Theorem~\ref{thm:linear_ff_main_power} remains valid if $\qd$ is only assumed to satisfy such a Sobolev-like regularity condition. To this end, Theorem~\ref{thm:linear_ff_main} in Appendix~\ref{appx:thm_full} generalizes Theorem~\ref{thm:linear_ff_main_power} to all of $\cH^r$, but at the expense of a worse convergence rate where $r$ in \eqref{eqn:thmcases_power} is replaced by $r-1/2$.

To conclude the discussion, we remark on more general estimators based on full posterior distributions.
\begin{remark}[posterior contraction rates]\label{rmk:contraction}
    First consider the \ref{item:ee} setting.
    With minor modifications, the more general Theorem~\ref{thm:linear_ee_main}, and hence also Theorem~\ref{thm:linear_ee_opt}, remains valid for the \emph{posterior sample estimator} $f^{(N)}\sim \normal(\bar{f}^{(N)},\Lambda^{(N)})$ instead of its mean. To see this, note that the KL expansion of the posterior~\eqref{eqn:posterior_ee} yields
    \begin{align}\label{eqn:error_posterior_full}
    \begin{split}
         \E^{f^{(N)}\sim \normal(\bar{f}^{(N)},\Lambda^{(N)})}\norm[\big]{(\Sigma')^{1/2}(\fd-{f}^{(N)})}^2 = &\, \norm[\big]{(\Sigma')^{1/2}(\fd-\bar{f}^{(N)})}^2 \\
         &\qquad + \tr[\big]{(\Sigma')^{1/2}\Lambda^{(N)}(\Sigma')^{1/2}}\,.
    \end{split}
    \end{align}
    Theorem~\ref{thm:linear_ee_main} bounds the conditional expectation of the first term on the right-hand side of the preceding equality. We see that the only new error term that the full posterior introduces is the second term, the \emph{posterior spread}. But the end of Subsection~\ref{appx:proofs_linear_ee_var} explains that the posterior spread may be upper bounded by a constant times the rate $\ep_N^2$ from Theorem~\ref{thm:linear_ee_main}. Thus, the end-to-end error bounds~\eqref{eqn:linear_ee_main} and \eqref{eqn:linear_ee_main_optimal} remain valid for the posterior sample estimator $f^{(N)}$ at the expense of enlarged constant factors. Posterior contraction rates then follow from a standard Chebyshev inequality argument~\cite[Section 3.3, p. 18]{de2023convergence}. Similar results may be deduced for the full-field setting \ref{item:ff} because the error analysis for the forward map in \cite[Section 3.4]{de2023convergence} already takes into account the full posterior distribution \eqref{eqn:posterior_ff}.
\end{remark}

\subsubsection{Sample complexity comparison}\label{sec:linear_compare}
To conclude Subsection~\ref{sec:linear_main}, we provide a detailed comparison of the end-to-end~\ref{item:ee} and full-field~\ref{item:ff} PtO map learning approaches to provide intuition about their statistical performance. Focusing on the specific setting of QoIs with power law coefficient decay and \emph{in-distribution} test error (i.e., $\al'=\al$), the following corollary is a consequence of Theorem~\ref{thm:linear_ee_opt} and Theorem~\ref{thm:linear_ff_main_power}. A short proof is provided in Appendix~\ref{appx:proofs_linear_compare}.

\begin{restatable}[sample complexity comparison]{corollary}{corlinearcompare}\label{cor:linear_compare}
    Instate the notation and assertions in Assumptions~\ref{ass:data_end_to_end}, \ref{ass:data_input_kl_expand}, and \ref{item:ass_ff_prior}.
    Suppose that the training and test distributions are equal, i.e., $\nu=\nu'$. Let the underlying true PtO map $\fd$ have the factorization $\fd=\qd\circ\Ld$, where $\abs{\qd(\varphi_j)}^2 \lesssim j^{-2r-1}$ as $j\to\infty$ and $\Ld$ is as in \eqref{eqn:linear_ff_Ldiag} with eigenvalues $\ld\in\cH^\beta$.
    If $\beta+r+1/2>0$, $\al+\beta+r>1/2$, and $\al+\beta>0$, then there exist constants $c>0$ and $C>0$ such that for all sufficiently large $N$, the following holds on an event with probability at least $1-C\exp(-c N^{\min(1,\frac{\al+\beta+r-1/2}{1+\al+\beta+r})})$ over $U=\{u_n\}_{n=1}^N\sim\nu^{\otimes N}$.
    The \ref{item:ee} posterior mean estimator $\bar{f}^{(N)}$ in \eqref{eqn:posterior_ee} (with $p\defeq \beta+r+1$ in \ref{item:ass_ee_prior}) trained on end-to-end data $(U,Y)$ satisfies
    \begin{equation}\label{eqn:linear_compare_ee}
        \E^{Y \condbar U}\E^{u\sim\nu}\abs[\big]{\qd(\Ld u)-\ip{\bar{f}^{(N)}}{u}}^2\lesssim
        N^{{-\bigl(1-\frac{1}{2+2\al+2\beta+2r}\bigr)}}\,.
    \end{equation}
    On the other hand, the \ref{item:ff} plug-in estimator $\qd\circ \bar{L}^{(N)}$ in \eqref{eqn:linear_ff_estimator} trained on full-field data $(U,\Upsilon)$ satisfies
    \begin{align}\label{eqn:linear_compare_ff}
        \E^{\Upsilon \condbar U}\E^{u\sim{\nu}}\abs[\big]{\qd(\Ld u) - \qd(\bar{L}^{(N)}u)}^2 \lesssim
        \begin{cases}
            N^{-\bigl(1-\frac{-2r}{1+2\al + 2\beta}\bigr)}\, , & \textit{if }\, r<0\,,\\
            N^{-1}\log N\, , & \textit{if }\, r=0\,,\\
            N^{-1}\, , & \textit{if }\, r>0\,.
        \end{cases}
    \end{align}
\end{restatable}

Several interesting insights may be deduced from the convergence rates \eqref{eqn:linear_compare_ee} and \eqref{eqn:linear_compare_ff}. Since a common set of assumptions have been identified in the statement of Corollary~\ref{cor:linear_compare}, these rates may be directly compared to assess whether the \ref{item:ee} or \ref{item:ff} approach is in general more accurate or, equivalently, more data-efficient, than the other---at least up to the sharpness of the upper bounds.
First, we note that the squared generalization error of \ref{item:ee} in \eqref{eqn:linear_compare_ee} has a nonparametric convergence rate $N^{-(1-\delta)}$ that is always slower than the parametric estimation rate $N^{-1}$ by a polynomial factor $N^\delta$ (where $\delta>0$). On the other hand, if $r\geq 0$, then \ref{item:ff} achieves the fast parametric rate $N^{-1}$ (up to a log factor if $r=0$); this always beats the \ref{item:ee} upper bound in the regime $r\geq 0$. This regime has an interesting interpretation because \emph{the QoI is continuous if $r>0$}. These types of QoIs appear naturally in scientific applications.

To study the regime $r<0$, let
\begin{align}\label{eqn:linear_compare_rateexp}
    \rho_{\mathrm{EE}}(r)\defeq 1-\frac{1}{2+2\al+2\beta+2r} \qa \rho_{\mathrm{FF}}(r) \defeq 1-\frac{2\max(-r,0)}{1+2\al+2\beta}
\end{align}
for $r\neq 0$ denote the convergence rate exponents of the end-to-end and full-field estimators, respectively (ignoring the log factor when $r=0$ in \eqref{eqn:linear_compare_ff}). A larger exponent implies faster convergence and better sample complexity. A simple algebraic factorization shows that $\rho_{\mathrm{EE}}(r)=\rho_{\mathrm{FF}}(r)=\rho$ at the points
\begin{align*}
   (r_0,\rho_0)=\left(-\frac{1+2\al+2\beta}{2}, 0\right) \qa (r_1,\rho_1)=\left(-\frac{1}{2}, \frac{2\al+2\beta}{1+2\al+2\beta}\right)\,.
\end{align*}
By the concavity of $r\mapsto \rho_{\mathrm{EE}}(r)$ in the range $[r_0, r_1]$ and affine structure of $r\mapsto \rho_{\mathrm{FF}}(r)$, we deduce the following two insights from the standpoint of upper bounds:
\begin{enumerate}[label=(I\arabic*),topsep=1.67ex,itemsep=0.5ex,partopsep=1ex,parsep=1ex]
    \item\label{item:insight_rough} \sfit{(\ref{item:ee} is better for rough QoIs)} \hspace{0.25em} $\rho_{\mathrm{EE}}(r) > \rho_{\mathrm{FF}}(r)$ for $r_0< r < -1/2$ and
    \item\label{item:insight_smooth} \sfit{(\ref{item:ff} is better for smooth QoIs)} $\rho_{\mathrm{EE}}(r) < \rho_{\mathrm{FF}}(r)$ for $r>-1/2$.
\end{enumerate}

\begin{figure}[tb]
    \begin{center}
        \includegraphics[width=0.7\textwidth]{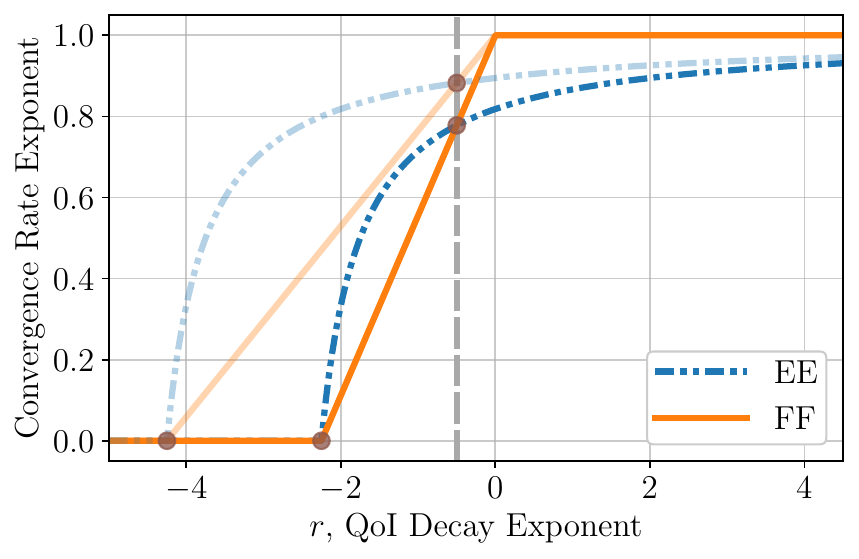}
        \caption{\ref{item:ee} vs. \ref{item:ff} convergence rate exponents \eqref{eqn:linear_compare_rateexp} as a function of QoI regularity exponent $r$. Larger exponents imply faster convergence rates. As the curves gets lighter, $\al+\beta$, an indicator of the smoothness of the problem, increases. The vertical dashed line corresponds to $r=-1/2$, which is the transition point where \ref{item:ee} and \ref{item:ff} have the same rate and the onset of power law decay for the QoI coefficients begins.}
        \label{fig:rates_plot}
    \end{center}
\end{figure}

The inequalities in \ref{item:insight_rough} and \ref{item:insight_smooth} \emph{suggest} that the \ref{item:ff} approach is advantageous when the QoI is smooth and the \ref{item:ee} approach is advantageous when the QoI is rough. Numerical evidence in the next subsection strongly supports this assertion. However, \emph{definitive} claims would require lower bounds to precisely separate the two methods. An example of a rough QoI is $\qd \colon h\mapsto (dh/dx)(x_0)$, which returns a point evaluation of the first derivative of a univariate function; see the last item in Remark~\ref{rmk:linear_qoi_examples}. Direct application of such a rough QoI to a function is an ill-posed operation, e.g., amplifies perturbations in the function. This may partially explain why \ref{item:ee} is preferable in this case, as the \ref{item:ee} estimator does not require the evaluation of $\qd$ while \ref{item:ff} does. We plot the functions $\rho_{\mathrm{EE}}$ and $\rho_{\mathrm{FF}}$ in Figure~\ref{fig:rates_plot}, which visualizes the main insights \ref{item:insight_rough} and \ref{item:insight_smooth} from the preceding discussion.

\subsection{Simulation study}\label{sec:linear_numeric}
To validate the theoretical results in Subsection~\ref{sec:linear_main}, we numerically implement the linear \ref{item:ee} and \ref{item:ff} estimators on a synthetic data benchmark problem.
Specializing to the setting of Corollary~\ref{cor:linear_compare}, we take $H=L^2((0,1);\R)$ with orthonormal basis $\{\varphi_j\}_{j\in\N}$ defined by $\varphi_j(x)=\sqrt{2}\sin(j\pi x)$ for each $x\in(0,1)$ and $j\in\N$.
The forward map $\Ld$ is defined by its diagonalization \eqref{eqn:linear_ff_Ldiag} with $\ld_j\defeq (j\pi)^{-2}$ for each $j\in\N$. In physical space, this corresponds to the solution operator of the one-dimensional Poisson differential equation equipped with Dirichlet boundary conditions. It follows that $\ld\in\cH^\beta(\N;\R)$ for any $\beta<\beta^\star\defeq 3/2$. We fix $\Ld$ in what follows. Next, we vary the ground truth QoI map $\qd$ across the three examples in Remark~\ref{rmk:linear_qoi_examples}---a point evaluation of a derivative, a point evaluation, and an integral. These QoIs satisfy the decay condition $\abs{\qd(\varphi_j)}^2\lesssim j^{-2r-1}$ for $r=-3/2$, $r=-1/2$, and $r=1/2$, respectively. The resulting true factorized linear functional $\fd\defeq \qd\circ\Ld$ has coordinates $\fd_j\defeq \fd(\varphi_j)=\qd(\varphi_j)\ld_j$ for each $j\in\N$.

The data and prior covariance operators are diagonalized in the basis $\{\varphi_j\}_{j\in\N}$ and are characterized by their eigenvalues, which for each $j\in\N$ belong to the set
\begin{align}\label{eqn:eig_set}
    \set{\tau^{2\vartheta-1}(\pi^2j^2 + \tau^2)^{-\vartheta}}{\tau\in \R_{> 0}\qa \vartheta\in\R}\,.
\end{align}
This spectrum corresponds to a Mat\'ern-like covariance structure. For the eigenvalues $\{\sigma_j\}_{j\in\N}$ of the data covariance operator $\Sigma$~\eqref{eqn:linear_data_center_and_cov}, we fix $\tau=15$ and $\vartheta=0.75$ in \eqref{eqn:eig_set}. This implies that the asymptotic data decay assumption~\ref{item:ass_ee_data} is satisfied with $\alpha=\vartheta=0.75$; the numerical results are qualitatively the same when $\al$ equals $2$ or $4.5$. The hypotheses $\beta+r+1/2>0$, $\al+\beta + r > 1/2$, and $\al+\beta>0$ of Corollary~\ref{cor:linear_compare} are always satisfied given the above values of $\al=0.75$, $\beta<3/2$, and $r\in\{-3/2, -1/2, 1/2\}$.
Next, we choose the data distribution $\nu$ to be a Gaussian measure $\normal(0,\Sigma)$, which has a KL expansion~\eqref{eqn:data_input_kl_expand} as in Assumption~\ref{ass:data_input_kl_expand} with i.i.d. coefficients $z_j\sim\normal(0,1)$ for each $j$.
For the Gaussian prior distributions, in \eqref{eqn:eig_set} we fix $\tau=2$ and take $\vartheta=\beta+r+1$ for the \ref{item:ee} estimator prior~\eqref{eqn:linear_ass_diag} and $\vartheta=\beta+1/2$ for the \ref{item:ff} estimator prior~\ref{item:ass_ff_prior}, as required by the hypotheses of Corollary~\ref{cor:linear_compare}.
Finally, we take the scalar standard deviation parameter for the Gaussian white noise processes that corrupt the observed outputs to be $\gamma=10^{-3}$ for both the \ref{item:ee} and \ref{item:ff} training datasets.

Recalling the prior covariance operator $\Lambda$~\eqref{eqn:linear_ass_diag}, the input training data $\{u_n\}_{n=1}^N\sim\nu^{\otimes N}$, and the sampling operator $S_N$ from Subsection~\ref{sec:linear_ee_setup}, we efficiently implement the \ref{item:ee} posterior mean estimator~\eqref{eqn:posterior_ee} in dual form via the explicit formula
\begin{align}\label{eqn:ee_kernel_formula}
    \bar{f}^{(N)} = \Lambda S_N^*\bigl(S_N^{\phantom{*}}\Lambda S_N^* + \gamma^2\Id_{\R^N}\bigr)^{-1} Y\,.
\end{align}
The preceding display follows from \eqref{eqn:posterior_ee}, \eqref{eqn:AN_and_postcov}, and an infinite-dimensional Woodbury matrix identity~\cite[Theorem 1]{ogawa1988operator}. Working in coordinates, we obtain for each $j\in\N$ the convenient sequence space representation
\begin{align}\label{eqn:ee_kernel_seq}
    \bar{f}^{(N)}_j\defeq \ip{\bar{f}^{(N)}}{\varphi_j}_{L^2} = \lambda_j\sum_{n=1}^N u_{jn}V_n^{(N)}\,,
\end{align}
where $\{\lambda_j\}_{j\in\N}$ are the eigenvalues of $\Lambda$, $u_{jn}\defeq \ip{u_n}{\varphi_j}_{L^2}$, and $V^{(N)}\defeq (S_N^{\phantom{*}}\Lambda S_N^* + \gamma^2\Id_{\R^N})^{-1} Y\in\R^N$. The Gram matrix $S_N^{\phantom{*}}\Lambda S_N^*\in\R^{N\times N}$ has entries $\sum_{j\in\N}\lambda_j u_{jn}u_{jn'}$ for each $n\in[N]$ and $n'\in[N]$. We discretize these entries by spectral truncation---which replaces sums over the countable set $\N$ by sums over the finite set $\{j\in\N \colon j\leq J\}$---and we also retain only the first $J$ modes in \eqref{eqn:ee_kernel_seq} for fixed $J=4096$. 
The \ref{item:ff} estimator $\tilde{f}^{(N)}\defeq \qd \circ \bar{L}^{(N)}$~\eqref{eqn:linear_ff_estimator} is defined by its coordinates $\tilde{f}^{(N)}_j\defeq\tilde{f}^{(N)}(\varphi_j)=\qd(\varphi_j)\bar{l}_j^{(N)}$, where $\bar{l}_j^{(N)}$ is as in \eqref{eqn:posterior_ff}. We also discretize this estimator by spectral truncation with $J=4096$ resolved Fourier modes. The results for both the \ref{item:ee} and \ref{item:ff} estimators are nearly identical if the resolution $J$ is increased to $8192$ or higher.

\begin{figure}[tb]
    \centering
    \begin{subfigure}[]{0.325\textwidth}
        \centering
        \includegraphics[width=\textwidth]{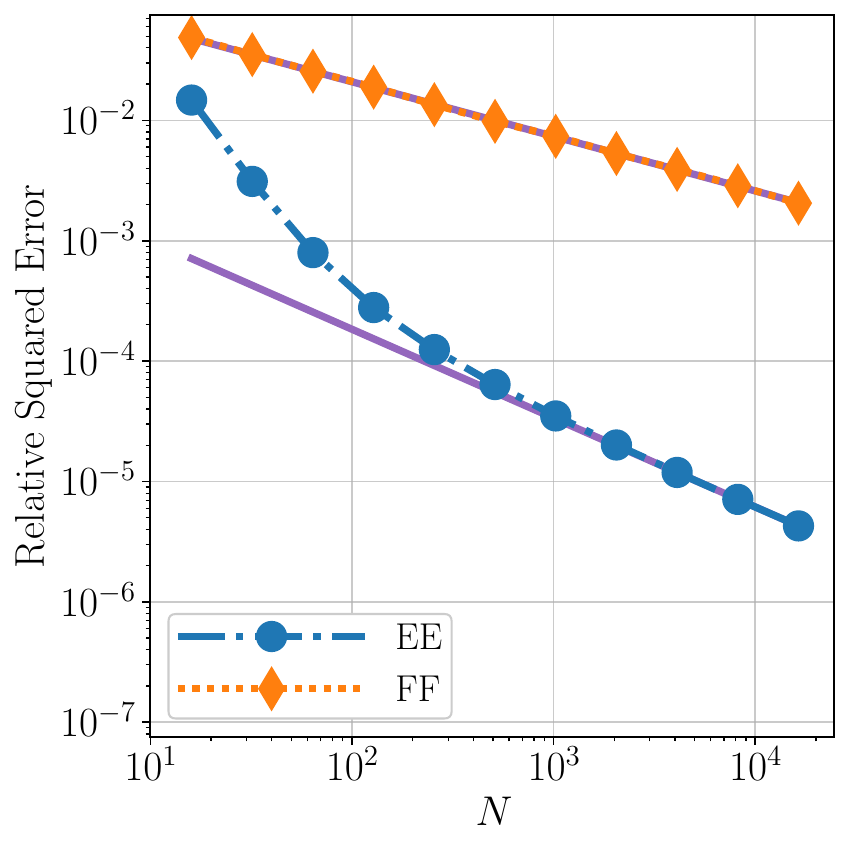}
        \caption{$r=-3/2$}
        \label{subfig:rate_compare_al0_qoi0}
    \end{subfigure}
    \hfill%
    \begin{subfigure}[]{0.325\textwidth}
        \centering
        \includegraphics[width=\textwidth]{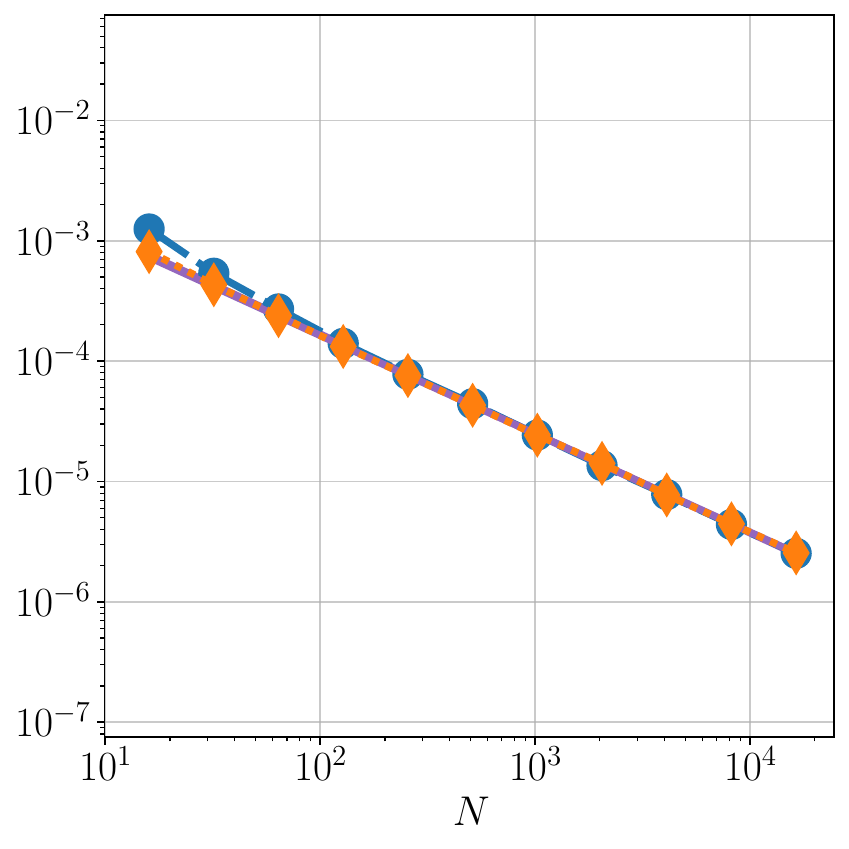}
        \caption{$r=-1/2$}
        \label{subfig:rate_compare_al0_qoi1}
    \end{subfigure}
    \hfill%
    \begin{subfigure}[]{0.325\textwidth}
        \centering
        \includegraphics[width=\textwidth]{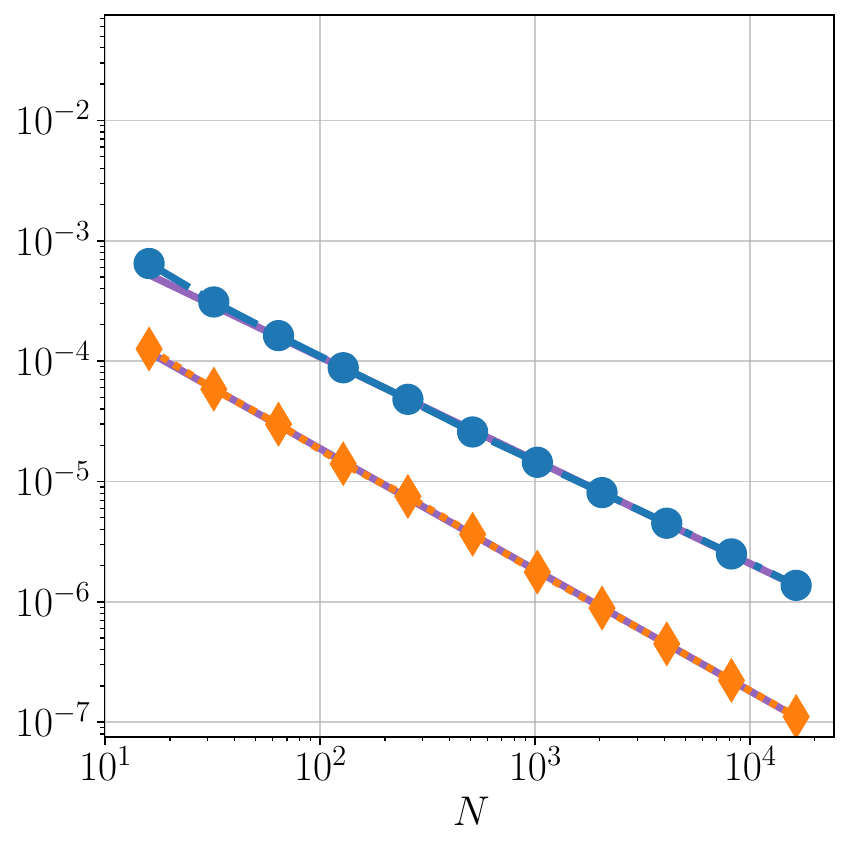}
        \caption{$r=1/2$}
        \label{subfig:rate_compare_al0_qoi2}
    \end{subfigure}
    \caption{Empirical sample complexity of the Bayesian \ref{item:ee} and \ref{item:ff} estimators for linear PtO maps based on a Poisson problem. The solid purple lines are best linear fits to the broken curves with markers, which correspond to numerically computed squared errors. In all three figures, the experimentally observed convergence rates are nearly perfect matches to those from the theoretical upper bounds in Corollary~\ref{cor:linear_compare} (see Table~\ref{tab:rates_linear}).}
    \label{fig:rate_compare_al0}
\end{figure}

In Figure~\ref{fig:rate_compare_al0}, we plot the expected relative squared $L^2_\nu(H;\R)$ Bochner norm error as a function of the sample size $N$. Specifically, the vertical axis represents
\begin{align}\label{eqn:rel_er_linear}
    \frac{\E\Bigl[\E^{u\sim\nu}\abs{\ip{\fd}{u}_{L^2} - \ip{\widehat{f}}{u}_{L^2}}^2\Bigr]}
    {\E^{u\sim\nu}\abs{\ip{\fd}{u}_{L^2}}^2} = \frac{\E\Bigl[\sum_{j=1}^\infty \sigma_j\abs{\fd_j - \widehat{f}_j}^2\Bigr]}{\sum_{k=1}^\infty \sigma_k\abs{\fd_k}^2}\,,
\end{align}
where $\widehat{f}$ is either the \ref{item:ee} or \ref{item:ff} estimator, $\widehat{f}_j\defeq \ip{\widehat{f}}{\varphi_j}_{L^2}$, and the outer expectation is with respect to the realizations of the $N$ training data pairs. For each $N\in\{2^4,\ldots, 2^{14}\}$, we approximate the outer expectation in the preceding display by averaging over $1000$ i.i.d. realizations of the training data; hence, $1000$ estimators are trained for each $N$. The error bounds in Corollary~\ref{cor:linear_compare} remain valid for the full outer expectation in \eqref{eqn:rel_er_linear} by \cite[Corollary D.2, p.~226]{nelsen2024statistical}. The infinite series in \eqref{eqn:rel_er_linear} are discretized by finite sums over the first $J=4096$ summands. The solid purple lines in Figure~\ref{fig:rate_compare_al0} are linear least squares fits to the logarithm of the computed squared errors. We then take the slopes of these lines as the experimental convergence rate exponents, which are recorded in parentheses in Table~\ref{tab:rates_linear}.

\begin{table}[tb]
    \centering
    \caption{Sample complexity comparison experiment. The entries of the table record the theoretical vs. experimental (in parentheses) convergence rate exponents $ \rho $ in $ O(N^{-\rho}) $ of the relative expected squared $ L_{\nu}^{2}(H;\R) $ in-distribution error~\eqref{eqn:rel_er_linear}. The theoretical rates are taken from the upper bound in Corollary~\ref{cor:linear_compare}. Here $\al = 0.75$.
    }
    \label{tab:rates_linear}
    \renewcommand{\arraystretch}{1.2}
    \begin{tabular}{l@{\hspace{10.25mm}}ccc}
        \toprule
        \multicolumn{1}{l}{} &
        \multicolumn{1}{c}{$r=-3/2$} &
        \multicolumn{1}{c}{$r=-1/2$} &
        \multicolumn{1}{c}{$r=1/2$} \\
        \midrule
        $\mathrm{EE}$ &  0.714 (0.738) & 0.818 (0.815) & 0.867 (0.864)\\
        $\mathrm{FF}$ &  0.455 (0.455) & 0.818 (0.818) & 1.000 (1.003)\\
        \bottomrule
    \end{tabular}
\end{table}

Table~\ref{tab:rates_linear} strongly suggests that for the specific linear functional learning problem under consideration, the convergence rates from the upper bounds in Corollary~\ref{cor:linear_compare} are sharp. Indeed, the empirical numerical results are in near perfect agreement with our theoretical predictions for all three QoI maps of varying smoothness. In particular, the \ref{item:ee} estimator indeed converges faster than the \ref{item:ff} one for the rough derivative point evaluation QoI~\eqref{eqn:qoi_pointd}. For the point evaluation QoI~\eqref{eqn:qoi_point}, the rates are identical, as predicted. The \ref{item:ff} estimator converges at the fast $N^{-1}$ rate for the mean QoI~\eqref{eqn:qoi_mean}, which is smooth, while the \ref{item:ee} estimator converges strictly slower. Going beyond the asymptotic theory, Figure~\ref{fig:rate_compare_al0} also shows that for fixed $N$, the estimator with a faster convergence rate also has a smaller expected relative squared error in the regime of sample sizes considered, i.e., better constants.

\section{Numerical experiments}\label{sec:numerics}
We now perform numerical experiments with the proposed FNM architectures. These experiments have two main purposes. The first is to numerically implement and compare the various FNM models on several PtO maps of practical interest; the second is to qualitatively validate the theory developed in the paper for such maps. We focus on nontrivial nonlinear problems with finite-dimensional observables that define the QoI maps. Although our linear theory from Section~\ref{sec:linear} does not apply to such nonlinear problems, we still observe qualitative validation of the main implications of the linear analysis. That is, for smooth enough QoIs, full-field learning is at least as data-efficient as end-to-end learning. Unlike the theory, however, our numerical results distinguish the two approaches only by constant factors and not by the actual convergence rates.

The continuum FNM architectures from Section~\ref{sec:nm} are implemented numerically by replacing all forward and inverse Fourier series calculations with their Discrete Fourier Transform counterparts. This enables fast summation of the series \eqref{eqn:fno_Kt}, \eqref{eqn:fnf_kernel_fourier_integral}, and \eqref{eqn:fnd_kernel} with the FFT. The inner products in these formulas are also computed with the FFT. In particular, the FFT performs Fourier space operations in the set $\{k\in \Z^d \colon \norm{k}_{\ell^\infty([d];\Z)}\leq K\}$ rather than over all $k \in \Z^d$.\footnote{In all numerical experiments to follow, $d=1$ or $d=2$.} In this case, we say that the FNM architecture has $K$ modes. This is analogous to the mode truncation used in standard FNO layers (see, e.g.,~\cite{li2021fourier}). Additionally, since we work with real vector-valued functions, conjugate symmetry of the Fourier coefficients may be exploited to write the Fourier linear functional~\eqref{eqn:fnf_kernel_fourier_integral} and decoder~\eqref{eqn:fnd_kernel} layers only in terms of the real part of the coefficients appearing in the summands. We also make a minor modification to the F2V and V2V FNMs. Since the discrete implementation of $\sG$ in \eqref{eqn:fnf_kernel_fourier_integral} requires discarding the higher frequencies in the input function, we define an auxiliary map $\sW\colon h\mapsto \int_{\Td} \mathsf{NN}(h(x))\dd{x}$ that makes use of all frequencies. Here $\mathsf{NN}(\slot)$ is a one hidden layer fully-connected neural network (NN). Then we replace $\sG$ in Definition~\ref{def:FNM} by the concatenated operator $(\sG, \sW)^{\tp}$.

Given a dataset of input--output pairs \( \{( u_n, \tilde{y}_n)\}_{n=1}^N \), we train a FNM $\Psi_\theta$ taking one of the forms given in Definition~\ref{def:FNM} (with the modifications from the preceding discussion) in a supervised manner by minimizing the average relative error
\begin{equation}\label{eqn:fnm_loss}
	\frac{1}{N} \sum_{n=1}^N \frac{\norm{\tilde{y}_n - \Psi_{\theta}(u_n)}}{\norm{\tilde{y}_n}}
\end{equation}
or the average absolute squared error
\begin{equation}
    \label{eqn:fnm_sq_abs_loss}
    \frac{1}{N} \sum_{n=1}^N \norm{\tilde{y}_n - \Psi_{\theta}(u_n)}^2
\end{equation}
over the FNM's tunable parameters \( \theta \) using mini-batch SGD with the ADAM optimizer. The choice of the loss function is dependent on the underlying problem. Moreover, the norm in the preceding displays are inferred from the space that the $\tilde{y}_n$ takes values in (i.e., finite-dimensional vector or infinite-dimensional function output spaces). To avoid numerical instability in our actual computations, we add $10^{-6}$ to the denominator of the ratio in~\eqref{eqn:fnm_loss}.

\begin{remark}[data discretization error]
In addition to the discretization error introduced by the discrete and implementable realizations of the continuum FNM architectures~\cite{lanthaler2024discretization}, there is another source of discretization error due to our choice of data generation procedure. Specifically, the training and test data in this section are generated from numerical solvers that discretely approximate an underlying continuum operator at a fixed resolution. The weights of the resulting trained FNMs have a complicated dependence on this discretization error. Although simpler operator learning architectures are stable to such errors~\cite[Example 3.9, pp.~5--6]{lanthaler2023error}, no such results exist yet for neural operators. Furthermore, in line with most of the literature, the empirical convergence results that we numerically report in this section are for the test error with respect to the discretized operator or PtO map. Thus, there is an implicit assumption that this discretized operator is sufficiently resolved so that the computable but discrete test error is an accurate surrogate for the true but inaccessible test error with respect to the continuum operator. Alternative data acquisition strategies may mitigate these effects to some extent~\cite{hasani2024generating}.
\end{remark}

The numerical experiments are organized as follows. In Subsection~\ref{sec:numerics_ad}, we extract the first four polynomial moments from the solution of a velocity-parametrized 2D advection--diffusion equation. Next, Subsection~\ref{sec:numerics_airfoil} considers the flow over an airfoil modeled by the steady compressible Euler equation. The PtO map sends the shape of the airfoil to the resultant drag and lift force vector. Last, we study an elliptic homogenization problem parametrized by material microstructure in Subsection~\ref{sec:numerics_elliptic}. Here, the QoI returns the effective tensor of the material.

\subsection{Moments of an advection--diffusion model}\label{sec:numerics_ad}
Our first model problem concerns a canonical advection--diffusion PDE in two spatial dimensions. This equation often arises in the environmental sciences and is useful for modeling the spread of passive tracers (e.g., contaminants, pollutants, aerosols), especially when the driving velocity field is coupled to another PDE such as the Navier--Stokes equation. Our setup is as follows. Let $\cD=(0,1)^2$ be the spatial domain and $\mathsf{n}$ denote the unit inward normal vector to $\cD$. For a prescribed time-independent velocity field $v\colon \cD\to\R^2$, the state $\phi\colon\cD\times \R_{>0}\to\R$ solves
\begin{align}\label{eqn:advection_diffusion}
\begin{alignedat}{2}
\partial_t \phi + \nabla \cdot (v\phi) - 0.05\lap \phi &= g \qin && \cD\times\R_{>0}\,,\\
\mathsf{n}\cdot \nabla \phi &= 0 \qon && \partial\cD\times \R_{>0}\,,\\
\phi &= 0 \qon && \cD \times \{0\}\,.
\end{alignedat}
\end{align}

The time-independent source term $g$ is a smoothed impulse located at $x_0\defeq (0.2, 0.5)^\tp$ and is defined for $x\in\cD$ by
\begin{align*}
    g(x)\defeq \frac{5}{2\pi(50)^{-2}}\exp\left(-\frac{\norm{x-x_0}^2_{\R^2}}{2(50)^{-2}}\right)\,.
\end{align*}

We associate our input parameter with the velocity field $v$ appearing in~\eqref{eqn:advection_diffusion}. Our parametrization takes the form
\begin{align}\label{eqn:ad_vel}
    v=(u, 0)^\tp\,,\qw u(x_1,x_2)=3 + \sum_{j=1}^{d_{\mathrm{KL}}} \sqrt{\tau_j}z_je_j(x_1)
\end{align}
for all $x=(x_1,x_2)\in\cD$. Note that $u$ is constant in the vertical $x_2$ direction. The eigenvalues $\{\tau_j\}_{j\in\N}$ and eigenfunctions $\{e_j\}_{j\in\N}$ correspond to the Mercer decomposition of a kernel obtained by restricting a Mat\'ern covariance function over $\R$ to $(0,1)\subset \R$. The covariance function has smoothness exponent $1.5$ and lengthscale $0.25$~\cite{rasmussen2006gaussian}. We choose
\begin{align*}
    z_j\diid  \mathsf{Uniform}([-1,1]) \qfa j\in[d_{\mathrm{KL}}]\,.
\end{align*}
Thus, up to normalization constants, the velocity field \eqref{eqn:ad_vel} is the (truncated) KL expansion of a subgaussian stochastic process. We take the input to either be the full $x_1$-velocity field $u\colon \cD\to\R$ or the i.i.d. realizations $z\defeq (z_1,\ldots,z_{d_{\mathrm{KL}}})^\tp$ of the random variables that affinely parametrize $u$.

\begin{figure}[tb]
    \setlength\tabcolsep{3pt} 
    \centering
    \begin{tabular}{@{} r M{0.18\linewidth} M{0.18\linewidth} @{}}
    & Velocity Input & State Output\\
    \begin{subfigure}{0.23\linewidth} \caption{$d_{\mathrm{KL}}=2$}\label{subfig:vis_ad_2} \end{subfigure} 
      & \includegraphics[width=\hsize]{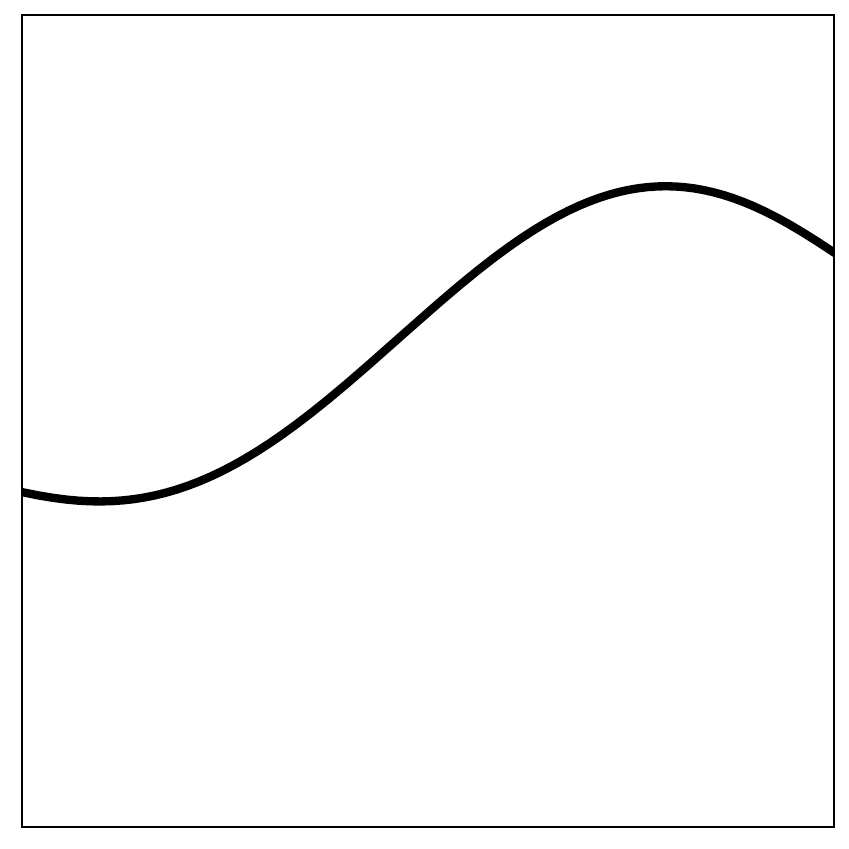}
      & \includegraphics[width=\hsize]{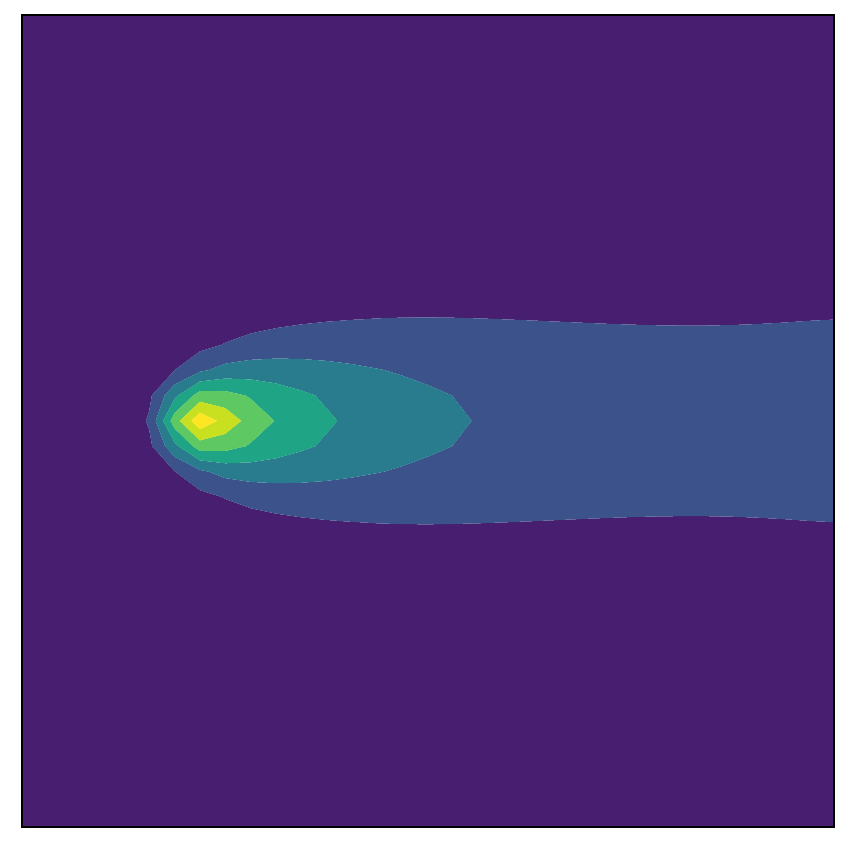}\\ \addlinespace
    \begin{subfigure}{0.21\linewidth} \caption{$d_{\mathrm{KL}}=20$}\label{subfig:vis_ad_20} \end{subfigure} 
      & \includegraphics[width=\hsize]{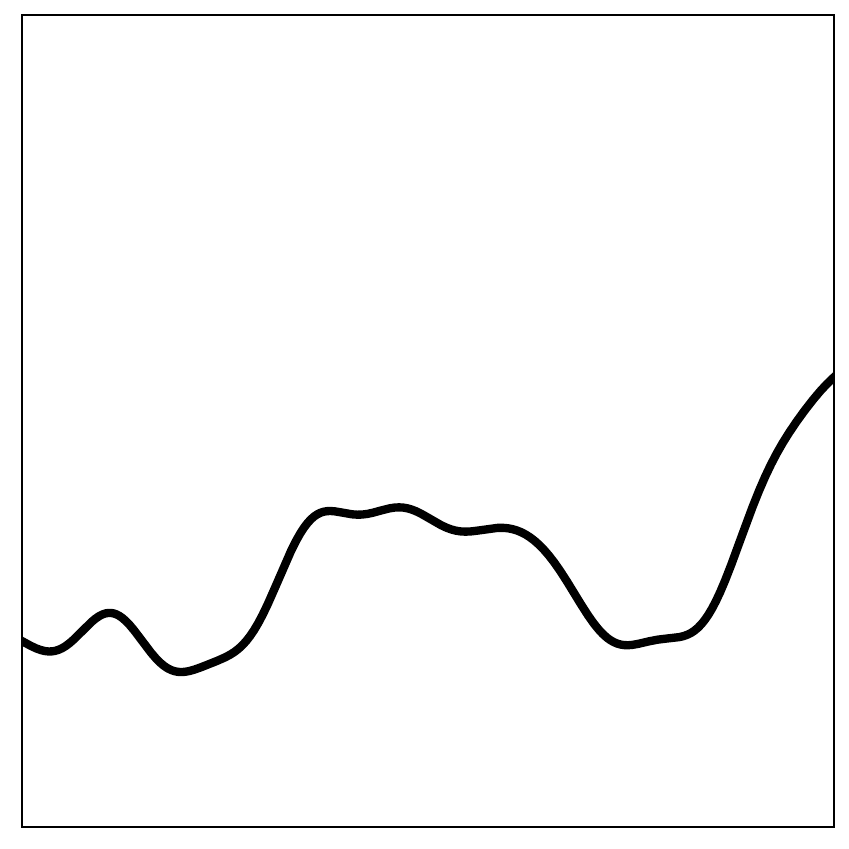}    
      & \includegraphics[width=\hsize]{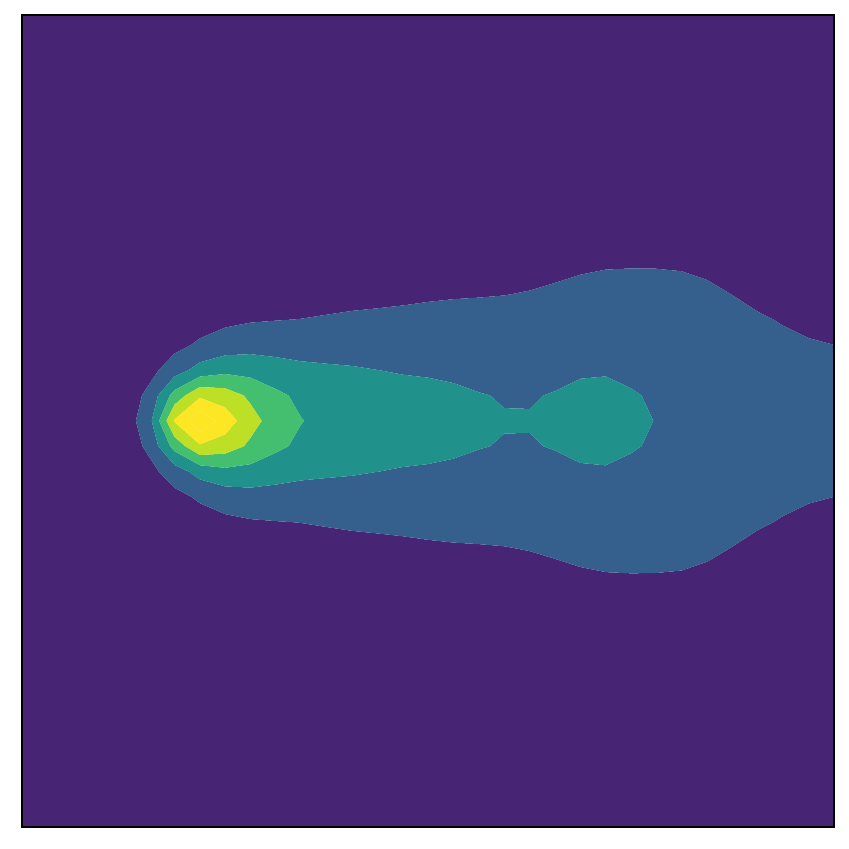}\\ \addlinespace
    \begin{subfigure}{0.18\linewidth} \caption{$d_{\mathrm{KL}}=1000$}\label{subfig:vis_ad_1000} \end{subfigure} 
      & \includegraphics[width=\hsize]{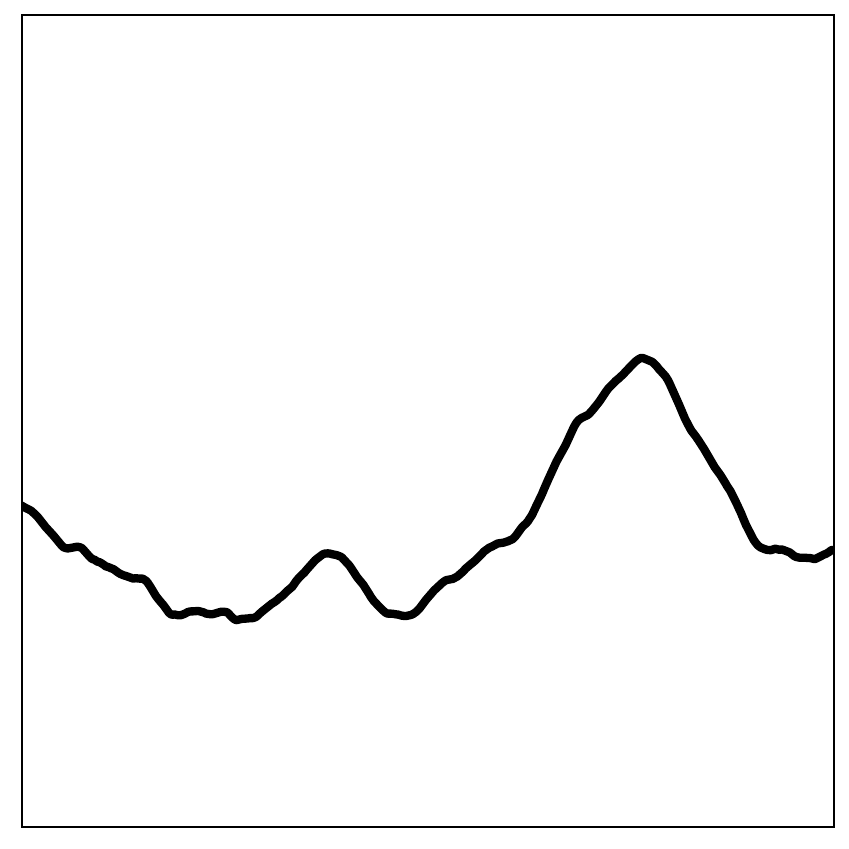} 
      & \includegraphics[width=\hsize]{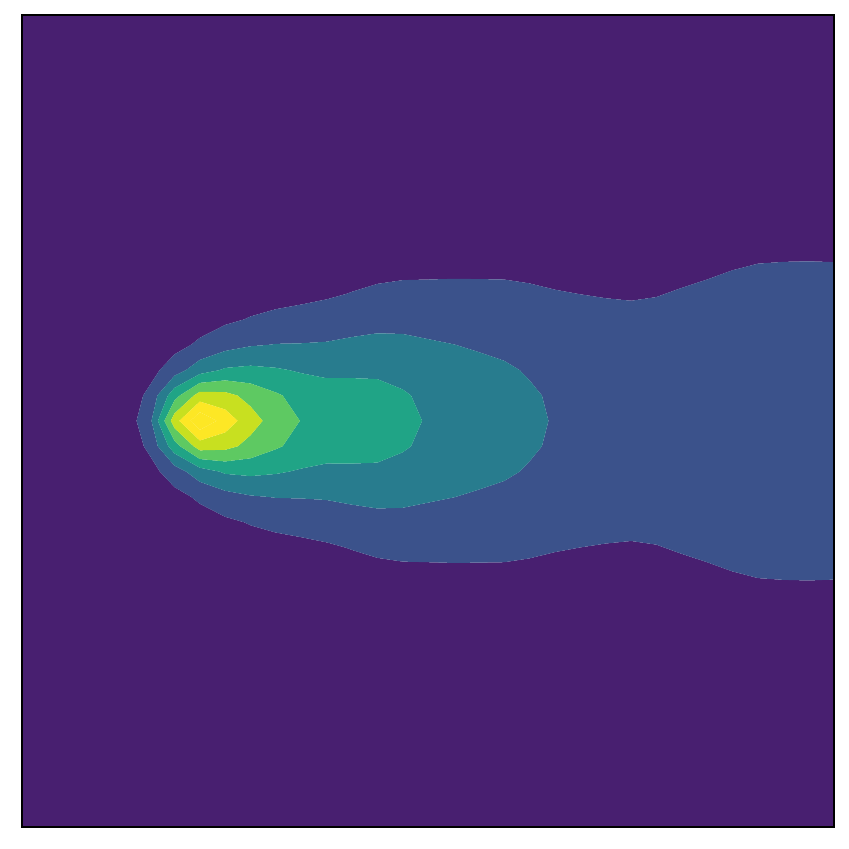}
    \end{tabular}
    \caption{Visualization of the velocity-to-state map for the advection--diffusion model. Rows denote the dimension of the KL expansion of the velocity profile and columns display representative input and output fields.}
    \label{fig:vis_ad}
\end{figure}

Define the \emph{nonlinear} QoI map $\qd \colon L^4(\cD;\R)\to\R^4$ as follows. First, for any $h\in L^2(\cD;\R)$, let
\begin{align}\label{eqn:ad_mean_var}
    \bar{m}(h)\defeq \int_{\cD} h(x)\dd{x}\qa \bar{s}(h)\defeq \left(\int_{\cD} \abs{h(x) - \bar{m}(h)}^2\dd{x}\right)^{1/2}
\end{align}
denote the mean and variance of the pushforward of the uniform distribution on $\cD=(0,1)^2$ under $h$, respectively. Then $\qd=(\qd_1,\qd_2,\qd_3,\qd_4)^\tp$ is given by
\begin{align}\label{eqn:ad_qoi}
    h\mapsto \qd(h)&\defeq 
        \begin{pmatrix}
           \bar{m}(h) \\[1.5mm]
           \bar{s}(h) \\[1.5mm]
           \bar{s}(h)^{-3}\int_{\cD}\bigl(h(x)-\bar{m}(h)\bigr)^3\dd{x} \\[1.5mm]
           -3 + \bar{s}(h)^{-4}\int_{\cD}\abs[\big]{h(x)-\bar{m}(h)}^4\dd{x}
         \end{pmatrix}
         \,.
\end{align}
Hence, $\qd_1$ is the mean, $\qd_2$ the standard deviation, $\qd_3$ the skewness, and $\qd_4$ the excess kurtosis. Our goal is to build FNM surrogates for the PtO map that sends the input representation (either the full velocity field or its finite number of i.i.d. coefficients) to the QoI values of the state $\phi$ at final time $t=3/4$ (see Figure~\ref{fig:vis_ad}). Therefore, we train FNMs to approximate each of the following ground truth maps:
\begin{align*}
\Psi^\dagger_{\mathrm{F2F}} \colon u&\mapsto \phi\big|_{t=3/4}\,,\\
\Psi^\dagger_{\mathrm{F2V}} \colon u&\mapsto \qd\Bigl(\phi\big|_{t=3/4}\Bigr)\,,\\
\Psi^\dagger_{\mathrm{V2F}} \colon z&\mapsto \phi\big|_{t=3/4}\,, \qa\\
\Psi^\dagger_{\mathrm{V2V}} \colon z&\mapsto \qd\Bigl(\phi\big|_{t=3/4}\Bigr)\,.
\end{align*}

The training data is obtained by solving \eqref{eqn:advection_diffusion} with a second-order Lagrange finite element method on a mesh of size $32\times 32$ and Euler time step $0.01$. For each $d_{\mathrm{KL}}\in\{2,20,1000\}$, we generate $10^4$ i.i.d. data pairs for training, $1500$ pairs for computing the test error (which is \eqref{eqn:fnm_loss} over the $1500$ test pairs instead of over the $N$ training pairs), and $500$ pairs for validation.
All FNM models with 2D spatial input or output functions use $12$ modes per dimension and a channel width of $32$. For the V2V-FNM, we use a 1D latent function space with $12$ modes and channel width of $96$.
We compare all FNM models to a standard fully-connected NN with $3$ layers and constant hidden width $2048$.
These architecture settings were selected based on a hyperparameter search over the validation dataset for $d_{\mathrm{KL}}=1000$ that mimics the parameter complexity experiments in \cite{bhattacharya2023learning,lanthaler2023nonlocal}. The models are trained on the relative loss~\eqref{eqn:fnm_loss} for $ 500$ epochs in $L^2$ output space norm for functions and Euclidean norm for vectors. The optimizer settings include a minibatch size of $20$, weight decay of $10^{-4}$, and an initial learning rate of $10^{-3}$ which is halved every $100$ epochs. We train $5$ i.i.d. realizations of the models for various values of $N$ and $d_{\mathrm{KL}}$ and report the results in Figure~\ref{fig:data_ad}.

\begin{figure}[tb]
        \centering
        \begin{subfigure}[]{0.3375\textwidth}
                \centering
                \includegraphics[width=\textwidth]{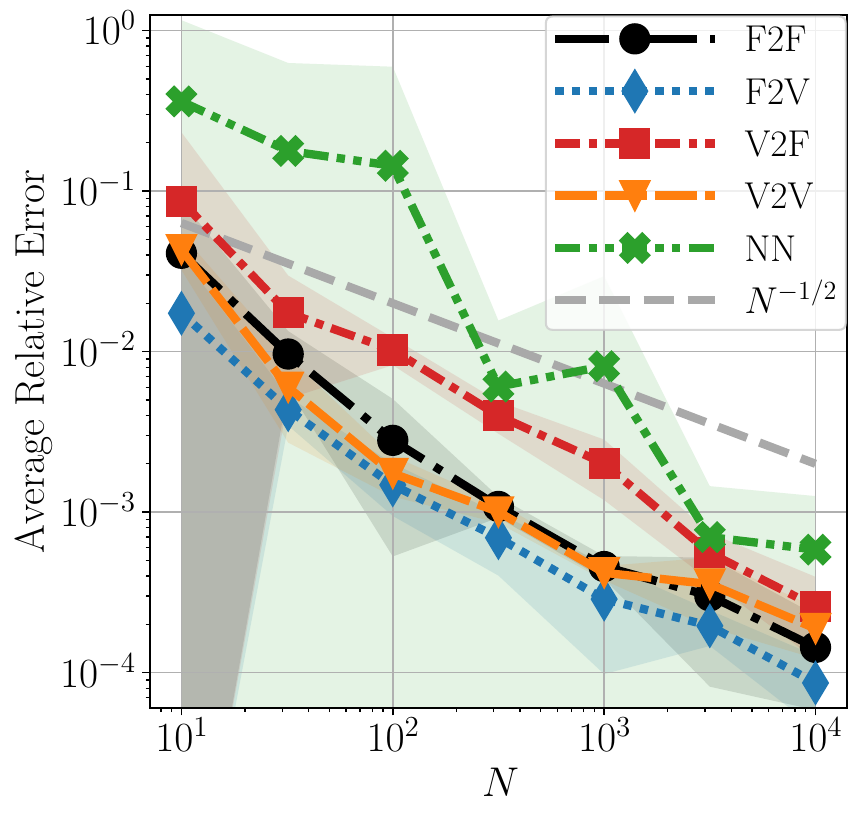}
                \caption{$d_{\mathrm{KL}}=2$}
                \label{subfig:data_ad_d2}
        \end{subfigure}
        \hfill%
        \begin{subfigure}[]{0.32\textwidth}
                \centering
                \includegraphics[width=\textwidth]{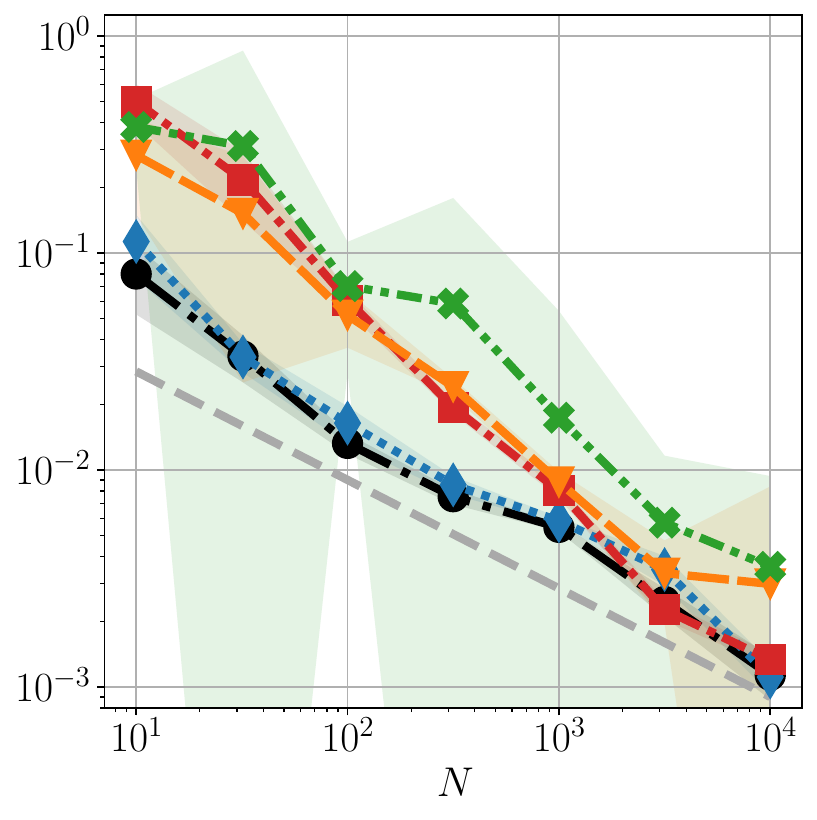}
                \caption{$d_{\mathrm{KL}}=20$}
                \label{subfig:data_ad_d20}
        \end{subfigure}
        \hfill%
        \begin{subfigure}[]{0.32\textwidth}
                \centering
                \includegraphics[width=\textwidth]{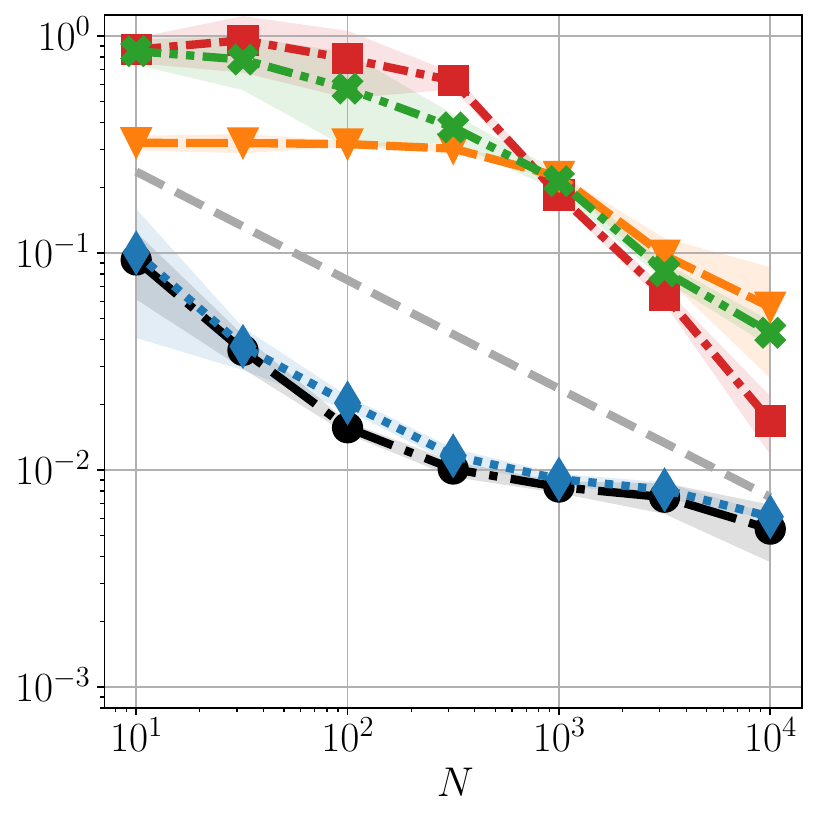}
                \caption{$d_{\mathrm{KL}}=1000$}
                \label{subfig:data_ad_d1000}
        \end{subfigure}
        \caption{Empirical sample complexity of FNM and NN architectures for the advection--diffusion PtO map (note that Figure~\ref{subfig:data_ad_d2} has a different vertical axis range). The shaded regions denote two standard deviations away from the mean of the test error over $5$ realizations of the random training dataset indices, batch indices during SGD, and model parameter initializations.}
        \label{fig:data_ad}
\end{figure}

Figure~\ref{fig:data_ad} reveals several interesting trends. In general, training models to emulate the advection--diffusion PtO map with finite-dimensional vectors as input is more difficult than adopting function space input variants of the problem. The difficulty is further exacerbated as the dimension of the input vector (here, $d_{\mathrm{KL}}$) increases. We hypothesis that this gap in performance would reduce if the vector input models received the weighted KL coefficients $\{\sqrt{\tau_j}z_j\}$ as input instead of the i.i.d. sequence $\{z_j\}$. This way the model would have access to decay information and hence an ordering of the coefficients. The standard finite-dimensional NN performs poorly across all KL expansion dimensions. The training of the NN is also quite erratic, as evidenced by the large green shaded regions indicating large variance over multiple training runs. The output space seems to play less of a role than the input space. Indeed, the F2F and F2V FNMs with function space inputs generally achieve the lowest test error regardless of $N$ and $d_{\mathrm{KL}}$.
The full-field F2F method slightly out performs the end-to-end F2V method by a small constant factor (except for when $d_{\mathrm{KL}}=2$). Since $\qd$ is a smoothing QoI due to its integral definition, this observation aligns with the theoretical insights from Subsection~\ref{sec:linear_compare}.
The fast convergence of some of the FNM models, especially for the low-dimensional cases $d_{\mathrm{KL}}=2$ and $d_{\mathrm{KL}}=20$, could potentially be explained by the lack of noise in the data, the smoothness of the QoIs, and the nonconvexity of the training procedure. When $d_{\mathrm{KL}}=1000$, the problem is essentially infinite-dimensional. The function space input FNMs (F2F and F2V) exhibit a nonparametric decay of test error as expected.

\subsection{Aerodynamic force exerted on an airfoil}\label{sec:numerics_airfoil}
Consider the following steady compressible Euler equation applied to an airfoil problem (see Figure~\ref{fig:NACA-Mesh}), as introduced in \cite{li2022fourier}:
\begin{equation}\label{eqn:euler}
\begin{aligned}
 \nabla \cdot (\rho v) = 0\,,
\\
 \nabla \cdot (\rho vv^\tp + p \Id_{\R^2}) = 0\,, 
\\
 \nabla \cdot  \bigl( (E + p) v \bigr) = 0\,.
\end{aligned}
\end{equation}
Here $\rho$ is the fluid density, $v$ is the velocity vector, $p$ is the pressure, and $E$ is the total energy. Equation~\eqref{eqn:euler} is equipped with the following far-field boundary conditions: 
$\rho_{\infty} = 1$, $p_{\infty}  = 1$,  $M_{\infty} = 0.8$, and $\mathsf{AoA} = 0$,
where $M_{\infty}$ is the Mach number and $\mathsf{AoA}$ is the angle of attack. This setup indicates that the flow condition is in the transonic regime. Additionally, the no-penetration condition $v \cdot  \mathsf{n} = 0$ is imposed at the airfoil, where $ \mathsf{n} $ represents the inward-pointing normal vector to the airfoil. Additional mathematical details about the setup may be found in \cite{lye2021multi,lye2020deep,lye2021iterative,mishra2021enhancing}.

\begin{figure}[bt]
\begin{center}
  \includegraphics[width=0.95\textwidth]{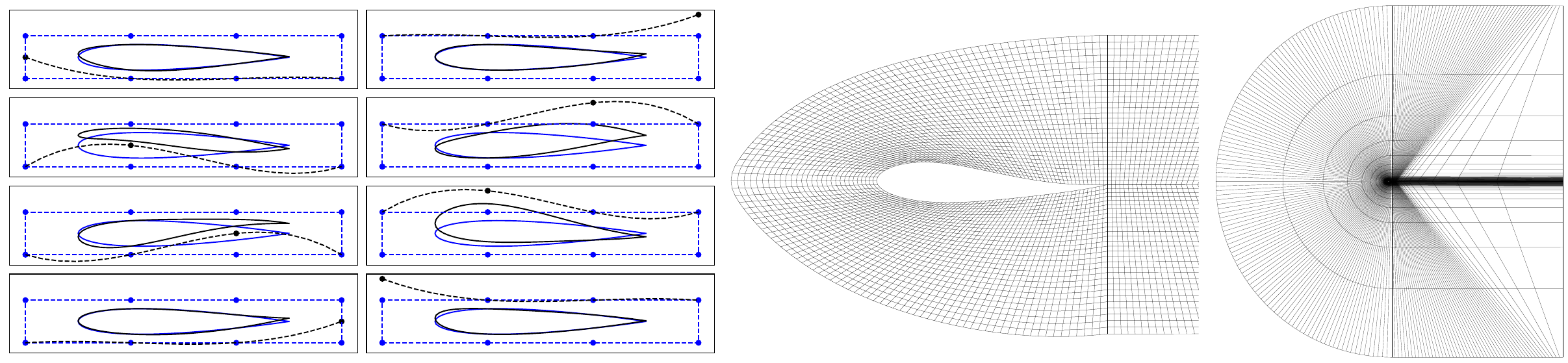}
  \caption{Flow over an airfoil. From left to right: visualization of the cubic design element and different airfoil configurations, guided by the displacement field of the control nodes;
  a close-up view of the $C$-grid surrounding the airfoil;
  the physical domain discretized by the $C$-grid.
  }\label{fig:NACA-Mesh}
  \end{center}
\end{figure}

In this context, we are interested in solving the aforementioned 2D Euler equation to predict the drag and lift performance of different airfoil shapes. Building fast yet accurate surrogates for this task facilities aerodynamic shape optimization~\cite{o2021adaptive,shukla2023deep} for various design goals, such as maximizing the lift to drag ratio~\cite{li2022fourier}. The drag and lift QoIs, which only depend on the pressure on the airfoil, are given by the force vector
\begin{align}\label{eqn:qoi_airfoil}
     (\mathsf{Drag},\,\mathsf{Lift})^\tp = \oint_{\cA} p \mathsf{n} \dd{s} \in \R^2\,.
\end{align}
Here $\cA$ denotes the closed curve defined by the union of the upper and lower surfaces of the airfoil. Different airfoil shapes are generated following the design element approach~\cite{farin2014curves} (Figure~\ref{fig:NACA-Mesh}).
The initial NACA-0012 shape is embedded into a ``cubic'' design element featuring $8$ control nodes, and the initial shape is morphed into a different one following the displacement field of the control nodes of the design element. The displacements of control nodes are restricted to the vertical direction only. Consequently, the intrinsic dimension of the input is $7$, as displacing all nodes in the vertical direction by a constant value does not change the shape of the airfoil.

To generate the training data, we used the traditional second-order finite volume method with the implicit backward Euler time integrator. The process begins by generating a new airfoil shape. Subsequently, a $C$-grid mesh~\cite{steger1980generation} consisting of $221 \times 51$ quadrilateral elements is created around the airfoil with adaptation near the airfoil. In total, we generated $2000$ training data and $400$ test data with the vertical displacements of each control node being sampled from a uniform distribution $\textsf{Uniform}([-0.05, 0.05])$.

\begin{figure}[tb]
\begin{center}
  \includegraphics[width=0.95\textwidth]{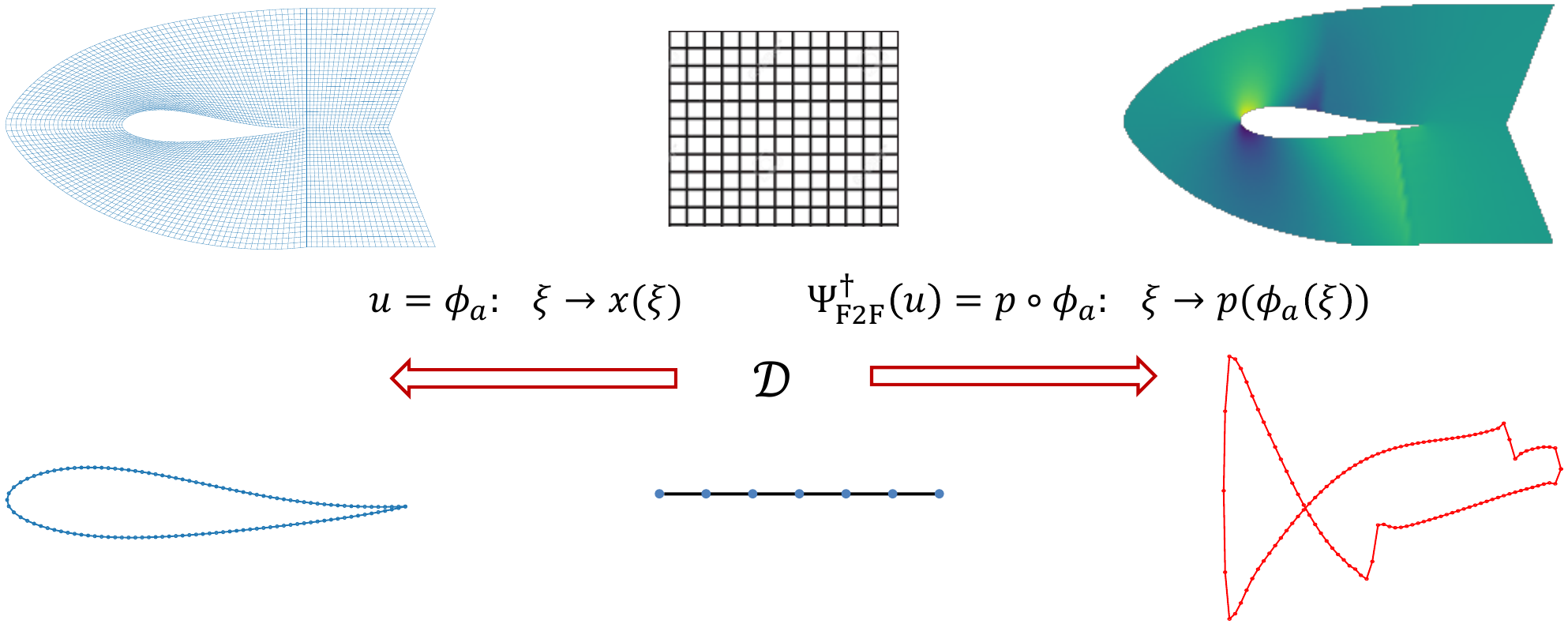}
  \caption{Flow over an airfoil. The 1D (bottom) and 2D (top) latent spaces are illustrated at the center; the input functions $\phi_a$ encoding the irregular physical domains, are shown on the left; and the output functions $p\circ\phi_a$ representing the pressure field on the irregular physical domains, are depicted on the right.}\label{fig:NACA-Operator}
  \end{center}
\end{figure}

Next, we will define the operator learning problem (see Figure~\ref{fig:NACA-Operator}). In the 2D setting, we aim to learn the entire pressure field. Let $\cD_a$ represent the irregular physical domain parametrized by $a$, indicating the shape of the airfoil. The domain $\cD_a$ is discretized by a structured $C$-grid~\cite{steger1980generation}. We introduce a latent space $\cD = [0,1]^2$ and the deformation map $\phi_a\colon \xi \rightarrow x(\xi)$ between  $\cD$ and $\cD_a$. Here the deformation map has an analytical format and maps the uniform grid in $\cD$ to the $C$-grid in $\cD_a$.
Subsequently, we formulate the operator learning problem in the latent space as
\begin{equation}
\label{eq:airfoil-operator-learning}
    \Psi^\dagger_{\mathrm{F2F}}\colon  \phi_a \to p\circ \phi_a\,. 
\end{equation} 
In this equation, the deformation map $\phi_a$ is a function defined in $\cD$, and $p\circ \phi_a$ represents the pressure function defined in $\cD$. 
As  mentioned previously, both lift and drag depend solely on the pressure distribution over the airfoil. 
Hence, we can alternatively formulate the learning problem in the 1D setting by focusing solely on learning the pressure distribution over the airfoil. We construct a one-dimensional latent space $\cD = [0,1]$ and also denote the deformation map as $\phi_a\colon \xi \rightarrow x(\xi)$ mapping from $\cD$ to the shape of the airfoil. The corresponding operator learning problem in this 1D setting has the same form as~\eqref{eq:airfoil-operator-learning}. The ground truth maps $\Psi^\dagger_{\mathrm{F2V}}$, $\Psi^\dagger_{\mathrm{V2F}}$, and $\Psi^\dagger_{\mathrm{V2V}}$ are defined similarly, mapping either the deformation function $\phi_a$ or the $7$-dimensional control node vector input to the pressure function or the QoI \eqref{eqn:qoi_airfoil} itself. We use all four variants of the FNM architectures and a finite-dimensional NN to approximate these maps from data.

For each sample size $N$, five i.i.d. realizations of the models are trained on the relative loss~\eqref{eqn:fnm_loss} for $ 2000$ epochs in $L^2$ output space norm for functions and Euclidean norm for vectors.
All FNM models use $4$ hidden layers, $12$ modes per dimension, and a channel width of $128$. We compare these models to a standard fully-connected NN with $4$ layers and a hidden width of $128$. In the case of  FNM models, we observe that learning in the 1D setting consistently outperforms the 2D setting across all sizes of training data. Therefore, we only present results for the 1D setting. Moreover, this set of architectural hyperparameters with a large channel width of $128$ in general outperforms other hyperparameter settings.

\begin{figure}[tb]
\begin{center}
  \includegraphics[width=0.5\textwidth]{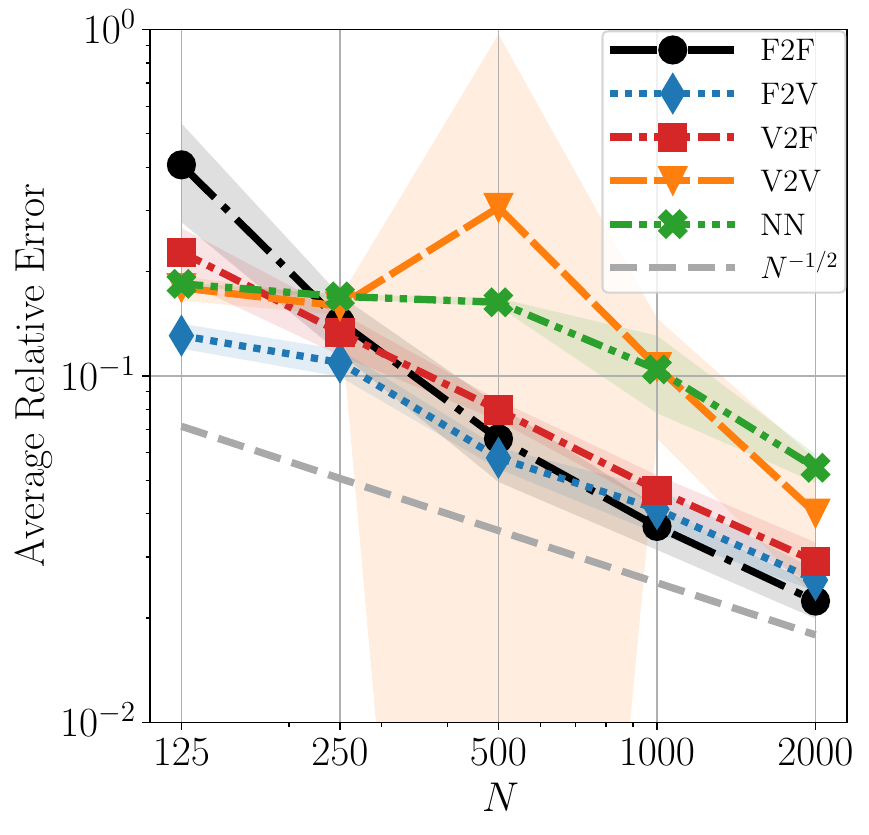}
  \caption{Flow over an airfoil. Comparative analysis of relative test error versus data size for the FNM and NN approaches. The shaded regions denote two standard deviations away from the mean of the test error over $5$ realizations of the batch indices during SGD and model parameter initializations.}\label{fig:NACA-data-error}
  \end{center}
\end{figure}

Figure~\ref{fig:NACA-data-error} contains the results and reveals several trends. As the data volume $N$ increases, all error curves decay at an algebraic rate that is slightly faster than $N^{-1/2}$. This may be due to the small sample sizes considered (under $2000$ data pairs) or, especially since the training data is noise-free, could be evidence of a data-driven ``superconvergence'' effect similar to that observed for QoI computations in adjoint methods for PDEs~\cite{giles2002adjoint}.
Overall, emulating PtO maps by training models with finite-dimensional vectors as both input and output (V2V and NN) is more challenging for this problem than adopting function space variants (F2F, F2V, V2F). The standard finite-dimensional NN performs similarly to V2V.

\subsection{Effective tensor for a multiscale elliptic equation}\label{sec:numerics_elliptic}
This example considers an equation that arises in elasticity in computational solid mechanics and relates the material properties on small scales to the effective property on a larger scale.
Formally, we consider the following linear multiscale elliptic equation on a bounded domain $\cD \subset \R^2$:
\begin{align}\label{eqn:ms_ellip}
\begin{alignedat}{2}
-\nabla \cdot\left(A^{\eps}\nabla u^{\eps}\right) &= g \qin && \cD\,,\\
u^\eps & = 0 \qon && \p\cD\,.  
\end{alignedat}
\end{align}
Here $A^{\eps}$ is given by $x\mapsto A^{\eps}(x) = A\left(\frac{x}{\eps}\right)$ for some $A\colon \T^2 \to \R^{2 \times 2}_{{\rm sym}, \succ 0}$ which is $1$-periodic and positive definite. The source term is $g$. This equation contains fine-scale dependence through $A^{\epsilon}$, which may be computationally expensive to evaluate without taking advantage of periodicity. The method of homogenization allows for elimination of the small scales in this manner and yields the homogenized equation
\begin{align}\label{eqn:homogenized}
\begin{alignedat}{2}
-\nabla \cdot\left(\overline{A}\nabla u \right) &= g \qin && \cD\,,\\
u & = 0 \qon && \p\cD\,, 
\end{alignedat}
\end{align}
where $\overline{A}$ is given by 
\begin{equation}
\label{eq:cell}
    \overline{A} = \int_{\T^2}\left(A(y) + A(y) \nabla \rchi(y)^\tp\right) \dd{y}
\end{equation}
and $\rchi\colon \T^2 \to \R^2$ solves the cell problem
\begin{align}\label{eqn:cellprob}
-\nabla \cdot\bigl( (\nabla \rchi) A\bigr) &= \nabla \cdot A \qin \T^2\,,\\
\int_{\T^2}\rchi(y)\dd{y} &= 0 \qa \text{$\rchi$ is $1$-periodic}\,.
\end{align}
For $0 < \eps \ll 1$, the solution $u^{\eps}$ of \eqref{eqn:ms_ellip} is approximated by the solution $u$ of \eqref{eqn:homogenized}. The error between the solutions converges to zero as $\epsilon \to 0$ \cite{bensoussan2011asymptotic,pavliotis2008multiscale}. 

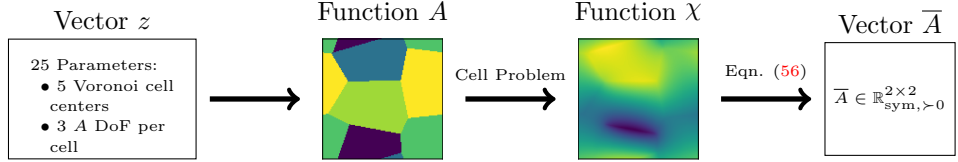
\begin{figure}[tb]
\centering
  \begin{tikzpicture}[
    node distance=1.55cm,
    block/.style={rectangle, draw, minimum size=1.5cm}
  ]

    \node[block,label = above:Vector $z$] (figure1) {
      \begin{tabular}{l} \vspace{-0.15cm}
        \tiny{$25$ Parameters:} \\ \vspace{-0.15cm}
        \tiny{\hspace{0.5em}$\bullet$ $5$ Voronoi cell} \\ \vspace{-0.15cm}
        \tiny{\hspace{1em}centers} \\ \vspace{-0.15cm}
        \tiny{\hspace{0.5em}$\bullet$ $3$ $A$ DoF per} \\ \vspace{-0.15cm} \tiny{\hspace{1em}cell}
      \end{tabular}
    };

    \node[rectangle, minimum size=1.5cm, right=of figure1, label=above:Function $A$] (figure2) {\includegraphics[width=1.6cm]{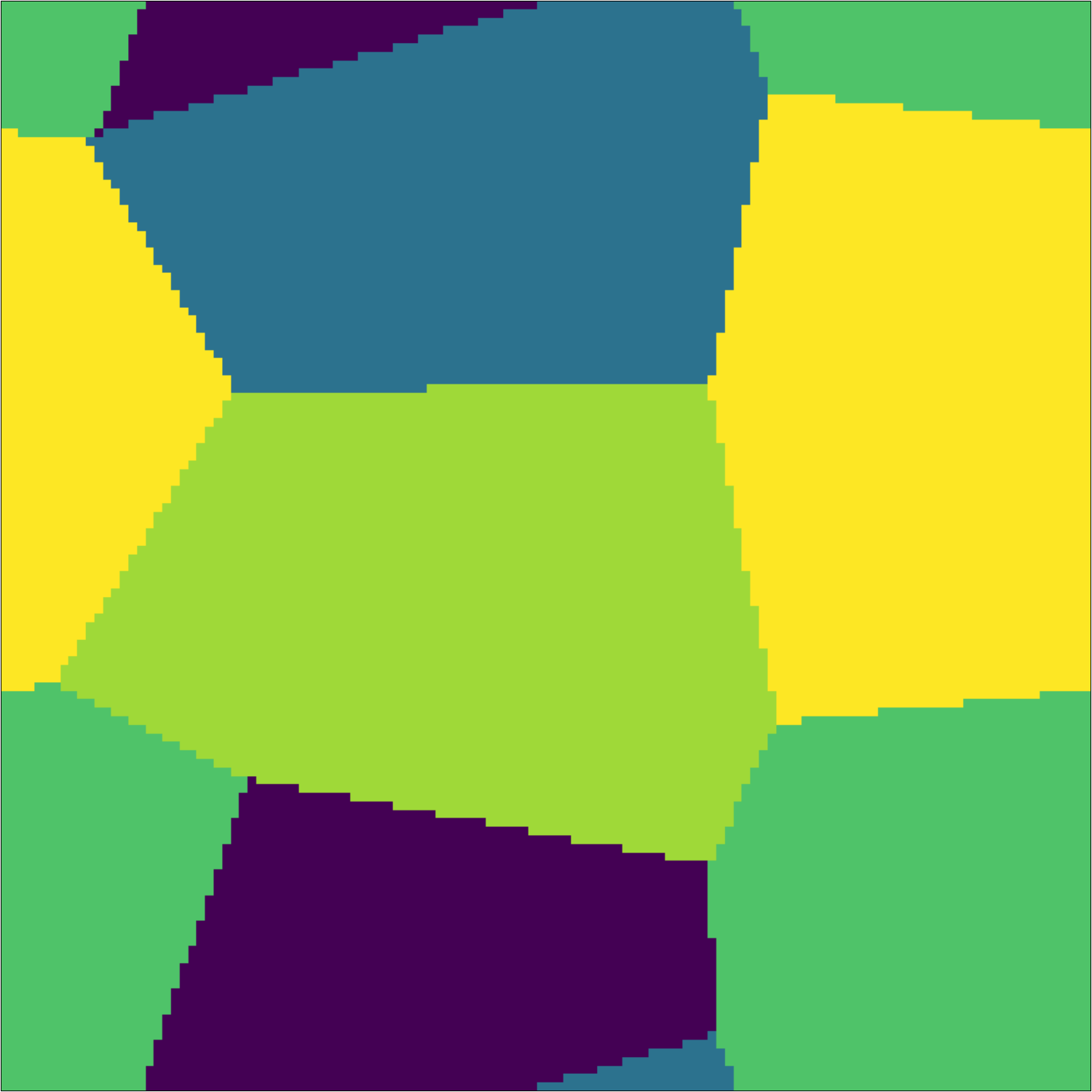}};

    \node[rectangle, minimum size=1.5cm, right=of figure2, label = above:Function $\rchi$] (figure3) {\includegraphics[width=1.6cm]{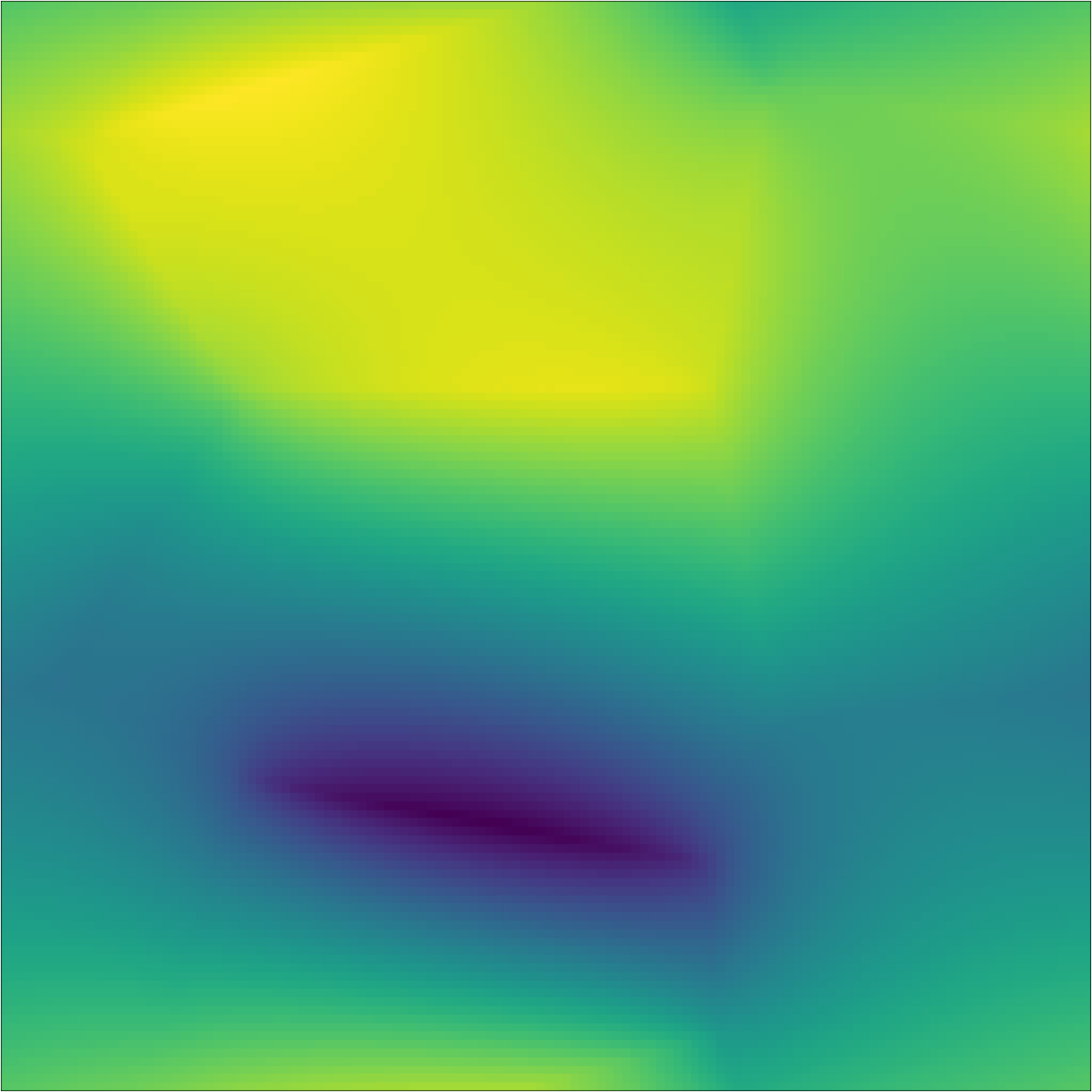}};

    \node[block, right=of figure3,label = above:Vector $\overline{A}$] (figure4) {\tiny{$\overline{A} \in \R^{2\times 2}_{{\rm sym}, \succ 0}$}};
    
    \draw[->, shorten <=5pt, shorten >=5pt, line width=2pt] (figure1.east) -- (figure2.west);
    \draw[->, shorten <=5pt, shorten >=5pt, line width=2pt] (figure2.east) -- node[midway, above=0.1cm] {\tiny{Cell Problem}}(figure3.west);
    \draw[->, shorten <=5pt, shorten >=5pt, line width=2pt] (figure3.east) -- node[midway, above=0.1cm] {\tiny{Eqn. \eqref{eq:cell}}}(figure4.west);

  \end{tikzpicture}
\caption{Diagram showing the homogenization experiment ground truth maps. The function $A$ is parametrized by a finite vector $z$. The quantity of interest $\overline{A}$ \eqref{eq:cell} is computed from both the material function $A$ and the solution $\rchi$ to the cell problem \eqref{eqn:cellprob}. Note that both $A$ and $\rchi$ are functions on the torus $\T^2$.} \label{fig:hom_map}
\end{figure}

The bottleneck step in obtaining the effective tensor $\overline{A}$, which is our QoI, is solving the cell problem \eqref{eqn:cellprob}. Learning the solution map $A \mapsto \rchi$ in \eqref{eqn:cellprob} corresponds to the F2F setting and is explored in detail in \cite{bhattacharya2023learning}. Alternately, one could learn the effective tensor $\overline{A}$ directly using the F2V-FNM architecture to approximate $A \mapsto \overline{A}$. Furthermore, though $A$ is a function from $\T^2$ to $\R^{2 \times 2}_{{\rm sym}, \succ 0}$, in certain cases it may have an exact finite vector parametrization. One example of this case is finite piecewise-constant Voronoi tessellations; $A$ takes constant values on a fixed number of cells, and the cell centers uniquely determine the Voronoi geometry. Denoting these parameters as $z \in \R^{d_u}$ for appropriate $d_u \in \N$, one could also learn the V2F map $z \to \rchi$ or the V2V map $z \to \overline{A}$. In this experiment, we compare the error in the QoI $\overline{A}$ using all four methods. A visualization of the possible maps is shown in Figure \ref{fig:hom_map}. Since our example is in two spatial dimensions, the $5$ Voronoi cell centers have two components each. The symmetry of $A$ yields three degrees of freedom (DoF) on each Voronoi cell. Altogether, this yields $25$ parameters that comprise the finite vector input.

For training, we use the absolute squared loss in \eqref{eqn:fnm_sq_abs_loss} with the $H^1$ norm for function output and Frobenius norm for vector output. Test error is also evaluated using this metric. Data are generated with a finite element solver using the method described in \cite{bhattacharya2023learning}; both $A$ and $\rchi$ are interpolated to a $128\times 128$ grid, and the Voronoi geometry is randomly generated for each sample. The test set size is $500$. Each map uses hyperparameters obtained via a grid search. For F2F, F2V, V2F, and V2V, the number of modes are $18$, $12$, $12$, and $18$, and the channel widths are $64$, $96$, $96$, and $64$, respectively. The fully-connected NN used as a comparison has a channel width of $576$ and $2$ hidden layers. As a consequence, all methods have a fixed model size of modes times width equaling $1152$. 

\begin{figure}[tb]
\begin{center}
  \includegraphics[width=0.5\textwidth]{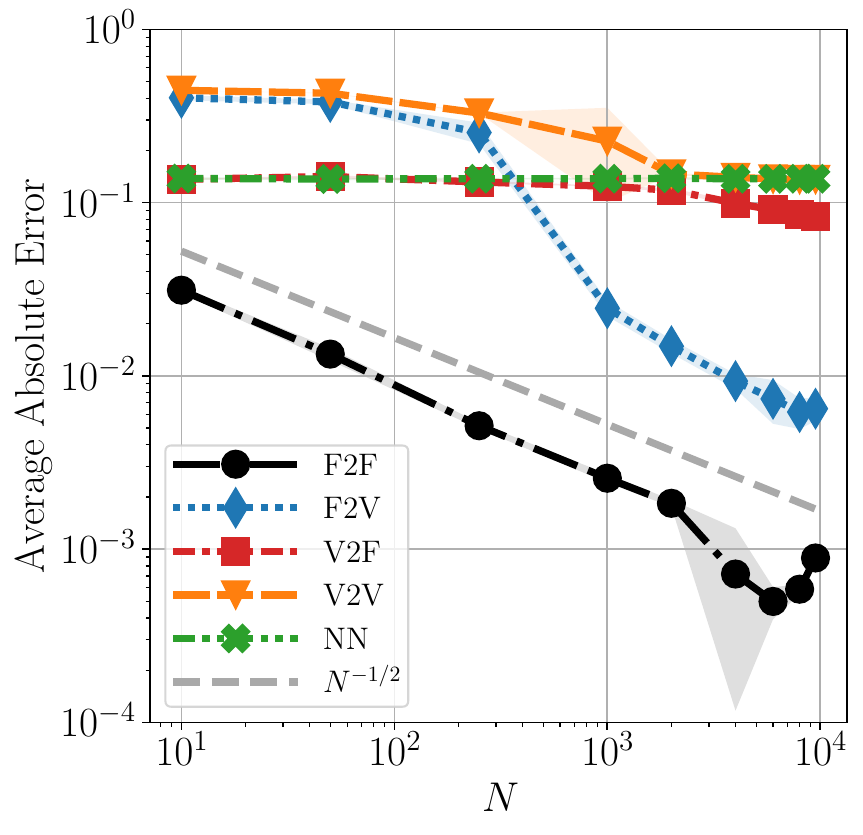}
  \caption{Elliptic homogenization problem. Absolute $\overline{A}$ error in the Frobenius norm versus data size for the FNM and NN architectures. The shaded regions denote two standard deviations away from the mean of the test error over $5$ realizations of batch indices during SGD and model parameter initializations.}\label{fig:homogenization_abs_err}
  \end{center}
\end{figure}

The results for the homogenization experiment in Figure~\ref{fig:homogenization_abs_err} reinforce the theoretical intuition from Section \ref{sec:linear} that learning with vector data results in higher error than learning with function data. Both the F2F and the F2V models approximately track the $N^{-1/2}$ rate, where $N$ is the number of training data. On the other hand, the V2V model and NN model fail to attain this rate and saturate at the same level of roughly $10\%$ error. The V2F map does achieve a slightly faster error decay rate than the V2V architecture for large enough sample sizes $N$, but it does not approach the $N^{-1/2}$ rate obtained by the F2F and F2V models. These rate differences occur when there is a difference in input dimension. On the other hand, for a difference in output dimension, while both the F2F and F2V models reach roughly the same convergence rate, the F2V error remains an order of magnitude higher than the F2F error.  We note that when measuring performance with relative test error instead, the qualitative behavior of Figure~\ref{fig:homogenization_abs_err} remains the same.

\section{Conclusion}\label{sec:conclusion}
This paper proposes the Fourier Neural Mappings (FNMs) framework as an operator learning method for approximating parameter-to-observable (PtO) maps with finite-dimensional vector inputs or outputs, or both. Universal approximation theorems demonstrate that FNMs are well-suited for this task.
Of central interest is the setting in which the PtO map factorizes into a vector-valued quantity of interest (QoI) map composed with a forward operator mapping between two function spaces. For this setting, the paper introduces the end-to-end \ref{item:ee} and full-field \ref{item:ff} learning approaches. The \ref{item:ee} approach directly estimates the PtO map from its own input--output pairs, while the \ref{item:ff} approach estimates the forward map first and then plugs this estimator into the QoI. The main theoretical results of the paper establish sample complexity bounds for Bayesian nonparametric regression of linear functionals with the \ref{item:ee} and \ref{item:ff} methods. The analysis reveals useful insights into how the smoothness of the QoI influences data efficiency. In particular, \ref{item:ff} is superior to \ref{item:ee} for smooth QoIs in this setting. The situation reverses for QoIs of low regularity. A simulation study verifies these findings. Finally, the paper implements the FNM architectures for three nonlinear problems arising from environmental science, aerodynamics, and materials modeling. The numerical results support the linear theory and extend beyond it by revealing the supremacy of function space representations of the input space over analogous finite-dimensional vector parametrizations.

Several avenues for future work remain open. One way to understand the data efficiency of the \ref{item:ee} and \ref{item:ff} learning approaches beyond the specific Bayesian linear estimators that this paper analyzes would involve the development of fundamental lower bounds. Besides a few recent works \cite{adcock2024optimal,jin2022minimax,kovachki2024data,lanthaler2023operator,lanthaler2023curse,li2023towards}, there has been little attention on minimax lower bounds and (statistical) optimality for operator learning. The statistical theory in the present paper fixes an infinite-dimensional input space and studies the influence of the output space (being either one-dimensional or infinite-dimensional). The derivation of similar insights to those in Subsection~\ref{sec:linear_compare} for PtO map learning with vector-to-function estimators could help explain the strong influence of the input space observed in the numerical experiments from Section~\ref{sec:numerics}. The scalar input and infinite-dimensional output case is partially addressed by \cite{reimherr2023optimal}. Furthermore, it remains to be seen whether the theory developed in the present paper for the linear functional setting can extend to certain classes of nonlinear maps and QoIs, perhaps by linearizing the maps in an appropriate manner. For a specific nonlinear map, relevant work in \cite{scarabosio2022deep} studies the approximation of nonsmooth QoIs.
On the practical side, it is of interest to further explore architectural improvements for the various FNMs and in particular whether the latent function space introduced by the vector-to-vector FNM~\ref{enum:V2V} can actually lead to improved performance over standard finite-dimensional neural networks.


\appendix

\section*{Data and code availability}
Links to datasets and all code used to produce the numerical results and figures in this paper are available at
\begin{center}
    \url{https://github.com/nickhnelsen/fourier-neural-mappings}\,.
\end{center}

\section{Additional variants of the sample complexity theorems}\label{appx:thm_full}
This appendix contains two additional theorems that are more general than their counterparts in Section~\ref{sec:linear} of the main text. While the main theorems as presented in Subsection~\ref{sec:linear_main} are sufficient to convey the primary message of that part of this paper, we present the extra theorems here for completeness, as the results may be of independent interest.

First, we state the general convergence theorem for the end-to-end posterior mean estimator $\bar{f}^{(N)}$ from~\eqref{eqn:posterior_ee}. The result is valid for any choice of prior covariance operator eigenvalue decay exponent $p>1/2$. The proof is in Appendix~\ref{appx:proofs_linear_ee}.

\begin{restatable}[end-to-end learning: general convergence rate]{theorem}{thmlineareemain}\label{thm:linear_ee_main}
    Let the input training data distribution $\nu$, the test data distribution $\nu'$, and the Gaussian prior $\normal(0,\Lambda)$ satisfy Assumptions~\ref{ass:data_end_to_end} and \ref{ass:data_input_kl_expand}. Let the ground truth linear functional $\fd\in\cH^s$ satisfy Assumption~\ref{ass:truth_regularity}. Let $\al$, $\al'$, and $p$ be as in \eqref{eqn:linear_ass_data_eig} and \eqref{eqn:linear_ass_prior_eig}.
    Then there exists $c\in (0,1/4)$ and $N_0\geq 1$ such that for any $N\geq N_0$, the mean $\bar{f}^{(N)}$ of the Gaussian posterior distribution~\eqref{eqn:posterior_ee} arising from the $N$ pairs of observed training data $(U,Y)$ in \eqref{eqn:model_data_qoi_cat_dagger} satisfies the error bound
    \begin{align}\label{eqn:linear_ee_main}
    \E^{Y\condbar U}\E^{u'\sim\nu'}\abs[\big]{\ip{\fd}{u'}-\ip{\bar{f}^{(N)}}{u'}}^2 \lesssim
        \bigl(1 + \norm{\fd}_{\cH^s}^2 \bigr)\ep_N^2
    \end{align}
    with probability at least $1-2\exp(-c N^{\min(1,\frac{\al+s-1}{\al+p})})$ over $U\sim\nu^{\otimes N}$, where
    \begin{align}\label{eqn:linear_ee_main_rate}
        \ep_N^2\defeq
        \begin{cases}
        N^{-\min\bigl(\frac{\al'+s}{\al+p},\, 1-\frac{\al+1/2-\al'}{\al+p}\bigr)}\, , & \text{if }\, \al'<\al+1/2\,,\\
        \max\bigl(N^{-\bigl(\frac{\al'+s}{\al+p}\bigr)},\, N^{-1}\log 2N\bigr)\, , &\text{if }\,\al'=\al+1/2\,,\\
        N^{-\min\bigl(\frac{\al'+s}{\al+p},\, 1\bigr)}\, , &\text{if }\,\al'>\al+1/2\,.
        \end{cases}
    \end{align}
    The constants $c$, $N_0$, and the implied constant in \eqref{eqn:linear_ee_main} do not depend on $N$ or $\fd$.
\end{restatable}

The second extra theorem develops convergence rates for the full-field learning plug-in estimator under a Sobolev regularity condition on the underlying QoI.

\begin{restatable}[full-field learning: convergence rate for Sobolev QoI]{theorem}{thmlinearffmain}\label{thm:linear_ff_main}
    Let the input training data distribution $\nu$, the test data distribution $\nu'$, the true forward map $\Ld$, and the QoI $\qd$ satisfy Assumption~\ref{ass:ff}, but instead of \ref{item:ass_ff_qoi}, suppose that $\{\qd(\varphi_j)\}_{j\in\N} \in\cH^r$ for some $r\in\R$. Let $\al$ and $\al'$ be as in \eqref{eqn:linear_ass_data_eig} and $\beta$ be as in \ref{item:ass_ff_L}. If $\min(\alpha, \alpha' + r) + \beta > 0$, then there exist constants $c>0$ and $C>0$ such that for all sufficiently large $N$, the plug-in estimator $\qd\circ \bar{L}^{(N)}$ in \eqref{eqn:linear_ff_estimator} based on the Gaussian posterior distribution~\eqref{eqn:posterior_ff} arising from the $N$ pairs of observed full-field training data $(U,\Upsilon)$ in \eqref{eqn:model_data_ff} satisfies the error bound
    \begin{equation}\label{eqn:thmerror}
        \E^{\Upsilon\condbar U}\E^{u'\sim{\nu'}}\abs[\big]{\qd(\Ld u') - \qd(\bar{L}^{(N)}u')}^2 \lesssim \ep_N^2
    \end{equation}
    with probability at least $1-Ce^{-cN}$ over $U\sim\nu^{\otimes N}$, where
    \begin{equation}\label{eqn:thmcases}
        \ep_N^2\defeq 
        \begin{cases}
            N^{-\bigl(\frac{2\al'+2\beta + 2r}{1+2\al + 2\beta}\bigr)}\, , & \textit{if }\, \al'+r<\al+1/2\,,\\
            N^{-1}\log N\, , & \textit{if }\, \al'+r=\al+1/2\,,\\
            N^{-1}\, , & \textit{if }\, \al'+r>\al+1/2\,.
        \end{cases}
    \end{equation}
    The constants $c$, $C$, and the implied constant in \eqref{eqn:thmerror} do not depend on $N$.
\end{restatable}

The proof is found in Appendix~\ref{appx:proofs_linear_ff}. The convergence rate~\eqref{eqn:thmcases} should be compared to the rate \eqref{eqn:thmcases_power} from Theorem~\ref{thm:linear_ff_main_power}. Theorem~\ref{thm:linear_ff_main} covers a larger class of QoIs than does Theorem~\ref{thm:linear_ff_main_power} due to the generality of the Sobolev regularity condition. Nonetheless, Theorem~\ref{thm:linear_ff_main} is not sharp when $\qd$ decays asymptotically like a power law, as discussed in Subsection~\ref{sec:linear_main_diss}. It is straightforward to derive comparison results like Corollary~\ref{cor:linear_compare} for the Sobolev QoI setting here, both for in-distribution and out-of-distribution test errors. We refrain from doing so for the sake of brevity.

\section{Proofs for Section~\ref{sec:ua}:~\nameref*{sec:ua}}\label{appx:proofs_ua}
This appendix begins with some universal approximation results for neural operators before establishing similar universal approximation results for neural mappings (i.e., neural functionals and decoders).

\subsection{Supporting approximation results for neural operators}\label{appx:proofs_ua_fno}
We need the following two lemmas that are simple generalizations of the universal approximation theorem for FNOs \cite[Theorem 9, p. 9]{kovachki2021universal} to the setting where only one of the input or output domain is the torus. These results may be extracted from the proof of \cite[Theorem 9, p. 9]{kovachki2021universal}.

\begin{lemma}[universal approximation for FNO: periodic output domain]\label{lem:ua_fno_periodic_output} 
    Let Assumption~\ref{ass:activation} hold. Let $s\geq 0$ and $s'\geq 0$, $\cD\subset \R^d$ be an open Lipschitz domain such that $\overline{\cD}\subset (0,1)^d$, and $\cU=H^s(\cD;\R^{d_u})$. Let $\cY=H^{s'}(\T^d;\R^{d_y})$ and $\cG\colon \cU\to \cY$ be a continuous operator. There exists a continuous linear extension operator $E\colon \cU\to H^s(\T^d;\R^{d_u})$ such that $(Eu)|_{\cD}=u$ for all $u\in\cU$. Moreover, let $K\subset \cU$ be compact in $\cU$. For any $\ep>0$, there exists a Fourier Neural Operator $\Psi\colon H^s(\T^d;\R^{d_u}) \to\cY$ of the form \eqref{eqn:fno} (with $\cE=\Id$, $\cR=\Id$, and items~\ref{item:channel_const_dim}~and~\ref{item:lift_project_affine} both holding true) such that
    \begin{align}
    \sup_{u\in K}\norm{\cG(u)-\Psi(Eu)}_{\cY}<\ep\,.
    \end{align}
\end{lemma}

The next lemma is analogous to the previous one and deals with periodic input domains.
\begin{lemma}[universal approximation for FNO: periodic input domain]\label{lem:ua_fno_periodic_input}
    Let Assumption~\ref{ass:activation} hold. Let $s\geq 0$ and $s'\geq 0$, $\cD\subset \R^d$ be an open Lipschitz domain such that $\overline{\cD}\subset (0,1)^d$, and $\cU=H^s(\T^d;\R^{d_u})$. Let $\cY=H^{s'}(\cD;\R^{d_y})$ and $\cG\colon \cU\to \cY$ be a continuous operator. Denote by $R\in\cL(H^{s'}(\T^d;\R^{d_y});\cY)$ the restriction operator $y\mapsto y|_{\cD}$. Let $K\subset \cU$ be compact in $\cU$. For any $\ep>0$, there exists a Fourier Neural Operator $\Psi\colon \cU \to H^{s'}(\T^d;\R^{d_y})$ of the form \eqref{eqn:fno} (with $\cE=\Id$, $\cR=\Id$, and items~\ref{item:channel_const_dim}~and~\ref{item:lift_project_affine} both holding true) such that
    \begin{align}
    \sup_{u\in K}\norm{\cG(u)-R\Psi(u)}_{\cY}<\ep\,.
    \end{align}
\end{lemma}

\subsection{Universal approximation proofs}\label{appx:proofs_ua_fnm}
The remainder of this appendix provides proofs of the main universal approximation theorems found in Section~\ref{sec:ua} for the proposed FNM family of architectures. We begin with the function-to-vector Fourier Neural Functionals (FNF) architecture.

\thmuafnf*
\begin{proof}
Let $\cY\defeq L^2(\T^d;\R^{d_y})$ and $\onebm\colon x\mapsto 1$ be the constant function on $\T^d$. We first convert the function-to-vector mapping $\Psi^\dagger$ to the function-to-function operator $ \cG^\dagger\colon \cU\to\cY$ defined by $u\mapsto \Psi^\dagger(u)\onebm $. We then establish the existence of a FNO that approximates $\cG^\dagger$. Finally, from this FNO we construct a FNF that approximates $\Psi^\dagger$.  To this end, fix $\ep'>0$. By the continuity of $\Psi^\dagger$, there exists ${\delta}>0$ such that $\norm{u_1-u_2}_\cU<{\delta}$ implies $\norm{\Psi^\dagger(u_1)-\Psi^\dagger(u_2)}_{\R^{d_y}}<\ep'$. Then
\begin{align*}
    \norm{\cG^\dagger(u_1)-\cG^\dagger(u_2)}_{\cY}^2&=\int_{\Td}\norm{\Psi^\dagger(u_1)\onebm(x) - \Psi^\dagger(u_2)\onebm(x)}^2_{\R^{d_y}}\dd{x}\\
    &=\abs{\T^d} \norm{\Psi^\dagger(u_1) - \Psi^\dagger(u_2)}^2_{\R^{d_y}}\\
    &<(\ep')^2\,.
\end{align*}
We used the fact that $\abs{\T^d}=1$ for the identification $\T^d\equiv (0,1)^d_{\mathrm{per}}$. This shows the continuity of $\cG^\dagger\colon \cU\to\cY$.
By the universal approximation theorem for FNOs (Lemma~\ref{lem:ua_fno_periodic_output}, applied with $s=s$, $s'=0$, $d_y=d_y$, and $\cG=\cG^\dagger$), there exists a continuous linear operator $E\colon \cU\to H^s(\T^d;\R^{d_u})$ and a FNO $\cG\colon H^s(\T^d;\R^{d_u})\to\cY$ of the form \eqref{eqn:fno} (with $\cR=\Id$, $\cE=\Id$, and items~\ref{item:channel_const_dim}~and~\ref{item:lift_project_affine} both holding true) such that 
\begin{align*}
    \sup_{u\in K}\norm{{\cG}^\dagger(u)-\cG(Eu)}_{\cY}<\ep\,.
\end{align*}
To complete the proof, we construct a FNF by appending a specific linear layer to the output of $\cG\circ E$. To this end, let $\overline{P}\colon \cY\to\R^{d_y}$ be the averaging operator
\begin{align*}
    u\mapsto \overline{P}u\defeq \int_{\T^d}u(x)\dd{x}\,.
\end{align*}
Clearly $\overline{P}$ is linear. It is continuous on $\cY$ because
\begin{align*}
    \norm{\overline{P}u}_{\R^{d_y}}\leq \int_{\T^d}\norm{u(x)}_{\R^{d_y}}\onebm(x)\dd{x}\leq \norm{u}_{\cY}
\end{align*}
by the triangle and Cauchy--Schwarz inequalities. Now define $\Psi\defeq (\overline{P}\circ\cG\circ E) \colon \cU\to\R^{d_y}$. This map has the representation 
\begin{align*}
    \Psi= \overline{P}\circ \Tilde{\cQ}\circ F\circ \Tilde{\cS} \circ E
\end{align*}
for some local linear operators $\tilde{\cQ}$ (identified with $\tilde{Q}\in\R^{d_y\times d_v}$ for channel dimension $d_v$) and $\tilde{\cS}$ (identified with $\tilde{S}\in\R^{d_v\times d_u}$), and where $F$ denotes the repeated composition of all nonlinear FNO layers of the form $\sL_t$ as in \eqref{eqn:fno_layer_nonlinear}. We claim that $\Psi$ belongs to the FNF class, i.e., \eqref{eqn:FNM} with modification \ref{enum:F2V}. To see this, choose $Q=I_{\R^{d_y}}\in\R^{d_y\times d_y}$ and $S=I_{\R^{d_u}}\in\R^{d_u\times d_u} $ (which we identify with $\Id_{\cU}\in\cL(\cU)$). Let $\cE\defeq (\tilde{\cS}\circ E) \colon \cU\to H^s(\T^d;\R^{d_v}) $. Define the linear functional layer $\sG\defeq (\overline{P}\circ\tilde{\cQ})\colon L^2(\T^d;\R^{d_v})\to\R^{d_y}$ which has the kernel linear functional representation
\[
u\mapsto \sG u = \int_{\T^d}\kappa(x)u(x)\dd{x}\,, \qw x\mapsto \kappa(x)\defeq \onebm(x) \tilde{Q}\in\R^{d_y \times d_v}
\]
as in \eqref{eqn:sG-sD}. Thus,
\begin{align*}
    \Psi&= \overline{P}\circ \Tilde{\cQ}\circ F \circ \Tilde{\cS} \circ E\\
    &= I_{\R^{d_y}}\circ (\overline{P}\circ \Tilde{\cQ})\circ F \circ (\Tilde{\cS} \circ E)\circ \Id_{\cU}
    \\& = Q\circ \sG \circ F \circ \cE\circ S
\end{align*}
as claimed. Finally, using the fact that $\overline{P}(z\onebm) = z$ for any $z\in\R^{d_y}$, it holds that
\begin{align*}
    \sup_{u\in K}\norm{\Psi^\dagger(u)-\Psi(u)}_{\R^{d_y}}=\sup_{u\in K}\norm{\overline{P}\cG^\dagger(u)-\overline{P}\cG(Eu)}_{\R^{d_y}}
    \leq \sup_{u\in K}\norm{\cG^\dagger(u)-\cG(Eu)}_{\cY}\,.
\end{align*}
The rightmost expression is less than $\ep$ and hence \eqref{eqn:ua_fnf_ep} holds.
\end{proof}

The universality proof for the vector-to-function Fourier Neural Decoder (FND) architecture follows similar arguments.
\thmuafnd*
\begin{proof}
    Let $\cU\defeq L^2(\T^d;\R^{d_u})$ and $\onebm\colon\Td\to\R$ be the constant function $x \mapsto 1$. Define the map $\mathsf{L}\colon \R^{d_u}\to\cU$ by  $z\mapsto z\onebm$. Clearly $\mathsf{L}$ is linear. To see that it is continuous, we compute
    \begin{align}\label{eqn:constant_func_op_norm}
    \norm{\mathsf{L}z}_{\cU}^2=\int_{\T^d}\norm{z\onebm(x)}^2_{\R^{d_u}}\dd{x}
        =\abs{\T^d}\norm{z}_{\R^{d_u}}^2 = \norm{z}_{\R^{d_u}}^2\,.
    \end{align}
    Thus, $\sfL$ is injective with 
 $\norm{\mathsf{L}}_{\cL(\R^{d_u};\cU)}= 1$. Choose $K\defeq \mathsf{L}\mathcal{Z}=\{\mathsf{L}z\colon z\in\mathcal{Z}\} \subset \cU$, which is compact in $\cU$ because continuous functions map compact sets to compact sets. Define $\cG^\dagger\colon K\to\cY$ by $\mathsf{L}z\mapsto \Psi^\dagger(z)$. First, we show that $\cG^\dagger$ is continuous. Fix $\ep'>0$. By the continuity of $\Psi^\dagger$, there exists ${\delta}>0$ such that if $\norm{\mathsf{L}z_1-\mathsf{L}z_2}_{\cU}=\norm{z_1-z_2}_{\R^{d_u}}<{\delta}$, then $\norm{\Psi^\dagger(z_1)-\Psi^\dagger(z_2)}_{\cY}<\ep'$.
    Thus for any $u_1=\mathsf{L}z_1\in K$ and $u_2=\mathsf{L}z_2\in K$ with $\norm{u_1-u_2}_{\cU}<\delta$, we have
    \begin{align*}
        \norm{\cG^\dagger(u_1)-\cG^\dagger(u_2)}_{\cY}=\norm{\Psi^\dagger(z_1)-\Psi^\dagger(z_2)}_{\cY}<\ep'\,.
    \end{align*}
    It follows that $\cG^\dagger\colon K\to\cY$ is continuous. By the Dugundji extension theorem~\cite{dugundji1951extension}, there exists a continuous operator $\tilde{\cG}^\dagger\colon \cU\to\cY$ such that $\tilde{\cG}^\dagger(u)={\cG}^\dagger(u)$ for every $u\in K$.
    By the universal approximation theorem for FNOs (Lemma~\ref{lem:ua_fno_periodic_input}, applied with $s=0$, $s'=t$, $d_u=d_u$, and $\cG=\tilde{\cG}^\dagger$), there exists a FNO $\cG\colon \cU\to H^t(\T^d;\R^{d_y})$ of the form \eqref{eqn:fno} (with $\cR=\Id$, $\cE=\Id$, and items~\ref{item:channel_const_dim}~and~\ref{item:lift_project_affine} both holding true) such that 
    \begin{align*}
        \sup_{u\in K}\norm{\tilde{\cG}^\dagger(u)-R\cG(u)}_{\cY}=\sup_{u\in K}\norm{{\cG}^\dagger(u)-R\cG(u)}_{\cY}<\ep\,.
    \end{align*}
    In the preceding display, $R\in\cL(H^{t}(\T^d;\R^{d_y});\cY)$ denotes the restriction operator $y\mapsto y|_{D}$.
    Now define the map $\Psi \defeq (R\circ \cG\circ \sfL) \colon\R^{d_u}\to\cY$. This map has the representation
    \begin{align*}
        \Psi= R\circ \Tilde{\cQ}\circ F \circ \Tilde{\cS} \circ \sfL
    \end{align*}
    for some local linear operators $\tilde{\cQ}$ (identified with $\tilde{Q}\in\R^{d_y\times d_v}$ for channel dimension $d_v$) and $\tilde{\cS}$ (identified with $\tilde{S}\in\R^{d_v\times d_u}$), and where $F$ denotes the repeated composition of all nonlinear FNO layers of the form $\sL_t$ as in \eqref{eqn:fno_layer_nonlinear}. We claim that $\Psi$ is of the FND form, i.e., \eqref{eqn:FNM} with modification \ref{enum:V2F}. To see this, choose $Q=I_{\R^{d_y}}\in\R^{d_y\times d_y}$ (which we identify with $\Id_{\cY}\in\cL(\cY)$) and $S=I_{\R^{d_u}}\in\R^{d_u\times d_u} $. Let $\cR\defeq (R\circ \tilde{\cQ})\colon H^t(\T^d;\R^{d_v}) \to \cY $. Define the linear decoder layer $\sD\defeq (\tilde{\cS}\circ\sfL) \colon \R^{d_u}\to L^2(\T^d;\R^{d_v})$ which has the kernel function product representation
    \[
    z\mapsto \sD z = \kappa(\cdot) z\,, \qw x\mapsto \kappa(x)\defeq \onebm(x) \tilde{S}\in\R^{d_v\times d_u}
    \]
    as in \eqref{eqn:sG-sD}. Thus,
    \begin{align*}
        \psi&= R\circ \Tilde{\cQ}\circ F\circ \Tilde{\cS} \circ \sfL\\
        &= \Id_{\cY}\circ (R\circ \Tilde{\cQ})\circ F \circ (\Tilde{\cS} \circ \sfL)\circ I_{\R^{d_u}}
        \\& = Q\circ \cR \circ F\circ \sD \circ S
    \end{align*}   
    as claimed. Finally, by the injectivity of $\mathsf{L}$ implied by \eqref{eqn:constant_func_op_norm}, any $u'\in K$ has the representation $u'=\mathsf{L}z'$ for some unique $z'\in\mathcal{Z}\subset\R^{d_u}$. It follows that
    \begin{align*}
        \sup_{u\in K}\norm{{\cG}^\dagger(u)-R\cG(u)}_{\cY}\geq \norm{{\cG}^\dagger(u')-R\cG(u')}_{\cY}=\norm{{\Psi}^\dagger(z')-\Psi(z')}_{\cY}\,.
    \end{align*}
    This implies the asserted result \eqref{eqn:ua_fnd_ep}.
\end{proof}

\section{Proofs for Section~\ref{sec:linear}:~\nameref*{sec:linear}}\label{appx:proofs_linear}
This appendix contains the lengthy arguments that underlie the statistical learning theory for regression of linear functionals from Section~\ref{sec:linear}.
We begin by recalling convenient properties of subgaussian and subexponential probability distributions in Appendix~\ref{appx:proofs_linear_sub}. Appendix~\ref{appx:proofs_linear_ee} contains proofs of the main \ref{item:ee} results from Section~\ref{sec:linear_main}. In particular, it develops a new bias--variance analysis of the linear functional regression problem in a Bayesian nonparametric setting that may be of independent interest. Proofs for the full-field learning approach (Subsection~\ref{sec:linear_ff}) to factorized linear functionals are provided in Appendix~\ref{appx:proofs_linear_ff}. The sample complexity comparison corollary is proved in Appendix~\ref{appx:proofs_linear_compare}. Technical lemmas used throughout the analysis are collected in Appendix~\ref{appx:proofs_linear_extra}.

\subsection{Subgaussian and subexponential distributions}\label{appx:proofs_linear_sub}
This appendix reviews the concept of subgaussian and subexponential random variables. These play a central role in the analysis leading to the high probability error bounds in Section~\ref{sec:linear_main}.

\begin{definition}[subgaussian]\label{def:app_subg}
    A real-valued random variable $X$ is said to be \emph{subgaussian} \cite[Section 2.5]{vershynin2018high} if for some $\sigma>0$ it satisfies the moment generating function bound 
    \begin{align}\label{eqn:mgf_subg}
        \E e^{\lambda(X-\E X)}\leq e^{\lambda^2\sigma^2/2} \qfa \lambda\in\R\,.
    \end{align}
\end{definition}
We write $X\in\SG{\sigma^2}$ when \eqref{eqn:mgf_subg} holds and define the subgaussian norm of $X$ by
\begin{align}\label{eqn:norm_subg}
    \norm{X}_{\psi_2}\defeq \sup_{p\in[1,\infty)}\frac{\bigl(\E\abs{X}^p\bigr)^{1/p}}{\sqrt{p}}\,.
\end{align}
It is known that $X$ is subgaussian if and only if $\norm{X}_{\psi_2}<\infty$. However, we often require random variables with heavier tails.

\begin{definition}[subexponential]\label{def:app_subexp}
    A real-valued random variable $Z$ is said to be \emph{subexponential} \cite[Section 2.7]{vershynin2018high} if for some $v>0$ and $b>0$ it satisfies the moment generating function bound
    \begin{align}\label{eqn:mgf_subexp}
        \E e^{\lambda(Z-\E Z)}\leq e^{\lambda^2v^2/2} \qfa \abs{\lambda}\leq \frac{1}{b}\,.
    \end{align}
\end{definition}
In contrast to the subgaussian case, the moment generating function of a subexponential random variable need only exist in a neighborhood of the origin instead of everywhere on the real line.
We write $Z\in\SE{v^2}{\al}$ when \eqref{eqn:mgf_subexp} holds and define subexponential norm by
\begin{align}\label{eqn:norm_subexp}
    \norm{Z}_{\psi_1}\defeq \sup_{p\in[1,\infty)}\frac{\bigl(\E\abs{Z}^p\bigr)^{1/p}}{p}\,.
\end{align}
It is known that $X$ is subgaussian if and only if $Z=X^2$ is subexponential. In fact, we have the following estimate relating the two norms.
\begin{lemma}[squared subgaussian]\label{lem:squared_subg_bound}
    Let $X$ be a real-valued random variable. Then
    \begin{align}
        \norm[\big]{X}_{\psi_2}^2\leq \norm[\big]{X^2}_{\psi_1}\leq 2\norm[\big]{X}_{\psi_2}^2\,.
    \end{align}
\end{lemma}
\begin{proof}
    We compute
    \begin{align*}
    \norm[\big]{X^2}_{\psi_1}=\sup_{p\in[1,\infty)}\frac{\bigl(\E\abs{X}^{2p}\bigr)^{1/p}}{p}
    &=2\sup_{p\in[1,\infty)}\frac{\bigl(\E\abs{X}^{2p}\bigr)^{1/{2p}}\bigl(\E\abs{X}^{2p}\bigr)^{1/{2p}}}{\sqrt{2p}\sqrt{2p}}\\
    & =2\left(\sup_{p\in[1,\infty)}\frac{\bigl(\E\abs{X}^{2p}\bigr)^{1/{2p}}}{\sqrt{2p}}\right)^2\\
    &\leq 2\Biggl(\sup_{2p\in {[1,\infty)}}\frac{\bigl(\E\abs{X}^{2p}\bigr)^{1/{2p}}}{\sqrt{2p}}\Biggr)^2.
    \end{align*}
    The final term inside parentheses on the right-hand side equals the subgaussian norm (upon replacing $2p$ with $p$). This is the asserted upper bound. The lower bound follows from the inequality $(\E\abs{X}^{2p})^{1/2}\geq \E\abs{X}^p$.
\end{proof}

\subsection{Proofs for end-to-end learning of general linear functionals}\label{appx:proofs_linear_ee}
The goal of this appendix is to prove Theorem~\ref{thm:linear_ee_main} (which implies Theorem~\ref{thm:linear_ee_opt}). This theorem provides a high probability convergence rate in terms of the sample size $N$ for the \emph{out-of-distribution test error}
\begin{equation}\label{eqn:app_test_error_weighted}
    \E^{u'\sim\nu'}\abs[\big]{\ip{\fd}{u'}-\ip{\bar{f}^{(N)}}{u'}}^2 = \norm[\big]{(\Sigma')^{1/2}(\fd-\bar{f}^{(N)})}^2
\end{equation}
conditioned on the covariates $U$. The equality in \eqref{eqn:app_test_error_weighted} is due to linearity and the fact that $\abs{\ip{\fd}{u'}-\ip{\bar{f}^{(N)}}{u'}}^2 = \ip{\fd - \bar{f}^{(N)}}{(u'\otimes u')(\fd - \bar{f}^{(N)})}$. For notational convenience in the proofs, we write
\begin{equation}\label{eqn:test_error_conditional}
    \sR_N \defeq \E^{Y\condbar U} \norm[\big]{(\Sigma')^{1/2}(\fd-\bar{f}^{(N)})}^2 = \E\Bigl[\norm[\big]{(\Sigma')^{1/2}(\fd-\bar{f}^{(N)})}^2 \lcondbar u_1,\ldots,u_N \Bigr]\,.
\end{equation}
The argument behind the proof of Theorem~\ref{thm:linear_ee_main} follows a classical bias--variance decomposition. We now state this decomposition in the following lemma.

\begin{lemma}[bias--variance decomposition]\label{lem:app_bias_var_decomp}
Instate the setting of Subsection \ref{sec:linear_ee_setup}. The test error~\eqref{eqn:test_error_conditional} satisfies the decomposition $\sR_N= B_N + V_N$,
where 
\begin{subequations}
    \begin{align}
    B_N &\defeq \norm[\big]{(\Sigma')^{1/2}(\Id_H - A_NS_N)\fd}^2 \qa \label{eqn:bias}\\
    V_N &\defeq \gamma^2 \E\left[\norm[\big]{(\Sigma')^{1/2}A_N\Xi}^2\lcondbar U\right]\, .\label{eqn:variance}
    \end{align}
\end{subequations}
\end{lemma}
\begin{proof}
    Denote $\norm{\slot}_{\nu'}\defeq \norm{(\Sigma')^{1/2}\slot}$. Let $m_N\defeq \E^{Y\condbar U}[\bar{f}^{(N)}]$. Expanding $\E^{Y\condbar U}\norm{\fd-\bar{f}^{(N)}}^2_{\nu'} = \E^{Y\condbar U}\norm{(\fd- m_N) + (m_N - \bar{f}^{(N)})}^2_{\nu'}$ shows that $\sR_N$ equals
    \begin{align*}
        \norm{\fd-m_N}^2_{\nu'} + 2\E^{Y\condbar U}\ip[\big]{\fd - m_N}{\Sigma'(m_N-\bar{f}^{(N)})}
        + \E^{Y\condbar U}\norm{m_N-\bar{f}^{(N)}}_{\nu'}^2\,.
    \end{align*}
    By linearity, the middle term equals zero. Recalling that $\bar{f}^{(N)}=A_NY$ from \eqref{eqn:posterior_ee} and $Y=S_N\fd +\gamma\Xi$ from \eqref{eqn:model_data_qoi_cat_dagger}, the fact $\Xi$ is centered implies that $m_N=A_NS_N\fd$. Thus, we recover $B_N$~\eqref{eqn:bias} as the first term in the preceding display. Similarly, $m_N-\bar{f}^{(N)} = -A_N(\gamma\Xi)$ so that $V_N$~\eqref{eqn:variance} equals the last term.
\end{proof}

The main proof novelty lies in the bound for the bias $B_N$, which relies on a careful conditioning argument and clever matrix identities. The analysis begins by estimating the variance $V_N$ in Appendix~\ref{appx:proofs_linear_ee_var}. The bias is studied in Appendix~\ref{appx:proofs_linear_ee_bias}. Finally, the bounds are combined to prove Theorem~\ref{thm:linear_ee_main} in Appendix~\ref{appx:proofs_linear_ee_main}.

\subsubsection{Bounding the variance}\label{appx:proofs_linear_ee_var}
We focus on controlling the variance term~\eqref{eqn:variance} first because it is easier to estimate than the bias. Our goal is to prove the following.

\begin{proposition}[variance upper bound]\label{prop:var}
Under Assumptions~\ref{ass:data_end_to_end} and \ref{ass:data_input_kl_expand}, there exists $c\in (0,1)$ and $N_0\geq 1$ (depending only on $\nu$, $ \nu' $, $\Lambda$, and $ \gamma^2 $) such that for all $N\geq N_0$, it holds that the variance term $V_N$ in \eqref{eqn:variance} satisfies the estimate
\begin{equation}\label{eqn:variance_upper_bound}
    V_N\leq 
    2\sum_{j=1}^\infty\frac{\sigma_j'\lambda_j}{1 + N\gamma^{-2}\sigma_j\lambda_j}
    \lesssim
    \begin{cases}
    N^{-\left(1-\frac{\al+1/2-\al'}{\al+p}\right)}\, , & \text{if }\, \al'<\al+1/2\,,\\
    N^{-1}\log 2N\, , &\text{if }\,\al'=\al+1/2\,,\\
    N^{-1}\, , &\text{if }\,\al'>\al+1/2
    \end{cases}
\end{equation}
with probability at least $1-e^{-cN}$.
\end{proposition}
The expression for the variance upper bound in Proposition~\ref{prop:var} is in precisely the same form as that found in \cite[Equation (A.1b), p. 24]{de2023convergence}. To prove the proposition, we require some preparatory results.
First, notice that
\[
\E\left[\norm[\big]{(\Sigma')^{1/2}A_N\Xi}^2\lcondbar U\right] = \E^{g\sim\normal(0,\Id_{{\R^N}})}\norm[\big]{(\Sigma')^{1/2}A_Ng}^2=\tr[\big]{(\Sigma')^{1/2}A_N^{\phantom{*}}A_N^*(\Sigma')^{1/2}}\,.
\]
Now using the fact that $\widehat{\Sigma}=S_N^*S_N^{\phantom{*}}/N$ from \eqref{eqn:sigmahat} and defining
\begin{align}\label{eqn:app_defn_C_hat}
    \widehat{\cC}\defeq \Lambda^{1/2}\widehat{\Sigma}\Lambda^{1/2}\,,
\end{align}
we see that
\[
A_N^{\phantom{*}}A_N^* = \frac{1}{N}\biggl[\Lambda^{1/2}\biggl(\widehat{\cC} + \frac{\gamma^2}{N}\Id_H\biggr)^{-1}\widehat{\cC}\biggl(\widehat{\cC} + \frac{\gamma^2}{N}\Id_H\biggr)^{-1}\Lambda^{1/2}\biggr]\,.
\]
Next, define
\begin{align}\label{eqn:app_defn_mu_C_hat_reg}
    \mu\defeq \gamma^2/N>0\qa \widehat{\cC}_\mu\defeq \widehat{\cC} + \mu\Id_{H}\,.
\end{align}
By the cyclic property of the trace,
\begin{align*}
    V_N&=\mu \tr[\big]{\Lambda^{1/2}\Sigma'\Lambda^{1/2}\widehat{\cC}_\mu^{-1}\widehat{\cC}\widehat{\cC}_\mu^{-1}}\\
    &= \mu \tr[\big]{[\widehat{\cC}_\mu^{-1/2}\Lambda^{1/2}\Sigma'\Lambda^{1/2}\widehat{\cC}_\mu^{-1/2}][\widehat{\cC}_\mu^{-1/2}\widehat{\cC}\widehat{\cC}_\mu^{-1/2}]}\\
    &\leq \mu \tr[\big]{\widehat{\cC}_\mu^{-1/2}\Lambda^{1/2}\Sigma'\Lambda^{1/2}\widehat{\cC}_\mu^{-1/2}}\norm[\big]{\widehat{\cC}_\mu^{-1/2}\widehat{\cC}\widehat{\cC}_\mu^{-1/2}}_{\cL(H)}\\
    &\leq \mu \tr[\big]{\widehat{\cC}_\mu^{-1/2}\Lambda^{1/2}\Sigma'\Lambda^{1/2}\widehat{\cC}_\mu^{-1/2}}\,.
\end{align*}
The first inequality is due to $\tr{AB}\leq\tr{A}\norm{B}_{\cL(H)}$, which holds for any symmetric positive-semidefinite trace-class $A$ and any bounded $B$; this follows from the von Neumann trace inequality. The second inequality follows from the simultaneous diagonalizability of the factors in the triple product inside the operator norm and the fact that $\lambda/(\lambda+\mu)\leq 1$ for any eigenvalue $\lambda$ of $\widehat{\cC}$.
Now define
\begin{align}\label{eqn:app_defn_Ct_C_Creg}
    \cC'\defeq \Lambda^{1/2}\Sigma'\Lambda^{1/2}\,,\quad \cC\defeq\Lambda^{1/2}\Sigma\Lambda^{1/2}\,,\qa \cC_\mu\defeq \cC+\mu\Id_H\,.
\end{align}
Lemma~\ref{lem:inverse_reg_neumann_identity} (with $A=\widehat{\cC}$, $B=\cC$, and $\lambda=\mu$) shows that $V_N$ is bounded above by
\begin{align}\label{eqn:var_bound_temp_op}
    \begin{split}
        \mu\tr[\big]{\cC'\widehat{\cC}_\mu^{-1}}&=\mu\tr[\Big]{\cC'\cC_\mu^{-1/2}\bigl(\Id_H - \cC_\mu^{-1/2}(\cC-\widehat{\cC})\cC_\mu^{-1/2}\bigr)^{-1}\cC_\mu^{-1/2}}\\
        &=\mu\tr[\Big]{\bigl[\cC_\mu^{-1/2}\cC'\cC_\mu^{-1/2}\bigr]\bigl(\Id_H - \cC_\mu^{-1/2}(\cC-\widehat{\cC})\cC_\mu^{-1/2}\bigr)^{-1}}\\
        &\leq \mu\tr[\big]{\cC_\mu^{-1/2}\cC'\cC_\mu^{-1/2}}\norm[\big]{\bigl(\Id_H - \cC_\mu^{-1/2}(\cC-\widehat{\cC})\cC_\mu^{-1/2}\bigr)^{-1}}_{\cL(H)}.
    \end{split}
\end{align}
The final inequality is again due to von Neumann's trace inequality. To bound the operator norm in the preceding display, we apply a Neumann series argument.

\begin{lemma}[Neumann series bound]\label{lem:good_set_inverse_bound}
	Let $\mu=\gamma^2/N$. There exists $ c\in(0,1) $ and $ N_0\geq 1 $ such that for any $N\geq N_0$, it holds that
	\begin{align}\label{eqn:good_set_inverse_bound}
		\norm[\Big]{\bigl(\Id_H - \cC_\mu^{-1/2}(\cC-\widehat{\cC})\cC_\mu^{-1/2}\bigr)^{-1}}_{\cL(H)}\leq 2
	\end{align}
	with probability at least $ 1-e^{-cN} $.
\end{lemma}
\begin{proof}
    By Lemma~\ref{lem:good_set}, the event \eqref{eqn:normalized_operator_norm} holds with probability at least $1-e^{-cN}$ for any $N\geq N_0$. On this event, we can invoke the Neumann series expansion
	\begin{equation*}
		\Bigl(\Id_H -  \cC_{\mu}^{-1/2} (\cC-\widehat{\cC}) \cC_{\mu}^{-1/2}\Bigr)^{-1} = \sum_{k=0}^\infty \Bigl(\cC_\mu^{-1/2}(\cC-\widehat{\cC})\cC_\mu^{-1/2}\Bigr)^k\,.
	\end{equation*}
	This delivers the operator norm bound
	\begin{align*}
		\norm[\Big]{\bigl(\Id_H - \cC_\mu^{-1/2}(\cC-\widehat{\cC})\cC_\mu^{-1/2}\bigr)^{-1}}_{\cL(H)} &\leq \sum_{k=0}^\infty \norm[\Big]{\bigl(\cC_\mu^{-1/2}(\widehat{\cC}-\cC)\cC_\mu^{-1/2}\bigr)^k}_{\cL(H)}\\
		&\leq \sum_{k=0}^\infty \norm[\big]{\cC_\mu^{-1/2}(\widehat{\cC}-\cC)\cC_\mu^{-1/2}}_{\cL(H)}^k\\
		&\leq \sum_{k=0}^\infty \left(\frac{1}{2}\right)^k
	\end{align*}
	by \eqref{eqn:normalized_operator_norm}. The fact that $ \sum_{k=0}^\infty (1/2)^k = (1-1/2)^{-1} = 2 $ completes the proof.
\end{proof}
We may now prove Proposition~\ref{prop:var}.
\begin{proof}[Proof of Proposition~\ref{prop:var}]
Combining \eqref{eqn:var_bound_temp_op} and Lemma~\ref{lem:good_set_inverse_bound} (with $c$ and $N_0$ as in the hypotheses there) shows that
\begin{equation}\label{eqn:variance_series}
    V_N\leq 2\mu \tr[\big]{\cC_\mu^{-1/2}\cC'\cC_\mu^{-1/2}} =2\sum_{j=1}^\infty\frac{\mu\sigma_j'\lambda_j}{\mu + \sigma_j\lambda_j}=2\sum_{j=1}^\infty\frac{\sigma_j'\lambda_j}{1 + N\gamma^{-2}\sigma_j\lambda_j}
\end{equation}
with probability at least $1-e^{-cN}$ if $N\geq N_0$. We used the assumed simultaneous diagonalizability of the prior and data covariance operators. Since $\sigma_j'\lesssim j^{-2\al'}$, $\sigma_j\asymp j^{-2\al} $, and $\lambda_j\asymp j^{-2p}$, all as $j\to\infty$, under Assumption~\ref{ass:data_end_to_end}, there exists $j_0\in\N$ (independent of $N$) such that the rightmost expression in \eqref{eqn:variance_series} is bounded above by
\begin{align*}
    \sum_{j\leq j_0}\frac{\sigma_j'\lambda_j}{1 + N\gamma^{-2}\sigma_j\lambda_j} + \sum_{j>j_0}\frac{j^{-2(\al'+p)}}{1 + Nj^{-2(\al+p)}}
    &\lesssim \frac{\gamma^2}{N}\sum_{j\leq j_0} \frac{\sigma_j'}{\sigma_j}+\sum_{j=1}^\infty \frac{j^{-2(\al'+p)}}{1 + Nj^{-2(\al+p)}}\\
    &\lesssim \sum_{j=1}^\infty \frac{j^{-2(\al'+p)}}{1+Nj^{-2(\al+p)}}\,.
\end{align*}
The last inequality in the preceding display follows from the fact that $1+N\leq 2N$ and an argument similar to the one used in the proof of Lemma~\ref{lem:app_effective_dim_trace}. The proof is complete after an application of Lemma~\ref{lem:app_series_rate2} (with $t=2(\al'+p)>1$, $u=2(\al+p)>0$, and $v=1$) yields the rightmost expression in~\eqref{eqn:variance_upper_bound}.
\end{proof}

The proof of Proposition~\ref{prop:var} also justifies the claim made in Remark~\ref{rmk:contraction}. Indeed, the second term on the right-hand side of the equality~\eqref{eqn:error_posterior_full} is the posterior spread with respect to the weighted norm~\eqref{eqn:test_error_weighted}. Equation~\eqref{eqn:var_bound_temp_op} upper bounds the posterior spread, which in turn upper bounds the variance~\eqref{eqn:variance} in the bias--variance decomposition of \eqref{eqn:test_error_weighted}. Thus, the variance and the posterior spread have the same upper bound.

\subsubsection{Bounding the bias}\label{appx:proofs_linear_ee_bias}
Recall from \eqref{eqn:bias} that the bias term is given by
\[
B_N=\norm[\big]{(\Sigma')^{1/2}(\Id_H - A_NS_N)\fd}^2\,.
\]
In this appendix, we establish the following upper bound on $B_N$.

\begin{proposition}[bias upper bound]\label{prop:bias}
    Let Assumptions~\ref{ass:data_end_to_end}, \ref{ass:data_input_kl_expand}, and \ref{ass:truth_regularity} hold. Let the bias $B_N$ be as in \eqref{eqn:bias}. There exists $c_0>8$, $c\in (0,1/4)$, and $N_0\geq 1$ (all independent of $N$ and $\fd$) such that for any $N\geq N_0$, it holds that
    \begin{align}\label{eqn:bias_upper_bound}
        B_N\leq 2\sum_{j=1}^\infty \frac{\sigma_j'\abs{\fd_j}^2}{\bigl(1+N\gamma^{-2}\sigma_j\lambda_j\bigr)^2} +
        c_0\norm[\big]{\fd}_{\cH^s}^2 \sum_{j=1}^\infty\frac{\sigma_j'\lambda_j}{1+N\gamma^{-2}\sigma_j\lambda_j}
    \end{align}
    with probability at least $1-2\exp(-c N^{\min(1,\frac{\al+s-1}{\al+p})})$. On the same event, the variance bound \eqref{eqn:variance_upper_bound} also holds true.
\end{proposition}
This bias bound is interesting because the second term in \eqref{eqn:bias_upper_bound} is the same as the upper bound on the variance $V_N$ in \eqref{eqn:variance_upper_bound} (up to constant factors depending on $\norm{\fd}_{\cH^s}$).
Thus, as long as $\fd$ is non-zero and not too small in norm, the total test error of the posterior mean estimator (i.e., the sum of bias and variance) is essentially dominated by the bias. Moreover, the hypotheses of Proposition~\ref{prop:bias} do not require the true linear functional $\fd$ to belong to the reproducing kernel Hilbert space of the prior $\normal(0,\Lambda)$ (i.e., we allow for $\sum_{j=1}^\infty \lambda_j^{-1}\abs{\fd_j}^2=\infty$). This is a significant advantage of our approach over related work; see Subsection~\ref{sec:linear_main_diss}.

The proof of Proposition~\ref{prop:bias} is very lengthy. We break up the argument into several lemmas and steps. To set the stage, we instate the notation and definitions from Appendix~\ref{appx:proofs_linear_ee_var}, in particular, the objects $\widehat{\cC}$ from \eqref{eqn:app_defn_C_hat}, $\mu$ and $\widehat{\cC}_\mu$ from \eqref{eqn:app_defn_mu_C_hat_reg}, and $\cC'$, $\cC$, and $\cC_\mu$ from \eqref{eqn:app_defn_Ct_C_Creg}. We also use the shorthand notation
\begin{equation}\label{eqn:That_Mhat}
    \widehat{T}\defeq \cC_\mu^{-1/2}(\cC-\widehat{\cC})\cC_\mu^{-1/2}\qa \widehat{M}\defeq (\Id_H-\widehat{T})^{-1}
\end{equation}
for two random operators that appear frequently in the sequel.

We begin our analysis with a useful random series representation of the bias.
\begin{lemma}[bias: series]\label{lem:bias_series}
    Under Assumption~\ref{ass:data_end_to_end}, $B_N$ satisfies the identity
    \begin{equation}\label{eqn:bias_series}
        B_N=\mu^2\sum_{k=1}^\infty \frac{\sigma_k'\lambda_k}{\sigma_k\lambda_k + \mu}\abs[\Bigg]{\sum_{j=1}^\infty{\frac{{\fd_j}\lambda_j^{-1/2}}{(\sigma_j\lambda_j + \mu)^{1/2}}{\ip{\varphi_k}{\widehat{M}\varphi_j}}}}^2.
    \end{equation}
\end{lemma}
\begin{proof}
By \eqref{eqn:sigmahat} and \eqref{eqn:AN_and_postcov}, we observe that
\[
A_NS_N = \Lambda^{1/2} \widehat{\cC}_\mu^{-1}\Lambda^{1/2}\widehat{\Sigma}
\]
and hence
\[
A_NS_N\fd = \sum_{j=1}^\infty \Lambda^{1/2} \widehat{\cC}_\mu^{-1}\Lambda^{1/2}\widehat{\Sigma} \fd_j\varphi_j=\sum_{j=1}^\infty \Lambda^{1/2} \widehat{\cC}_\mu^{-1}\widehat{\cC} \fd_j\lambda_j^{-1/2}\varphi_j\,.
\]
We used the diagonalization of $\Lambda=\sum_j\lambda_j\varphi_j\otimes\varphi_j$ in the last equality. Noticing that
\[
\Id_H-\widehat{\cC}_\mu^{-1}\widehat{\cC}=\widehat{\cC}_\mu^{-1}\widehat{\cC}_\mu - \widehat{\cC}_\mu^{-1}\widehat{\cC}=\mu\widehat{\cC}_\mu^{-1}\,,
\]
we have the chain of equalities
\begin{align*}
    \fd - A_NS_N\fd &= \sum_{j=1}^\infty \left(\fd_j\lambda_j^{-1/2}\Lambda^{1/2}\varphi_j - \fd_j\lambda_j^{-1/2}\Lambda^{1/2} \widehat{\cC}_\mu^{-1}\widehat{\cC} \varphi_j\right)\\
    &=\sum_{j=1}^\infty \fd_j\lambda_j^{-1/2}\Lambda^{1/2}(\Id_H-\widehat{\cC}_\mu^{-1}\widehat{\cC})\varphi_j\\
    &=\mu \sum_{j=1}^\infty \fd_j\lambda_j^{-1/2}\Lambda^{1/2}\widehat{\cC}_\mu^{-1}\varphi_j\,.
\end{align*}
Recalling $\widehat{M}$ from \eqref{eqn:That_Mhat} and using the identity \eqref{eqn:inverse_reg_neumann_identity} from Lemma~\ref{lem:inverse_reg_neumann_identity} gives
\[
\widehat{\cC}_\mu^{-1} = \cC_\mu^{-1/2}\bigl(\Id_H - \cC_\mu^{-1/2}(\cC-\widehat{\cC})\cC_\mu^{-1/2}\bigr)^{-1}\cC_\mu^{-1/2}=  \cC_\mu^{-1/2}\widehat{M} \cC_\mu^{-1/2}\,.
\]
Next, we expand in the shared orthonormal eigenbasis $\{\varphi_j\}$ of $\Lambda$ and $\Sigma$ to obtain
\begin{align*}
    (\Sigma')^{1/2}(\Id_H - A_NS_N)\fd &= \mu \sum_{j=1}^\infty \frac{\fd_j\lambda_j^{-1/2}}{(\sigma_j\lambda_j + \mu)^{1/2}}(\Sigma')^{1/2}\Lambda^{1/2}\cC_\mu^{-1/2}\widehat{M}\varphi_j\\
    &=\mu \sum_{j=1}^\infty \frac{\fd_j\lambda_j^{-1/2}}{(\sigma_j\lambda_j + \mu)^{1/2}} \sum_{k=1}^\infty\ip{\varphi_k}{(\Sigma')^{1/2}\Lambda^{1/2}\cC_\mu^{-1/2}\widehat{M}\varphi_j}\varphi_k\\
    &=\mu \sum_{j=1}^\infty \frac{\fd_j\lambda_j^{-1/2}}{(\sigma_j\lambda_j + \mu)^{1/2}} \sum_{k=1}^\infty\frac{(\sigma_k')^{1/2}\lambda_k^{1/2}}{(\sigma_k\lambda_k + \mu)^{1/2}}\ip{\varphi_k}{\widehat{M}\varphi_j}\varphi_k\,.
\end{align*}
By continuity of the inner product,
\begin{align*}
    \ip[\big]{(\Sigma')^{1/2}(\Id_H - A_NS_N)\fd }{\varphi_i}&=\mu \sum_{j=1}^\infty \frac{\fd_j\lambda_j^{-1/2}}{(\sigma_j\lambda_j + \mu)^{1/2}} \frac{(\sigma_i')^{1/2}\lambda_i^{1/2}}{(\sigma_i\lambda_i + \mu)^{1/2}}\ip{\varphi_i}{\widehat{M}\varphi_j}\\
    &=\mu \frac{(\sigma_i')^{1/2}\lambda_i^{1/2}}{(\sigma_i\lambda_i + \mu)^{1/2}}\sum_{j=1}^\infty \frac{\fd_j\lambda_j^{-1/2}}{(\sigma_j\lambda_j + \mu)^{1/2}} \ip{\varphi_i}{\widehat{M}\varphi_j}\,.
\end{align*}
Summing the square of the preceding display over all $i\in\N$ completes the proof.
\end{proof}

Next, we note by direct calculation that $\widehat{M}$ from \eqref{eqn:That_Mhat} satisfies the key identity
\[
\widehat{M}=\Id_H+\widehat{M}\widehat{T}\,.
\]
Thus, the right-hand side of the display \eqref{eqn:bias_series} in Lemma~\ref{lem:bias_series} is bounded above by
\begin{equation}\label{eqn:bias_split}
    \begin{split}
    &2\mu^2\sum_{k=1}^\infty \frac{\sigma_k'\lambda_k}{\sigma_k\lambda_k + \mu}\abs[\Bigg]{\sum_{j=1}^\infty{\frac{\abs{\fd_j}\lambda_j^{-1/2}}{(\sigma_j\lambda_j + \mu)^{1/2}}\abs[\big]{\ip{\varphi_k}{\varphi_j}}}}^2\\
    & \qquad \qquad +
    2\mu^2\sum_{k=1}^\infty \frac{\sigma_k'\lambda_k}{\sigma_k\lambda_k + \mu}\abs[\Bigg]{\sum_{j=1}^\infty{\frac{\abs{\fd_j}\lambda_j^{-1/2}}{(\sigma_j\lambda_j + \mu)^{1/2}}\abs[\big]{\ip{\varphi_k}{\widehat{M}\widehat{T}\varphi_j}}}}^2\\
    & = \underbrace{2\mu^2\sum_{k=1}^\infty \frac{\sigma_k'\abs{\fd_k}^2}{(\sigma_k\lambda_k + \mu)^2}}_{\eqdef \overline{B}_N} + 
    \underbrace{2\mu^2\sum_{k=1}^\infty \frac{\sigma_k'\lambda_k}{\sigma_k\lambda_k + \mu}\abs[\Bigg]{\sum_{j=1}^\infty{\frac{\abs{\fd_j}\lambda_j^{-1/2}}{(\sigma_j\lambda_j + \mu)^{1/2}}\abs[\big]{\ip{\varphi_k}{\widehat{M}\widehat{T}\varphi_j}}}}^2}_{\eqdef \widehat{B}_N}\hspace{-1mm}.
    \end{split}
\end{equation}
We used the fact that the $\{\varphi_j\}$ are orthonormal to obtain the equality. The first term $\overline{B}_N$ is the standard bias term one would expect from a simultaneously diagonalizable linear inverse problem \cite{de2023convergence,knapik2011bayesian}.
The second term $\widehat{B}_N$ is a residual due to finite data. This is the term that we focus on estimating.

To this end, let $\sfE$ be the event from \eqref{eqn:normalized_operator_norm}:
\begin{equation}\label{eqn:var_good_event}
    \sfE=\Bigl\{\norm[\big]{\widehat{T}}_{\cL(H)}\leq 1/2 \Bigr\}\,.
\end{equation}
Fix $\ep>0$ to be determined later. Define another event
\begin{equation}\label{eqn:bias_good_event}
    \sfE_0\defeq \bigl\{\widehat{B}_N\leq \ep\bigr\}\,.
\end{equation}
Let the intersection $\sfE_0\cap \sfE$ be our ``good'' event. On this event, our variance and bias bounds will hold simultaneously. For our results to be meaningful, we must show that $\sfE_0\cap \sfE$ has high probability. Since $\P(\sfE_0\condbar \sfE) = 1-\P(\sfE_0^\comp\condbar \sfE)$, we have
\begin{align}\label{eqn:bias_good_event_lower_bound}
    \begin{split}
        \P(\sfE_0\cap \sfE)&=\P(\sfE_0\condbar \sfE)\P(\sfE)\\
        &= \P(\sfE) - \P(\sfE_0^\comp\condbar \sfE)\P(\sfE)\\
        &= \P(\sfE) - \P(\sfE_0^\comp\cap \sfE)\,.
    \end{split}
\end{align}
Thus, to lower bound the probability of $\sfE_0\cap \sfE$, it suffices to upper bound
\begin{equation*}
    \P(\sfE_0^\comp\cap \sfE)=\P\bigl(\{\widehat{B}_N>\ep\}\cap \sfE\bigr)\,.
\end{equation*}
On the event $\sfE$, it holds that $\norm{\widehat{M}}_{\cL(H)}\leq 2$ by \eqref{eqn:good_set_inverse_bound}. This, the symmetry of $\widehat{M}$, and the Cauchy--Schwarz inequality imply that
\begin{align*}
    \widehat{B}_N &\leq 2\mu^2\sum_{k=1}^\infty \frac{\sigma_k'\lambda_k}{\sigma_k\lambda_k + \mu}\abs[\Bigg]{\sum_{j=1}^\infty {\frac{\abs{\fd_j}\lambda_j^{-1/2}}{(\sigma_j\lambda_j + \mu)^{1/2}}\norm[\big]{\widehat{M}\varphi_k}\norm[\big]{\widehat{T}\varphi_j}}}^2\\
    &\leq 8 \, \underbrace{\abs[\Bigg]{\sum_{j=1}^\infty \frac{\mu^{1/2}\abs{\fd_j}\lambda_j^{-1/2}}{(\sigma_j\lambda_j + \mu)^{1/2}}\norm[\big]{\widehat{T}\varphi_j}}^2}_{\eqdef (\cI_N)^2}
    \ \sum_{k=1}^\infty \frac{\mu \sigma_k'\lambda_k}{\sigma_k\lambda_k + \mu}
\end{align*}
on the event $\sfE$. Notice that in the last line of the preceding display, the factor $(\cI_N)^2$ is multiplying our high probability upper bound \eqref{eqn:variance_upper_bound} on $V_N$. Hence, for the contribution from $\widehat{B}_N$ to be negligible relative to the upper bound on $V_N$, it suffices to show that $\cI_N\lesssim 1$ for all sufficiently large $N$ with high probability. Indeed, for some $\tau>0$ to be determined later, choose
\begin{equation}\label{eqn:app_eps_bhat}
    \ep\defeq 8\tau^2 \sum_{k=1}^\infty \frac{\mu \sigma_k'\lambda_k}{\sigma_k\lambda_k + \mu}>0
\end{equation}
in the definition \eqref{eqn:bias_good_event} of $\sfE_0$. Then the monotonicity of probability measure (i.e., if $\mathsf{A}_1\subseteq \mathsf{A}_2$, then $\P(\mathsf{A}_1)\leq\P(\mathsf{A}_2)$) implies that
\begin{equation}\label{eqn:app_good_event_cap_tau}
    \P(\sfE_0^\comp\cap \sfE)\leq \P(\{\cI_N>\tau\}\cap \sfE)\leq \P\{\cI_N>\tau\}\,.
\end{equation}
In the rest of the argument, we develop an upper tail bound for the random series $\cI_N$ to control the rightmost expression in \eqref{eqn:app_good_event_cap_tau}. To ease the notation, we write
\begin{align}\label{eqn:app_series_nuisance}
    \cI_N=\sum_{j=1}^\infty s_j\norm[\big]{\widehat{T}\varphi_j}\,,\qw s_j\defeq \frac{\mu^{1/2}\abs{\fd_j}\lambda_j^{-1/2}}{(\sigma_j\lambda_j + \mu)^{1/2}}\,.
\end{align}
Our strategy is to show that:\footnote{Appendix~\ref{appx:proofs_linear_sub} reviews subgaussian and subexponential random variables.}
\begin{enumerate}[label=\textbf{Step~\arabic*.}, leftmargin=4\parindent,topsep=1.67ex,itemsep=0.33ex,partopsep=1ex,parsep=1ex]
    \item \textit{the individual summands of $\cI_N$ are subexponential random variables},\label{item:step1}
    \item \textit{the entire random series $\cI_N$ is subexponential}, \textit{and}\label{item:step2}
    \item \textit{the entire random series $\cI_N$ has a fast tail decay}.\label{item:step3}
\end{enumerate}

We now proceed with this three step proof procedure.

\paragraph*{\ref{item:step1}}
Recalling the definition of $\widehat{T}$ from \eqref{eqn:That_Mhat}, we see by the symmetry of $\cC_\mu^{-1/2}$ and $\Lambda^{1/2}$ that
\begin{align*}
    -\widehat{T}&=\cC_\mu^{-1/2}(\widehat{\cC}-\cC)\cC_\mu^{-1/2}\\
    &=\cC_\mu^{-1/2}\frac{1}{N}\sum_{n=1}^N\Bigl(\Lambda^{1/2}u_n\otimes \Lambda^{1/2}u_n - \E\bigl[\Lambda^{1/2}u_1\otimes \Lambda^{1/2}u_1\bigr]\Bigr) \cC_\mu^{-1/2}\\
    &=\frac{1}{N}\sum_{n=1}^N \bigl(v_n\otimes v_n - \E[v_1\otimes v_1]\bigr)\,,\qw v_n\defeq \cC_\mu^{-1/2}\Lambda^{1/2}u_n
\end{align*}
and $\{u_n\}_{n=1}^N\sim\nu^{\otimes N}$. Thus, it holds that
\begin{align}\label{eqn:app_summands_iid}
    -\widehat{T}\varphi_j = \frac{1}{N}\sum_{n=1}^N\zeta_{jn}\,,\qw \zeta_{jn}\defeq \ip{v_n}{\varphi_j}v_n-\E\bigl[\ip{v_1}{\varphi_j}v_1\bigr]\,.
\end{align}
That is, $\widehat{T}\varphi_j$ is a sum of i.i.d. random vectors in the Hilbert space $H$. To show that the scalar random variables $\norm{\widehat{T}\varphi_j}$ are subexponential, we first need to control the subexponential norm of the $\norm{\zeta_{jn}}$. The next lemma accomplishes this task.

\begin{lemma}[moments]\label{lem:summand_bernstein_bound}
    Under Assumption~\ref{ass:data_end_to_end}~and~\ref{ass:data_input_kl_expand}, for every $j$ it holds that
    \begin{align}\label{eqn:summand_bernstein_bound}
        \E\norm{\zeta_{j1}}^\ell\leq \Bigl(4e m^2\rho_j\sqrt{\tr[\big]{\cC_\mu^{-1}\cC}}\Bigr)^\ell \ell!\qfa \ell\in\{2,3,\ldots\}\,,
    \end{align}
    where $m\geq 0$ is as in \eqref{eqn:data_input_subg_uniform_bound} and
    \begin{align}\label{eqn:app_rho_j}
        \rho_j\defeq \sqrt{\frac{\sigma_j\lambda_j}{\mu+\sigma_j\lambda_j}}\,.
    \end{align}
\end{lemma}
\begin{proof}
    Fix an integer $\ell\geq 2$. The inequality $\abs{a+b}^\ell\leq 2^{\ell-1}(\abs{a}^\ell +\abs{b}^\ell)$ shows that
    \begin{align*}
        \E\norm{\zeta_{j1}}^\ell&\leq 2^{\ell-1}\bigl(\E\norm{\ip{v_1}{\varphi_j}v_1}^\ell + \norm{\E[\ip{v_1}{\varphi_j}v_1]}^\ell\bigr)\\
        &\leq 2^\ell \E\norm{\ip{v_1}{\varphi_j}v_1}^\ell\,.
    \end{align*}
    The second line is due to Jensen's inequality. Let $u$ be an i.i.d. copy of $u_1\sim \nu$, so that $v\defeq \cC_\mu^{-1/2}\Lambda^{1/2}u$ is an i.i.d. copy of $v_1$. By Assumption~\ref{ass:data_input_kl_expand} and the assumption~\ref{item:ass_ee_diag} that $\Lambda$ and $\Sigma$ share the orthonormal eigenbasis $\{\varphi_j\}$, it holds that
    \begin{align*}
        v=\sum_{j=1}^\infty\rho_j z_j\varphi_j\,,\qw \rho_j= \sqrt{\frac{\sigma_j\lambda_j}{\mu+\sigma_j\lambda_j}}\geq 0\,.
    \end{align*}
    Thus, $\ip{v}{\varphi_j}=\rho_j z_j$ and $\E\norm{\ip{v_1}{\varphi_j}v_1}^\ell=\E\norm{\ip{v}{\varphi_j}v}^\ell$ equals
    \begin{align*}
        \rho_j^\ell\E\bigl[\abs{z_j}^\ell\norm{v}^\ell\bigr] &= \rho_j^\ell\E\bigl[\abs{z_j}^\ell\bigl(\norm{v}^2\bigr)^{\ell/2}\bigr]\\
        &=\rho_j^\ell\E\Biggl[\abs{z_j}^\ell\Biggl(\sum_{k=1}^\infty\rho_k^2z_k^2\Biggr)^{\ell/2}\Biggr]\\
        &=\rho_j^\ell\E\abs[\Bigg]{\sum_{k=1}^\infty\rho_k^2z_j^2z_k^2}^{\ell/2}.
    \end{align*}
    The triangle inequality in the Banach space $L^{\ell/2}_{\P}(\Omega;\R)$ (since $\ell\geq 2$) and the Cauchy--Schwarz inequality yields
    \begin{align*}
        \norm[\Bigg]{\sum_{k=1}^\infty\rho_k^2z_j^2z_k^2}_{L^{\ell/2}_{\P}}&\leq \sum_{k=1}^\infty\rho_k^2\norm[\big]{z_j^2z_k^2}_{L^{\ell/2}_{\P}}\\
        &= \sum_{k=1}^\infty\rho_k^2\Bigl(\E\bigl[\abs{z_j}^\ell\abs{z_k}^\ell\bigr]\Bigr)^{2/\ell}\\
        &\leq  \sum_{k=1}^\infty\rho_k^2\bigl(\E\abs{z_j}^{2\ell}\bigr)^{1/\ell} \bigl(\E\abs{z_k}^{2\ell}\bigr)^{1/\ell}\,.
    \end{align*}
    By the definition \eqref{eqn:norm_subexp} of the subexponential norm, Lemma~\ref{lem:squared_subg_bound}, and \eqref{eqn:data_input_subg_uniform_bound}, we have
    \begin{align*}
        \sum_{k=1}^\infty\rho_k^2\bigl(\E\abs{z_j}^{2\ell}\bigr)^{1/\ell} \bigl(\E\abs{z_k}^{2\ell}\bigr)^{1/\ell}
        &\leq \sum_{k=1}^\infty\rho_k^2\bigl(\ell\norm{z_j^2}_{\psi_1}\bigr) \bigl(\ell\norm{z_k^2}_{\psi_1}\bigr)\\
        &\leq \sum_{k=1}^\infty\rho_k^2\bigl(2\ell\norm{z_j}^2_{\psi_2}\bigr) \bigl(2\ell\norm{z_k}^2_{\psi_2}\bigr)\\
        &\leq 4\ell^2m^4\sum_{k=1}^\infty\frac{\sigma_k\lambda_k}{\mu+\sigma_k\lambda_k}\,.
    \end{align*}
    Taking the $\ell/2$-th power and putting together the pieces, we deduce that
    \begin{align*}
        \E\norm{\zeta_{j1}}^\ell&\leq (2\rho_j)^\ell\left(2m^2\sqrt{\tr[\big]{\cC_\mu^{-1}\cC}}\right)^\ell \ell^\ell\,.
    \end{align*}
    Recalling from Stirling's formula that $(\ell/e)^\ell\leq \ell!$ completes the proof.
\end{proof}

We need the following proposition that relies heavily on \cite[Theorem 1, p. 144]{pinelis1986remarks}.
\begin{proposition}[Hilbert space norm of independent sums]\label{prop:app_mgf_subexp_partial_sum}
    Let $\{X_n\}_{n=1}^\infty$ be a sequence of independent centered random vectors with values in a separable Hilbert space $(\cX,\ip{\cdot}{\cdot},\norm{\slot})$. Let $N\in\N$ be arbitrary. If the Bernstein moment condition
    \begin{align}\label{eqn:app_bernstein_moment}
        \sum_{n=1}^N\E\norm{X_n}^\ell\leq\frac{1}{2}\ell!\sigma^2b^{\ell-2} \qfa \ell\in\{2,3,\ldots\}
    \end{align}
    holds for some $\sigma>0$ and $b>0$, then the partial sums $\cS_N\defeq \sum_{n=1}^NX_n$ satisfy the subexponential condition $\norm{\cS_N}\in \SE{2\sigma^2}{2b}$, that is,
    \begin{align}\label{eqn:mgf_subexp_partial_sum}
        \E e^{\lambda(\norm{\cS_N}-\E \norm{\cS_N})}\leq e^{\lambda^2\sigma^2} \qfa \abs{\lambda}\leq \frac{1}{2b}\,.
    \end{align}
\end{proposition}
\begin{proof}
Since $\cX$ is a separable Banach space, \cite[Theorem 1, p. 144]{pinelis1986remarks} shows that
\begin{align*}
    \E e^{\abs{\lambda}(\norm{\cS_N}-\E\norm{\cS_N})}\leq \exp\left(\sum_{n=1}^N\E\Bigl[e^{\abs{\lambda}\norm{X_n-\E X_n}} - 1 - \abs{\lambda} \norm{X_n-\E X_n}\Bigr]\right)
\end{align*}
for all $\lambda\in\R$.
Under the Bernstein moment condition \eqref{eqn:app_bernstein_moment}, we compute using the Taylor expansion of the exponential function that
\begin{align*}
    \sum_{n=1}^N\E\Bigl[e^{\abs{\lambda}\norm{X_n-\E X_n}} - 1 - \abs{\lambda} \norm{X_n-\E X_n}\Bigr]&=
    \sum_{\ell=2}^\infty\sum_{n=1}^N\frac{\abs{\lambda}^\ell\E\norm{X_n-\E X_n}^\ell}{\ell!}\\
    &\leq \frac{1}{2}\lambda^2\sigma^2\sum_{\ell=2}^\infty\bigl(\abs{\lambda}b\bigr)^{\ell-2}\\
    &=\frac{\lambda^2\sigma^2}{2(1-\abs{\lambda}b)}\\
    &\leq \lambda^2\sigma^2
\end{align*}
provided that $\abs{\lambda}b\leq 1/2$. Thus, it holds that
\begin{align*}
    \E e^{\abs{\lambda}(\norm{\cS_N}-\E\norm{\cS_N})}\leq e^{\lambda^2\sigma^2} \qfa \abs{\lambda}\leq 1/(2b)\,.
\end{align*}
For $\norm{\cS_N}$ to be subexponential, by the definition~\eqref{eqn:mgf_subexp} we also need to show that
\begin{align*}
    \E e^{-\abs{\lambda}(\norm{\cS_N}-\E\norm{\cS_N})}\leq e^{\lambda^2\sigma^2} \qfa \abs{\lambda}\leq 1/(2b)\,.
\end{align*}
But since $\norm{\cS_N}\geq 0$ almost surely, the one-sided Bernstein moment generating function bound \cite[Proposition 2.14, Equation (2.22a), p. 31]{wainwright2019high} applied to $-\norm{\cS_N}$ yields
\begin{align*}
    \E e^{-\abs{\lambda}(\norm{\cS_N}-\E\norm{\cS_N})}\leq e^{\lambda^2\E\norm{\cS_N}^2/2}\qfa \abs{\lambda}<\infty\,.
\end{align*}
It remains to bound $\E\norm{\cS_N}^2$ in terms of $\sigma^2$. Using the facts that $\E X_n=0$ and $\cX$ is a Hilbert space yield
\begin{align*}
    \E\norm{\cS_N}^2=\sum_{n=1}^N\sum_{n'=1}^N\E\ip{X_n}{X_{n'}}=\sum_{n=1}^N\E\norm{X_n}^2\leq \sigma^2\,.
\end{align*}
We used the Bernstein moment condition \eqref{eqn:app_bernstein_moment} with $\ell=2$ to obtain the final inequality. Noting that $\abs{\lambda}<\infty$ implies $\abs{\lambda}\leq 1/(2b)$ and that $\sigma^2/2\leq\sigma^2$ completes the proof.
\end{proof}

We are now in a position to prove the following lemma about the empirical sums.
\begin{lemma}[norm of empirical sum is subexponential]\label{lem:app_subexp_emp_sum}
    Fix $j\in\N$. Let 
    \begin{align}\label{eqn:app_varsigma_defn}
        \varsigma_j\defeq 8em^2\rho_j\sqrt{\tr[\big]{\cC_\mu^{-1}\cC}}
    \end{align}
    with $\rho_j$ as in \eqref{eqn:app_rho_j}. Under Assumptions~\ref{ass:data_end_to_end} and \ref{ass:data_input_kl_expand}, it holds that
    \begin{align}\label{eqn:app_subexp_emp_sum}
        \norm[\big]{\widehat{T}\varphi_j}\in \SE[\bigg]{\frac{2\varsigma_j^2}{N}}{\frac{2\varsigma_j}{N}}\,.
    \end{align}
\end{lemma}
\begin{proof}
    For fixed $j$, the independence of the $\{\zeta_{jn}\}_{n=1}^N$ and Lemma~\ref{lem:summand_bernstein_bound} imply that
    \begin{align*}
        \sum_{n=1}^N\E\norm{\zeta_{jn}}^\ell=N\E\norm{\zeta_{j1}}^\ell&\leq 2\biggl(\frac{N}{2}\biggr)\Bigl(4e m^2\rho_j\sqrt{\tr[\big]{\cC_\mu^{-1}\cC}}\Bigr)^\ell \ell!\\
        &\leq \frac{N}{2}\ell!\varsigma_j^\ell\\
        &= \frac{1}{2}\ell!\bigl(N\varsigma_j^2\bigr)\varsigma_j^{\ell-2}
    \end{align*}
    for any $\ell\in \{2,3,\ldots\}$ (using $2\leq 2^\ell$ to get to the second line).
    Recalling from \eqref{eqn:app_summands_iid} that $\E \zeta_{j1}=0$, Proposition~\ref{prop:app_mgf_subexp_partial_sum} applied with $\sigma^2= N\varsigma_j^2$ and $b=\varsigma_j$ in \eqref{eqn:app_bernstein_moment} yields
    \begin{align*}
        \E e^{N\lambda(\norm{\widehat{T}\varphi_j}-\E \norm{\widehat{T}\varphi_j})}\leq e^{\lambda^2 N\varsigma_j^2} \qfa \abs{\lambda}\leq \frac{1}{2\varsigma_j}\,.
    \end{align*}
    Replacing $\lambda$ with $\lambda/N$ and recalling Definition~\ref{def:app_subexp} gives the asserted result.
\end{proof}

\paragraph*{\ref{item:step2}}
The $\ell^1(\N)\defeq \ell^1(\N;\R)$ norm of the nonnegative sequence $\{s_j\varsigma_j\}_{j\in\N}$ plays a central role in the analysis to follow. To this end, denote
\begin{align}
    w^{(N)}\defeq \{w_j^{(N)}\}_{j\in\N}\,, \qw w_j^{(N)}\defeq s_j\varsigma_j \qfa j\in\N \,.
\end{align}
We now upper bound the deterministic series $\norm{w^{(N)}}_{\ell^1(\N)}$.

\begin{lemma}[deterministic series convergence rate]\label{lem:app_series_w_upper_bound}
Under Assumption~\ref{ass:data_end_to_end} and \ref{ass:truth_regularity}, it holds that
    \begin{align}\label{eqn:app_series_w_upper_bound}
    \frac{1}{N}\norm[\big]{w^{(N)}}_{\ell^1(\N)}^2
    \lesssim \norm[\big]{\fd}^2_{\cH^s}\times 
        \begin{cases}
            N^{-\bigl(\frac{\al + s -1}{\al + p}\bigr)}\, , & \textit{if }\,\frac{\al+s-1/2}{\al+p}<2\,,\\[1mm]
            N^{-\bigl(1+\frac{\al + p -1/2}{\al + p}\bigr)}\log 2N\, , & \textit{if }\, \frac{\al+s-1/2}{\al+p}=2\,,\\[1mm]
            N^{-\bigl(1+\frac{\al + p -1/2}{\al + p}\bigr)}\, , & \textit{if }\,\frac{\al+s-1/2}{\al+p}>2
        \end{cases}
    \end{align}
    for all $N\in\N$.
\end{lemma}
\begin{proof}
    Recalling the definitions of $\{s_j\}_{j\in\N}$ \eqref{eqn:app_series_nuisance} and $\{\varsigma_j\}_{j\in\N}$ \eqref{eqn:app_varsigma_defn}, we have
    \begin{align*}
        \norm[\big]{w^{(N)}}_{\ell^1(\N)}=\sum_{j=1}^\infty s_j\varsigma_j
        &=8em^2\sqrt{\tr[\big]{\cC_\mu^{-1}\cC}}\sum_{j=1}^\infty \frac{\mu^{1/2}\abs{\fd_j}\sigma_j^{1/2}}{\mu + \sigma_j\lambda_j}\\
        &\simeq N^{1/2}\sqrt{\tr[\big]{\cC_\mu^{-1}\cC}}\sum_{j=1}^\infty \frac{\abs{\fd_j}\sigma_j^{1/2}}{1 + N\sigma_j\lambda_j}\,.
    \end{align*}
    The preceding display, the asymptotics of $\{\sigma_j\}_{j\in\N}$ and $\{\lambda_j\}_{j\in\N}$ from Assumption~\ref{ass:data_end_to_end}, and the Cauchy--Schwarz inequality imply that
    \begin{align*}
        \frac{1}{N}\norm[\big]{w^{(N)}}_{\ell^1(\N)}^2\lesssim
        \norm[\big]{\fd}^2_{\cH^s}\tr[\big]{\cC_\mu^{-1}\cC}\frac{1}{N^2} +
        \tr[\big]{\cC_\mu^{-1}\cC}\abs[\Bigg]{\sum_{j=1}^\infty \frac{\abs{\fd_j}j^{-\al}}{1 + N j^{-2(\al+p)}}}^2 .
    \end{align*}
    By Lemma~\ref{lem:app_effective_dim_trace}, it holds that $\tr{\cC_\mu^{-1}\cC}\simeq N^{1/(2(\al+p))}$ because $\mu=\gamma^2/N$~\eqref{eqn:app_defn_mu_C_hat_reg}. The series factor in the second term in the preceding display satisfies the estimate
    \begin{align*}
        \abs[\Bigg]{\sum_{j=1}^\infty \frac{\abs{\fd_j}j^{-\al}}{1 + N j^{-2(\al+p)}}}^2\leq \norm[\big]{\fd}^2_{\cH^s}\sum_{j=1}^\infty \frac{j^{-2(\al + s)}}{(1 + N j^{-2(\al+p)})^2}
    \end{align*}
    by the Cauchy--Schwarz inequality. The rightmost series converges (because $2(\al+s)>2>1$ by the last assertion of Assumption~\ref{ass:truth_regularity}) and is bounded above by a constant (independent of $N$) times
    \begin{align*}
        \begin{cases}
            N^{-\bigl(\frac{\al + s -1/2}{\al + p}\bigr)}\, , & \textit{if }\,\frac{\al+s-1/2}{\al+p}<2\,,\\[1mm]
            N^{-2}\log 2N\, , & \textit{if }\, \frac{\al+s-1/2}{\al+p}=2\,,\\[1mm]
            N^{-2}\, , & \textit{if }\,\frac{\al+s-1/2}{\al+p}>2
        \end{cases}
    \end{align*}
    by Lemma~\ref{lem:app_series_rate2} (applied with $t=2(\al+s)>1$, $u=2(\al+p)>1>0$, and $v=2$). Bounding $N^{-2}$ above by the preceding display and multiplying this bound by the $N^{1/(2(\al+p))}$ trace bound completes the proof.
\end{proof}

Combining the previous results with those of \cite[Appendix B, pp. 30--31]{de2023convergence} establishes that the entire series $\cI_N$~\eqref{eqn:app_series_nuisance} is a real-valued subexponential random variable.
\begin{lemma}[random series: subexponential]\label{lem:app_series_subexp}
    Let Assumptions~\ref{ass:data_end_to_end}, \ref{ass:data_input_kl_expand}, and \ref{ass:truth_regularity} be satisfied and $\cI_N$ be defined as in \eqref{eqn:app_series_nuisance}. It holds that
    \begin{align}\label{eqn:app_series_subexp}
        \cI_N\in\SE[\bigg]{\frac{2}{N}\norm[\big]{w^{(N)}}_{\ell^1(\N)}^2}{\frac{2}{N}\norm[\big]{w^{(N)}}_{\ell^1(\N)}} \qfa N\in\N\,.
    \end{align}
\end{lemma}
\begin{proof}
    Fix $N\in\N$. A change of variables in Definition~\ref{def:app_subexp} of subexponential and Lemma~\ref{lem:app_subexp_emp_sum} imply that $s_j\norm{\widehat{T}\varphi_j}\in \SE{2s_j^2\varsigma_j^2/N}{2s_j\varsigma_j/N}$ for any $j$. For any $J\in\N$, let $\cI_N^{(J)}\defeq \sum_{j=1}^Js_j\norm{\widehat{T}\varphi_j}$. Even though the summands $\{s_j\norm{\widehat{T}\varphi_j}\}_{j\in\N}$ are a dependent sequence of random variables, \cite[Lemma B.5, p. 30]{de2023convergence} shows that
    \begin{align*}
        \cI_N^{(J)} \in \SE[\Bigg]{\frac{2}{N}\abs[\bigg]{\sum_{j=1}^Js_j\varsigma_j}^2}{\frac{2}{N}\sum_{j=1}^Js_j\varsigma_j}\,.
    \end{align*}
    Next, Jensen's inequality yields the bound $\sum_{j=1}^\infty \E\bigl[s_j \norm{\widehat{T}\varphi_j}\bigr]\leq \sum_{j=1}^\infty s_j\sqrt{\E\norm{\widehat{T}\varphi_j}^2}$. For fixed $j\in\N$, we compute
    \begin{align*}
        \E\norm[\big]{\widehat{T}\varphi_j}^2=\frac{1}{N^2}\sum_{n=1}^N\sum_{n'=1}^N\E\ip{\zeta_{jn}}{\zeta_{jn'}}
        &=\frac{1}{N}\E\norm{\zeta_{j1}}^2\\
        &\leq \frac{2}{N}\Bigl(4em^2\rho_j\sqrt{\tr{\cC_\mu^{-1}\cC}}\Bigr)^2\\
        &=\frac{\varsigma_j^2}{2N}\,.
    \end{align*}
    In the preceding display, we used orthogonality, independence, and the Bernstein moment condition \eqref{eqn:summand_bernstein_bound} (Lemma~\ref{lem:summand_bernstein_bound} applied with $\ell=2$). Thus,
    \begin{align}\label{eqn:app_series_mean_bound}
        \sum_{j=1}^\infty \E\Bigl[s_j \norm[\big]{\widehat{T}\varphi_j}\Bigr]\leq \frac{1}{\sqrt{2N}}\sum_{j=1}^\infty s_j\varsigma_j\,.
    \end{align}
    Since the right-hand side of \eqref{eqn:app_series_mean_bound} is finite by Lemma~\ref{lem:app_series_w_upper_bound}, a monotone convergence argument~\cite[Lemma B.3, p. 30]{de2023convergence} shows that
    \begin{align*}
        \P\Bigl\{\lim_{J\to\infty} \cI_N^{(J)} = \cI_N \Bigr\}=1\,.
    \end{align*}
    Using this almost sure convergence and again recalling Definition~\ref{def:app_subexp}, it holds that
    \begin{align*}
        \E\exp\bigl(\lambda(\cI_N - \E\cI_N)\bigr)&= \E\lim_{J\to\infty}\exp\Bigl(\lambda\bigl(\cI_N^{(J)} - \E\cI_N^{(J)}\bigr)\Bigr) \\
        &\leq \liminf_{J\to\infty} \E\exp\Bigl(\lambda\bigl(\cI_N^{(J)} - \E\cI_N^{(J)}\bigr)\Bigr)\\
        &\leq \liminf_{J\to\infty} \exp\biggl(\frac{\lambda^2}{2} \frac{2}{N}\abs[\bigg]{\sum_{j=1}^Js_j\varsigma_j}^2\biggr)\\
        &=\exp\Bigl(\frac{\lambda^2}{2}\frac{2}{N}\norm[\big]{w^{(N)}}_{\ell^1(\N)}^2\Bigr)
    \end{align*}
    for all $\abs{\lambda}\leq (\frac{2}{N}\norm{w^{(N)}}_{\ell^1})^{-1}$ because $\frac{2}{N}\sum_{j=1}^Jw_j^{(N)}\leq \frac{2}{N}\norm{w^{(N)}}_{\ell^1}$ for any $J\in\N$. In the preceding display, the first line is due to the identity $\E\lim_J\cI_N^{(J)} = \lim_J\E\cI_N^{(J)}$ (which follows from monotone convergence), the second due to Fatou's lemma, and the third due to the fact that $\cI_N^{(J)}$ is subexponential (hence \eqref{eqn:mgf_subexp} holds). Therefore, the entire random series $\cI_N$ satisfies Definition~\ref{def:app_subexp} and the proof is complete.
\end{proof}

\paragraph*{\ref{item:step3}}
A consequence of the previous lemma is a strong tail decay bound for $\cI_N$.
\begin{lemma}[random series: tail bound]\label{lem:app_series_full_tail_bound}
    Let Assumptions~\ref{ass:data_end_to_end}, \ref{ass:data_input_kl_expand}, and \ref{ass:truth_regularity} be satisfied and $\cI_N$ be defined as in \eqref{eqn:app_series_nuisance}. For any $N\in\N$, it holds that
    \begin{align}\label{eqn:app_series_full_tail_bound}
        \P\biggl\{\cI_N\geq \frac{\norm{w^{(N)}}_{\ell^1(\N)}}{\sqrt{2N}} + t \biggr\}\leq \exp\biggl(-\frac{Nt^2}{4 \norm{w^{(N)}}_{\ell^1(\N)}^2}\biggr)
    \end{align}
    for all  $0\leq t \leq \norm{w^{(N)}}_{\ell^1(\N)}$.
\end{lemma}
\begin{proof}
By Lemma~\ref{lem:app_series_subexp} and \cite[Proposition 2.9, p. 26]{wainwright2019high}, it holds that
\begin{align*}
    \P\{\cI_N\geq \E\cI_N + t\}\leq \exp\biggl(-\frac{Nt^2}{4 \norm{w^{(N)}}_{\ell^1(\N)}^2}\biggr) \qfa 0\leq t \leq \norm[\big]{w^{(N)}}_{\ell^1(\N)}\,.
\end{align*}
By the Fubini--Tonelli theorem and \eqref{eqn:app_series_mean_bound}, we obtain
\begin{align*}
    \E\cI_N=\sum_{j=1}^\infty \E\Bigl[s_j \norm[\big]{\widehat{T}\varphi_j}\Bigr]\leq \frac{\norm{w^{(N)}}_{\ell^1(\N)}}{\sqrt{2N}}\,.
\end{align*}
The assertion \eqref{eqn:app_series_full_tail_bound} follows from the monotonicity of probability measure.
\end{proof}

The previous lemma implies a high probability uniform upper bound on $\cI_N$.
\begin{lemma}[random series: uniform bound]\label{lem:app_series_uniform_bound}
    Let Assumptions~\ref{ass:data_end_to_end}, \ref{ass:data_input_kl_expand}, and \ref{ass:truth_regularity} be satisfied and $\cI_N$ be defined as in \eqref{eqn:app_series_nuisance}. There exists $c_0>1$ and $c\in(0,1/4)$, both independent of $N$ and $\fd$, such that
    \begin{align}\label{eqn:app_series_uniform_bound}
        \P\bigl\{\cI_N\geq c_0\norm{\fd}_{\cH^s}\bigr\} \leq \exp\Bigl(-c N^{\min\bigl(1,\frac{\al+s-1}{\al+p}\bigr)}\Bigr) \qfa N\in\N\,.
    \end{align}
\end{lemma}
\begin{proof}
    Let $t\defeq \min(\norm{\fd}_{\cH^s}, \norm{w^{(N)}}_{\ell^1})$. Clearly $0\leq t\leq \norm{w^{(N)}}_{\ell^1}$. Also, $t\leq \norm{\fd}_{\cH^s}$ so that monotonicity of probability measure yields
    \begin{align*}
        \P\biggl\{\cI_N\geq \frac{\norm{w^{(N)}}_{\ell^1}}{\sqrt{2N}} + \norm{\fd}_{\cH^s} \biggr\}&\leq
        \P\biggl\{\cI_N\geq \frac{\norm{w^{(N)}}_{\ell^1}}{\sqrt{2N}} + t \biggr\}\\
        &\leq \exp\biggl(-\frac{N\min\bigl(\norm{\fd}_{\cH^s}^2, \norm{w^{(N)}}_{\ell^1}^2\bigr)}{4 \norm{w^{(N)}}_{\ell^1}^2}\biggr)\\
        &=
        \begin{cases}
            e^{-N\norm{\fd}_{\cH^s}^2/(4\norm{w^{(N)}}_{\ell^1}^{2})} , & \textit{if }\,\norm{w^{(N)}}_{\ell^1}\geq \norm{\fd}_{\cH^s},\\
            e^{-N/4} , & \textit{if }\,\norm{w^{(N)}}_{\ell^1}<\norm{\fd}_{\cH^s}.
        \end{cases}
    \end{align*}
    The second inequality is due to Lemma~\ref{lem:app_series_full_tail_bound}. Next, we upper bound the probability in the first case $\norm{w^{(N)}}_{\ell^1}\geq \norm{\fd}_{\cH^s}$ by bounding $N^{-1}\norm{w^{(N)}}_{\ell^1}^2$ from above. To do so, let $\delta\defeq (\al+p-\frac{1}{2})/(\al+p)$. Then $\delta>0$ because $\al+p>1/2$ by Assumption~\ref{ass:data_end_to_end}. Clearly $N^{-\delta}\leq 1$. Since $x\mapsto \log 2x$ is slowly varying, $\lim_{N\to\infty}N^{-\delta}\log 2N=0$. Hence, $\sup_{N\in\N}N^{-\delta}\log 2N<\infty$. Lemma~\ref{lem:app_series_w_upper_bound} then yields
    \begin{align*}
        \frac{1}{N}\norm[\big]{w^{(N)}}_{\ell^1}^2 &\lesssim \norm{\fd}^2_{\cH^s}\times
        \begin{cases}
            N^{-\bigl(\frac{\al + s -1}{\al + p}\bigr)}\, , & \textit{if }\,\frac{\al+s-1/2}{\al+p}<2\,,\\[1mm]
            N^{-1}\bigl(N^{-\delta}\log 2N\bigr)\, , & \textit{if }\, \frac{\al+s-1/2}{\al+p}=2\,,\\[1mm]
            N^{-1}N^{-\delta}\, , & \textit{if }\,\frac{\al+s-1/2}{\al+p}>2
        \end{cases}
        \\
        & \lesssim \norm{\fd}^2_{\cH^s}\times
         \begin{cases}
            N^{-\bigl(\frac{\al + s -1}{\al + p}\bigr)}\, , & \textit{if }\,\frac{\al+s-1/2}{\al+p}<2\,,\\[1mm]
            N^{-1}\, , & \textit{if }\,\frac{\al+s-1/2}{\al+p}\geq 2
        \end{cases}\\
        &\leq \norm{\fd}^2_{\cH^s} \, N^{-\min\bigl(1, \frac{\al + s -1}{\al + p}\bigr)}
    \end{align*}
    for all $N\in\N$. It follows that there exists $c'>0$ such that
    \begin{align*}
        \exp\biggl(-\frac{N\norm{\fd}_{\cH^s}^2}{4\norm{w^{(N)}}_{\ell^1}^2}\biggr)&=\exp\biggl(-\frac{\norm{\fd}_{\cH^s}^2}{4\norm{w^{(N)}}_{\ell^1}^2/N}\biggr)\\
        &\leq \exp\Bigl(-c' N^{\min\bigl(1,\frac{\al+s-1}{\al+p}\bigr)}/4\Bigr)\,.
    \end{align*}
    For the second case $\norm{w^{(N)}}_{\ell^1}< \norm{\fd}_{\cH^s}$, taking the minimum yields the bound
    \begin{align*}
        e^{-N/4}\leq \exp\Bigl(-N^{\min\bigl(1,\frac{\al+s-1}{\al+p}\bigr)}/4\Bigr)\,.
    \end{align*}
    Writing $c\defeq \min(1,c')/4\in (0,1/4)$, it follows that
    \begin{align*}
        \max\bigl(e^{-N\norm{\fd}_{\cH^s}^2/(4\norm{w^{(N)}}_{\ell^1}^{2})}, e^{-N/4}\bigr)\leq \exp\Bigl(-c N^{\min\bigl(1,\frac{\al+s-1}{\al+p}\bigr)}\Bigr)
    \end{align*}
    for all $N\in\N$. The right-hand side of the preceding display is thus an upper bound to  $\P\{\cI_N\geq \norm{w^{(N)}}_{\ell^1}/\sqrt{2N} + \norm{\fd}_{\cH^s}\}$. To complete the proof, notice that $\norm{w^{(N)}}_{\ell^1(\N)}/\sqrt{2N}\leq (c'/2) \norm{\fd}_{\cH^s}N^{-\min(1,(\al+s-1)/(\al+p))/2}\leq (c'/2)\norm{\fd}_{\cH^s}$ because $\al+s>1$ by Assumption~\ref{ass:truth_regularity}. By monotonicity of probability measure, this implies the asserted result \eqref{eqn:app_series_uniform_bound} with $c_0\defeq 1 + c'/2$.
\end{proof}

This completes \ref{item:step1}, \ref{item:step2}, and~\ref{item:step3}
With a uniform upper bound on $\cI_N$ in hand from the previous lemma, the claimed bound on the bias in Proposition~\ref{prop:bias} follows easily. The details are provided in the following proof.
\begin{proof}[Proof of Proposition~\ref{prop:bias}]
    Recalling the event $\sfE_0$ from \eqref{eqn:bias_good_event} with $\ep$ as in \eqref{eqn:app_eps_bhat}, choose $\tau=c'\norm{\fd}_{\cH^s}$ with $c'>1$ equal to the constant $c_0$ in the hypotheses of Lemma~\ref{lem:app_series_uniform_bound}. On the good event $\sfE_0\cap \sfE$ from \eqref{eqn:bias_good_event} and \eqref{eqn:var_good_event}, it holds by \eqref{eqn:bias_split} that $B_N \leq \overline{B}_N + \ep$ which is precisely the claimed upper bound \eqref{eqn:bias_upper_bound} with $c_0\defeq 8(c')^2$.
    It remains to lower bound the probability of $\sfE_0\cap \sfE$. By \eqref{eqn:bias_good_event_lower_bound} and \eqref{eqn:app_good_event_cap_tau}, it holds that
    \begin{align*}
        \P(\sfE_0\cap \sfE)=\P(\sfE)-\P(\sfE_0^\comp \cap \sfE)
        \geq \P(\sfE) - \P(\cI_N > c'\norm{\fd}_{\cH^s})\,.
    \end{align*}
    By hypothesis, the assertion of Lemma~\ref{lem:good_set} (i.e., that $\P(\sfE)\geq 1-e^{-c_1N}$ for some $c_1\in (0,1)$) holds true provided that $N\geq N_0$ with $N_0\geq 1$ defined in \eqref{eqn:lower_bound_N_const}. Combining this with Lemma~\ref{lem:app_series_uniform_bound} shows that, for all $N\geq N_0$, the good set satisfies
    \begin{align*}
        \P(\sfE_0\cap \sfE)&\geq 1-e^{-c_1N} - \exp\Bigl(-c_2 N^{\min\bigl(1,\frac{\al+s-1}{\al+p}\bigr)}\Bigr)\\
        &\geq 1-2\exp\Bigl(-c N^{\min\bigl(1,\frac{\al+s-1}{\al+p}\bigr)}\Bigr)
    \end{align*}
    for some $c_2\in (0,1/4)$. We defined $c\defeq \min(c_1,c_2)\in (0,1/4)$ in the last line of the preceding display.
    To complete the proof, notice that if $\sfE_0\cap \sfE$ occurs, then $\sfE$ also occurs. But the variance bound \eqref{eqn:variance_upper_bound} also holds true on $\sfE$ (as shown in the proof of Proposition~\ref{prop:var}). This proves the final assertion of the proposition.
\end{proof}

\subsubsection{Proof of Theorem~\ref{thm:linear_ee_main}}\label{appx:proofs_linear_ee_main}
Combining the bias and variance bounds leads to the main theoretical result for end-to-end learning, Theorem~\ref{thm:linear_ee_main}, which we now prove.

\thmlineareemain*

\begin{proof}
    Combining Propositions~\ref{prop:var}~and~\ref{prop:bias} shows that, for all $N\geq N_0$, the out-of-distribution test error~\eqref{eqn:test_error_conditional} satisfies the upper bound
    \begin{align}\label{eqn:error_total_no_rate}
        \sR_N\leq 2\sum_{j=1}^\infty \frac{\sigma_j'\abs{\fd_j}^2}{\bigl(1+N\gamma^{-2}\sigma_j\lambda_j\bigr)^2} +
        \Bigl(2 + c_0\norm[\big]{\fd}_{\cH^s}^2 \Bigr)\sum_{j=1}^\infty\frac{\sigma_j'\lambda_j}{1+N\gamma^{-2}\sigma_j\lambda_j}
    \end{align}
    with the asserted probability. The second term converges at the rate specified by the rightmost expression in \eqref{eqn:variance_upper_bound}. We focus on controlling the first term in the upper bound \eqref{eqn:error_total_no_rate}. Using the asymptotics of $\{\sigma_j'\}_{j\in\N}$, $\{\sigma_j\}_{j\in\N}$, and $\{\lambda_j\}_{j\in\N}$ from Assumption~\ref{ass:data_end_to_end}, there exists $j_0\in\N$ (independent of $N$ and $\fd$) such that
    \begin{align*}
        \sum_{j=1}^\infty \frac{\sigma_j'\abs{\fd_j}^2}{\bigl(1+N\gamma^{-2}\sigma_j\lambda_j\bigr)^2} &\lesssim \sum_{j\leq j_0} \frac{\sigma_j'\abs{\fd_j}^2}{\bigl(1+N\gamma^{-2}\sigma_j\lambda_j\bigr)^2}  + \sum_{j>j_0} \frac{j^{-2\al'}\abs{\fd_j}^2}{\bigl(1+Nj^{-2(\al+p)}\bigr)^2}\\
        &\lesssim \frac{\gamma^4}{N^2}\sum_{j\leq j_0}\bigl(j^{-2s}\sigma_j'\sigma_j^{-2}\lambda_j^{-2}\bigr)j^{2s}\abs{\fd_j}^2
        +
        \sum_{j=1}^\infty \frac{j^{-2\al'}\abs{\fd_j}^2}{\bigl(1+Nj^{-2(\al+p)}\bigr)^2}\\
        &\lesssim N^{-2}\norm[\big]{\fd}_{\cH^s}^2
        +
         \sum_{j=1}^\infty \frac{j^{-2\al'}\abs{\fd_j}^2}{\bigl(1+Nj^{-2(\al+p)}\bigr)^2}\,.
    \end{align*}
    To obtain the last line, we took the maximum over $j\leq j_0$ of the factor in parentheses in the first term appearing in the second line. Lemma~\ref{lem:app_series_rate1} (applied with $t=2\al'\geq -2s$, $u=2(\al+p)>1>0$, and $v=2$) shows that the rightmost series in the preceding display is bounded above by 
    \begin{align*}
        N^{-\min\bigl(2, \frac{\al'+s}{\al+p}\bigr)}\norm[\big]{\fd}_{\cH^s}^2\leq N^{-2}\norm[\big]{\fd}_{\cH^s}^2 + N^{-\bigl( \frac{\al'+s}{\al+p}\bigr)}\norm[\big]{\fd}_{\cH^s}^2\,.
    \end{align*}
    However, the $N^{-2}$ contribution is bounded above by the variance contribution in \eqref{eqn:variance_upper_bound}. Putting together the pieces by enlarging constants and bounding the sum of two nonnegative terms by twice their maximum yields \eqref{eqn:linear_ee_main} and \eqref{eqn:linear_ee_main_rate} as required.
\end{proof}

\subsection{Proofs for full-field learning of factorized linear functionals}\label{appx:proofs_linear_ff}
This appendix proves the main results from Appendix~\ref{appx:thm_full} and Section~\ref{sec:linear_main} for the \ref{item:ff} approach. We begin with a lemma that gives a convenient expression for the $L^2_{\nu'}(H;\R)$ Bochner norm of a factorized linear PtO map.

\begin{lemma}[squared Bochner norm of linear PtO map]\label{lem:error_qoi_def}
    Let $\nu'$ satisfy Assumption~\ref{item:ass_ff_data}. Let $q$ be a linear functional and $L$ be a linear operator such that $L=\sum_{j=1}^\infty l_j\varphi_j\otimes\varphi_j$ for eigenbasis $\{\varphi_j\}_{j\in\N}$ of $\Sigma'=\Cov(\nu')$ and eigenvalue sequence $\{l_j\}_{j\in\N}\subset \R$. Write $\{\sigma_j'\}_{j\in\N}$ for the eigenvalues of $\Sigma'$. If $q\circ L$ is continuous, then 
    \begin{equation}\label{eqn:error_qoi_series}
        \E^{u'\sim\nu'}\abs{q(Lu')}^2=\sum_{j=1}^\infty\sigma_j' \abs{q(\varphi_j)}^2l_j^2\,.
    \end{equation}
\end{lemma}
\begin{proof}
    Using linearity and the fact that $qL=q\circ L$ is scalar-valued, we compute
    \begin{align*}
        \E^{u'\sim\nu'}\abs{qLu'}^2=\E^{u'\sim\nu'}[(qLu') (qLu')]&=\E^{u'\sim\nu'}[(qLu') (qLu')^*]\\
        &=\E^{u'\sim\nu'}[(qL)u'\otimes u' (qL)^*]\\
        &=(qL)\Sigma' (qL)^*\\
        &=\tr[\big]{qL(\Sigma')^{1/2} (qL(\Sigma')^{1/2})^*}\,.
    \end{align*}
    The adjoint $(qL)^*\in H$ is well-defined due to the continuity of $qL$. The definition of Hilbert--Schmidt norm and the fact that $L$ and $\Sigma'$ share an eigenbasis yield
    \begin{align*}
        \tr[\big]{qL(\Sigma')^{1/2} (qL(\Sigma')^{1/2})^*}=
        \norm[\big]{qL(\Sigma')^{1/2}}^2_{\HS(H;\R)}
        &=\sum_{j=1}^\infty \abs[\big]{qL(\Sigma')^{1/2}\varphi_j}^2\\
        &=\sum_{j=1}^\infty \sigma_j' l_j^2 \abs{q(\varphi_j)}^2
    \end{align*}
    as asserted.
\end{proof}

Lemma~\ref{lem:error_qoi_def} will be used in the proofs of following two theorems. The arguments rely on \cite[Theorem 3.9, pp. 18--19]{de2023convergence}.

\thmlinearffmain*
\begin{proof}
    Write $\qd_j=\qd(\varphi_j)$ for each $j\in\N$. Lemma~\ref{lem:error_qoi_def} shows that
    \begin{align}\label{eqn:app_pto_error_series_temp}
        \E^{u'\sim{\nu'}}\abs[\big]{\qd(\Ld u') - \qd(\bar{L}^{(N)}u')}^2&=\sum_{j=1}^\infty\sigma_j' \abs{\qd_j}^2\abs[\big]{\ld_j - \bar{l}_j^{(N)}}^2
    \end{align}
    because $\Ld$ and $\bar{L}^{(N)}$ share an eigenbasis and $\qd\circ\Ld$ is continuous by hypothesis. The assumed asymptotics $\sigma_j'\lesssim j^{-2\al'}$ as $j\to\infty$ deliver an index $j_0\in\N$ such that the preceding display is bounded above by
    \begin{align*}
        &\sum_{j\leq j_0}\bigl(j^{2\al'}\sigma_j' j^{2r}\abs{\qd_j}^2\bigr) j^{-2(\al'+r)} \abs[\big]{\ld_j - \bar{l}_j^{(N)}}^2 + \sum_{j>j_0}\bigl(j^{2r}\abs{\qd_j}^2\bigr)j^{-2(\al'+r)}\abs[\big]{\ld_j - \bar{l}_j^{(N)}}^2\\
        &\qquad\qquad\leq      
        \left(\max_{k\leq j_0}k^{2\al'}\sigma_k' k^{2r}\abs{\qd_k}^2\right) \sum_{j\leq j_0}j^{-2(\al'+r)} \abs[\big]{\ld_j - \bar{l}_j^{(N)}}^2 \\
        &\qquad\qquad\qquad\qquad\qquad + \left(\sup_{k\in\N}k^{2r}\abs{\qd_k}^2\right)\sum_{j>j_0}j^{-2(\al'+r)}\abs[\big]{\ld_j - \bar{l}_j^{(N)}}^2\\
        &\qquad\qquad\lesssim \sum_{j=1}^\infty j^{-2(\al'+r)}\abs[\big]{\ld_j - \bar{l}_j^{(N)}}^2\,.
    \end{align*}
    The supremum is finite because its argument is summable due to $\qd\in\cH^r$. The asserted result follows from \cite[Theorem 3.9, pp. 18--19]{de2023convergence} by using the result for the posterior mean, choosing $\abs{\vartheta_j'}^2=j^{-2(\al'+r)}$ (i.e, replacing $\al'$ with $\al'+r$), and choosing $\delta$ to be a sufficiently small constant.
\end{proof}

The theorem for power law QoI decay has a proof similar to the previous one.
\thmlinearffmainpower*
\begin{proof}
    The proof mimics that of Theorem~\ref{thm:linear_ff_main}. After enlarging $j_0$ to accommodate the asymptotics of $\qd_j$, the only difference is that \eqref{eqn:app_pto_error_series_temp} is now bounded above by
    \begin{align*}
        &    
        \max_{k\leq j_0}k^{2\al'}\sigma_k' k^{2r+1}\abs{\qd_k}^2 \sum_{j\leq j_0}j^{-2(\al'+r+1/2)} \abs[\big]{\ld_j - \bar{l}_j^{(N)}}^2
         + \sum_{j>j_0}j^{-2(\al'+r+1/2)}\abs[\big]{\ld_j - \bar{l}_j^{(N)}}^2\\
        &\qquad\qquad\qquad\lesssim \sum_{j=1}^\infty j^{-2(\al'+r+1/2)}\abs[\big]{\ld_j - \bar{l}_j^{(N)}}^2\,.
    \end{align*}
     Replacing $\al'$ with $\al'+r+1/2$ in \cite[Theorem 3.9, pp. 18--19]{de2023convergence} completes the proof.
\end{proof}

\subsection{Proof of sample complexity comparison}\label{appx:proofs_linear_compare}
We now prove Corollary~\ref{cor:linear_compare}.

\corlinearcompare*
\begin{proof}
    First, we claim that $\fd=\qd\circ\Ld\in\cH^s$ for any $s\leq \beta+r+1/2$ under the hypotheses.
    Since $\fd_j\defeq \fd(\varphi_j)=\ld_j\qd(\varphi_j) \eqdef \ld_j\qd_j$, we compute
    \begin{align*}
        \sum_{j=1}^\infty j^{2s}\abs{\fd_j}^2= \sum_{j=1}^\infty j^{2s} \abs{\ld_j}^2\abs{\qd_j}^2\lesssim 1 + \sum_{j=1}^\infty \bigl(j^{2s-2r-1-2\beta}\bigr)j^{2\beta}\abs{\ld_j}^2\lesssim 1 + \norm{\ld}_{\cH^\beta}^2<\infty
    \end{align*}
    if $2s-2r-1-2\beta\leq 0$. This proves the claim. For the end-to-end bound, choose the maximal $s=\beta+r+1/2$ in Theorem~\ref{thm:linear_ee_opt}. With optimal $p=s+1/2$, the assumption $p>1/2$ gives $s>0$ so that $\beta+r+1/2>0$ as hypothesized. This condition also satisfies the continuity requirement \ref{item:ass_ff_pto} by a similar calculation. The condition $\al+s>1$ from Assumption~\ref{ass:truth_regularity} is the same as $\al+\beta+r>1/2$. Finally, since $\nu=\nu'$ implies $\al=\al'$, to satisfy the condition $\min(\al, \al+r+1/2) + \beta>0$ from Theorem~\ref{thm:linear_ff_main_power}, it suffices for $\al+\beta>0$ because the other case has $\al+r+1/2+\beta > \max(\al, 1)>0$. With a common set of hypotheses identified, the convergence rates may now simply be read off from Theorem~\ref{thm:linear_ee_opt} and Theorem~\ref{thm:linear_ff_main_power} after plugging in $\al'=\al$ and $s=\beta+r+1/2$. The fact that these bounds hold simultaneously on an event with the asserted probability follows by a union bound and enlarging constants.
\end{proof}

\subsection{Technical lemmas}\label{appx:proofs_linear_extra}
We conclude Appendix~\ref{appx:proofs_linear} with several supporting lemmas.
The following two technical results emphasize the non-asymptotic nature of analogous asymptotic bounds on parametrized series from \cite[Lemmas 8.1--8.2, pp. 2653--2655]{knapik2011bayesian}. The first result is useful for controlling the bias error term arising in Appendix~\ref{appx:proofs_linear_ee_bias}.
\begin{lemma}[series decay: Sobolev regularity]\label{lem:app_series_rate1}
    Let $ q\in\R $, $ t\geq -2q $, $ u>0 $, and $ v\geq 0 $. Let $N\in \N$ be arbitrary. For every $ \xi\in\cH^{q}(\N;\R) $, it holds that
    \begin{equation}\label{eqn:app_series_rate1}
    \sum_{j=1}^\infty \dfrac{j^{-t}\xi_j^2}{\left(1+Nj^{-u}\right)^{v}}
    \lesssim
    N^{-\min\left(v, \frac{t+2q}{u}\right)}\norm{\xi}^2_{\cH^q}\,.
    \end{equation}
\end{lemma}
\begin{proof}
    The asserted result may be extracted from the proof of \cite[Lemma 8.1]{knapik2011bayesian}.
\end{proof}

The next lemma is analogous to the previous one and is useful for controlling the variance error term from Appendix~\ref{appx:proofs_linear_ee_var}.
\begin{lemma}[series decay: power law regularity]\label{lem:app_series_rate2}
    Let $ t>1 $, $ u>0 $, and $ v\geq 0 $. For all $N\in\N$, it holds that
    \begin{align}\label{eqn:app_series_rate2}
        \sum_{j=1}^{\infty}\dfrac{j^{-t}}{\left(1+Nj^{-u}\right)^{v}}
            \lesssim
        \begin{cases}
        N^{-\left(\frac{t-1}{u}\right)}\, , & \text{if }\, (t-1)/u<v\,,\\
        N^{-v}\log 2N\, , &\text{if }\,(t-1)/u=v\,,\\
        N^{-v}\, , &\text{if }\,(t-1)/u>v\, .
        \end{cases}
    \end{align}
\end{lemma}
\begin{proof}
    We split the series into two parts by summing over disjoint index sets defined by the critical index $N^{1/u}$. If $j\leq N^{1/u}$, then $Nj^{-u}\leq 1+Nj^{-u}\leq 2Nj^{-u}$. Otherwise $j>N^{1/u}$ and $1\leq 1 + Nj^{-u}\leq 2$. Hence, for the bulk part of the series,
    \begin{align}\label{eqn:app_proof_series_power_bulk}
        \sum_{j\leq N^{1/u}} \dfrac{j^{-t}}{\left(1+Nj^{-u}\right)^{v}}\simeq \frac{1}{N^v}\sum_{j\leq N^{1/u}} j^{uv-t}\,.
    \end{align}
    If $(t-1)/u>v$, then $uv-t<-1$ so that the right-hand side of \eqref{eqn:app_proof_series_power_bulk} is bounded above by $N^{-v}\sum_{j=1}^\infty j^{uv-t}\lesssim N^{-v}$. Next, if $(t-1)/u=v$, then the right-hand side equals $N^{-v}\sum_{j\leq N^{1/u}} j^{-1}$. The $J$-th harmonic number satisfies $\sum_{j=1}^Jj^{-1}\leq 1+\log J $ (by an integral comparison). Since $\log(\slot)$ is increasing, $1+\log N^{1/u}\leq (\log 2)^{-1}\log (2N) + (\log 2N)/u \lesssim \log 2N$ and the second case in \eqref{eqn:app_series_rate2} follows. Finally, if $(t-1)/u<v$, then we consider the regimes $-1<uv-t< 0$ and $uv-t\geq 0$ separately. In the first regime, $h\colon x\mapsto x^{uv-t}$ is nonincreasing. Thus, if $j-1\leq x\leq j$, then $h(j)\leq h(x)$ and hence $h(j)\leq \int_{j-1}^jh(x)\dd{x}$. Summing this inequality leads to
    \begin{align*}
        \sum_{j\leq N^{1/u}}j^{uv-t}\leq \int_0^{N^{1/u}}x^{uv-t}\dd{x}=\frac{N^{v-(t-1)/u}}{1+uv-t}
    \end{align*}
    because $0<1+uv-t<1$. Multiplying by $N^{-v}$ shows that \eqref{eqn:app_proof_series_power_bulk} is bounded above by a constant times $N^{-(t-1)/u}$. In the second regime, $uv-t\geq 0$ so $h$ is now nondecreasing. A similar argument to the preceding one yields
    \begin{align*}
        \sum_{j\leq N^{1/u}}j^{uv-t}\leq \int_1^{1+N^{1/u}}x^{uv-t}\dd{x}\leq \int_0^{1+N^{1/u}}x^{uv-t}\dd{x}=\frac{(1 + N^{1/u})^{1+uv-t}}{1+uv-t}\,.
    \end{align*}
    Since $1+uv-t\geq 1$ and $N^{1/u}\geq 1$, the right-hand side is bounded above by $(2N^{1/u})^{1+uv-t}\lesssim N^{v-(t-1)/u}$. This proves that the inequality $\eqref{eqn:app_series_rate2}$ remains valid if the infinite series is replaced by the bulk partial sum \eqref{eqn:app_proof_series_power_bulk}.
    It remains to estimate the tail part of the series. By an analogous integral comparison,
    \begin{align*}
        \sum_{j> N^{1/u}} \dfrac{j^{-t}}{\left(1+Nj^{-u}\right)^{v}}\simeq \sum_{j> N^{1/u}} j^{-t}\leq \int_{\ceil{N^{1/u}}-1}^\infty x^{-t}\dd{x}\leq \int_{\max(1,N^{1/u}-1)}^\infty x^{-t}\dd{x}\,.
    \end{align*}
    The rightmost integral converges because $t>1$ and evaluates to 
    \begin{align*}
        \frac{\max(1,N^{1/u}-1)^{-(t-1)}}{t-1}\leq \frac{2^{t-1}}{t-1} N^{-\left(\frac{t-1}{u}\right)}\,.
    \end{align*}
    We used $\max(a,b)\geq (a+b)/2$ for nonnegative $a$ and $b$ to obtain the inequality. This shows that the tail series has the same upper bound $N^{-(t-1)/u}$ as the bulk sum if $(t-1)/u<v$. Otherwise, $N^{-(t-1)/u}\lesssim N^{-v}\log 2N$ if $(t-1)/u=v$ or $N^{-(t-1)/u}\lesssim N^{-v}$ if $(t-1)/u>v$. Thus, the assertion \eqref{eqn:app_series_rate2} follows.
\end{proof}

Application of the previous lemma gives an exact estimate for the \emph{effective dimension} corresponding to the prior-normalized training data covariance operator $\cC=\Lambda^{1/2}\Sigma\Lambda^{1/2}$ under the assumption of asymptotically exact power law decay of its eigenvalues. It plays a role in both the bias and variance bounds.
\begin{lemma}[effective dimension]\label{lem:app_effective_dim_trace}
    Under Assumption~\ref{ass:data_end_to_end}, it holds that
    \begin{align}\label{eqn:app_effective_dim_trace}
        \tr[\big]{\cC_\mu^{-1}\cC} \simeq \mu^{-\frac{1}{2(\al+p)}} \qfa 0<\mu\lesssim 1\,. 
    \end{align}
\end{lemma}
\begin{proof}
    Let $u\defeq 2(\al+p)$. Then $u>1$ by hypothesis.
    By the simultaneous diagonalization from Assumption~\ref{ass:data_end_to_end}, the eigenvalues of $\cC=\Lambda^{1/2}\Sigma\Lambda^{1/2}$ are $\{\sigma_j\lambda_j\}_{j\in\N}$. Then since $\sigma_j\asymp j^{-2\al}$ and $\lambda_j\asymp j^{-2p}$ as $j\to\infty$, there exists $j_0\in\N$ such that
    \begin{align*}
        \mu\tr[\big]{\cC_\mu^{-1}\cC}= \sum_{j=1}^\infty\frac{\mu\sigma_j\lambda_j}{\mu + \sigma_j\lambda_j}
        &\simeq \sum_{j\leq j_0}\frac{\sigma_j\lambda_j}{1 + \mu^{-1}\sigma_j\lambda_j} + \sum_{j>j_0}\frac{j^{-u}}{1 + \mu^{-1}j^{-u}}\\
        &\leq \mu j_0 + \sum_{j=1}^\infty\frac{j^{-u}}{1 + \mu^{-1}j^{-u}}\,.
    \end{align*}
    Since $1+\mu^{-1}\lesssim 2\mu^{-1}$ follows from the hypothesis $\mu\lesssim 1$, it holds that
    \begin{align*}
        \mu \lesssim \frac{2}{1+\mu^{-1}}\leq 2 \sum_{j=1}^\infty\frac{j^{-u}}{1 + \mu^{-1}j^{-u}}\,.
    \end{align*}
    Application of Lemma~\ref{lem:app_series_rate2} (applied in the first case with $t=u$, $u=u$, $ v=1$, and $ N=\mu^{-1}$ because $(u-1)/u<1$) shows that the series in the preceding display is bounded above by a constant times $\mu^{1-1/u} $. This implies the upper bound in \eqref{eqn:app_effective_dim_trace}.

    Now let $J_\mu\defeq \max(j_0, \mu^{-1/u})$. For the lower bound, we compute
    \begin{align*}
        \mu\tr[\big]{\cC_\mu^{-1}\cC}\geq \sum_{j>j_0}\frac{\mu\sigma_j\lambda_j}{\mu + \sigma_j\lambda_j}
        \simeq \sum_{j>j_0}\frac{j^{-u}}{1 + \mu^{-1}j^{-u}}
        \geq \sum_{j>J_\mu} \frac{j^{-u}}{1 + \mu^{-1}j^{-u}}\,.
    \end{align*}
    Since $j>J_\mu\geq \mu^{-1/u}$, it holds that $1\leq 1+\mu^{-1}j^{-u}\leq 2$. Hence, the right-hand side of the preceding display is bounded above and below by a constant times $\sum_{j>J_\mu}j^{-u}$. By comparison to an integral as in the proof of Lemma~\ref{lem:app_series_rate2}, we obtain
    \begin{align*}
        \sum_{j>J_\mu}j^{-u}\geq \int_{\ceil{J_\mu}}^\infty x^{-u}\dd{x} &\geq \frac{(J_\mu + 1)^{-(u-1)}}{u-1}\\
        & \geq \frac{(\mu^{-1/u} + j_0 + 1)^{-(u-1)}}{u-1}\\
        & \geq \frac{\bigl(\mu^{-1/u} + \mu^{-1/u}(j_0 + 1)\bigr)^{-(u-1)}}{u-1}\,.
    \end{align*}
    We used $J_\mu\leq j_0 + \mu^{-1/u}$ in the second line and $1\lesssim \mu^{-1/u}$ in the third line. Since the third line evaluates to $((j_0+2)^{-(u-1)}/(u-1))\mu^{1-1/u}$, it follows that $\tr{\cC_\mu^{-1} \cC}\gtrsim \mu^{-1/u}$ as asserted.
\end{proof}

The next lemma, whose proof requires the previous effective dimension estimate, defines a high probability event on which the operator norm of a certain normalized and centered empirical covariance is uniformly bounded in the sample size.
\begin{lemma}[good set]\label{lem:good_set}
    Let $\mu=\gamma^2/N$ and $\cC$, $\cC_\mu$, and $\widehat{\cC}$ be as in \eqref{eqn:app_defn_C_hat} and \eqref{eqn:app_defn_Ct_C_Creg}. Under Assumption~\ref{ass:data_end_to_end} and \ref{ass:data_input_kl_expand}, there exists a constant $c\in(0,1)$ (depending only on $ \nu $, $ \Lambda $, and $ \gamma^2 $) such that if
	\begin{align}\label{eqn:lower_bound_N_const}
		N\geq N_0\defeq c^{-1}\one_{\{\gamma^{-2}\tr{\cC}>c\}} + 1\,,
	\end{align}
	then the event
	\begin{equation}\label{eqn:normalized_operator_norm}
		\sfE\defeq \left\{\norm[\big]{\cC_\mu^{-1/2}(\widehat{\cC}-\cC)\cC_\mu^{-1/2}}_{\cL(H)}\leq 1/2 \right\}
	\end{equation}
	satisfies $\P(\sfE)\geq 1-e^{-cN}$.
\end{lemma}
\begin{proof}
    By \cite[Lemma 5, p. 13]{hucker2022note}, there exists $ c_0\in(0,1) $ such that if $\tr{\cC_\mu^{-1}\cC}\leq c_0 N$ and the pushforward $ (\Lambda^{1/2})_\sharp \nu $ is strongly-subgaussian, then $ \P(\sfE)\geq 1-e^{-c_0N} $. An $H$-valued random variable $Z$ (equivalently, its law) is \emph{strongly-subgaussian} if
    \begin{equation}\label{eqn:app_subg_need_indep}
        \norm[\big]{\ip{Z}{h}}_{\psi_2} = \sup_{p\geq 1}\frac{(\E\abs{\ip{Z}{h}}^p)^{1/p}}{\sqrt{p}}\lesssim \sqrt{\E \ip{Z}{h}^2} \qfa h\in H\,.
    \end{equation}
    To complete the proof, we will verify the trace condition (for sufficiently large $N$) and the subgaussian condition. For the latter, it is sufficient to show that $\nu$ itself is strongly-subgaussian because then
    the random variable $\Lambda^{1/2}u$ with $u\sim\nu$ satisfies 
    \begin{equation*}
        \norm[\big]{\ip{\Lambda^{1/2}u}{h}}_{\psi_2} = \sup_{p\geq 1}\frac{(\E\abs{\ip{u}{\Lambda^{1/2}h}}^p)^{1/p}}{\sqrt{p}}\lesssim \sqrt{\E \ip{u}{\Lambda^{1/2}h}^2}=\sqrt{\ip{h}{\cC h}}
    \end{equation*}
    for every $h\in H$ as desired. By Assumption~\ref{ass:data_input_kl_expand}, $u\sim\nu$ has the series expansion $u=\sum_j\sqrt{\sigma_j}z_j\varphi_j$ with the $\{z_j\}_{j\in\N}$ centered and independent. Fix $h\in H$ and define $h_j\defeq \ip{h}{\varphi_j}$ for each $j\in\N$. Estimating the moment generating function of $\ip{u}{h}$ with an argument similar to the one used in the proof of Lemma~\ref{lem:app_series_subexp} leads to
    \begin{align*}
        \E e^{\lambda \ip{u}{h}}=\E \prod_{j=1}^\infty e^{\lambda \sqrt{\sigma_j}z_jh_j}&=\prod_{j=1}^\infty \E e^{\lambda \sqrt{\sigma_j}z_jh_j}\\
        &\leq \prod_{j=1}^\infty e^{c'\lambda ^2\sigma_jh_j^2\norm{z_j}_{\psi_2}^2}\\
        &\leq \exp\biggl(c'm^2\lambda^2\sum_{j=1}^\infty \sigma_jh_j^2 \biggr)\\
        &=\exp\bigl(c'm^2\lambda^2 \ip{h}{\Sigma h}\bigr)
    \end{align*}
    for some absolute constant $c'>0$ \cite[p. 28]{vershynin2018high} and any $\lambda\in\R$. We used independence of the $\{z_j\}_{j\in\N}$ to obtain the second equality in the preceding display and subgaussianity to obtain the first inequality. Again by \cite[p. 28]{vershynin2018high}, $m\ip{h}{\Sigma h}^{1/2}$ is equivalent to the subgaussian norm of $\ip{u}{h}$. Thus, $\nu$ is strongly-subgaussian as required.

    Next, we show that the trace condition holds. Denote the eigenvalues of $\cC=\Lambda^{1/2}\Sigma\Lambda^{1/2}$ by $\{\lambda_j(\cC)\}_{j\in\N}$. Since $\lambda_j(\cC)\asymp j^{-2(\al+p)}$ and $2(\al+p)>1$ by Assumption~\ref{ass:data_end_to_end}, the operator $\cC$ is trace--class and
    \[
    \tr[\big]{\cC_\mu^{-1}\cC}=\sum_{j=1}^\infty\frac{\lambda_j(\cC)}{\lambda_j(\cC)+\gamma^2/N}\leq N\gamma^{-2}\sum_{j=1}^{\infty}\lambda_j(\cC) = N\gamma^{-2}\tr{\cC}\,.
    \]
    Thus, $\tr{\cC_\mu^{-1}\cC}\leq c_0N$ if $\gamma^{-2}\tr{\cC}\leq c_0$. Otherwise, application of Lemma~\ref{lem:app_effective_dim_trace} shows that there exists a constant $c_1>0$ such that $\tr{\cC_\mu^{-1}\cC}\leq c_1 N^{1/(2(\al+p))}$. Hence, $\tr{\cC_\mu^{-1}\cC}\leq c_0N$ holds if $N\geq c_2\defeq \max(1,(c_1/c_0)^{(\al+p)/(\al+p-1/2)}) $. Finally, let $c=\min(c_0,c_2^{-1})$. We conclude by showing that the constant $N_0$ in \eqref{eqn:lower_bound_N_const} suffices to verify the trace condition. If $\gamma^{-2}\tr{\cC}\leq c$, then $\gamma^{-2}\tr{\cC}\leq c_0$ and $N_0\geq 1$ suffices. Otherwise, $N_0\geq c^{-1} + 1\geq \max(c_0^{-1}, c_2) + 1\geq c_2$ as required. The fact that $1-e^{-c_0 N}\geq 1-e^{-c N}$ completes the proof.
\end{proof}

As discussed in Remark~\ref{rmk:kl_wo}, we are able to weaken the strongly-subgaussian requirement on the input training distribution $\nu$~\eqref{eqn:linear_data_center_and_cov} by only requiring its KL expansion coefficients~\eqref{eqn:data_input_kl_expand} to be pairwise uncorrelated instead of statistically independent. However, the following result shows that this improvement to Assumption~\ref{ass:data_input_kl_expand} leads to a strictly worse failure probability for the good event $\mathsf{E}$~\eqref{eqn:normalized_operator_norm}. Indeed, Lemma~\ref{lem:good_set} gives $\P(\mathsf{E}^\comp)\leq e^{-cN}$, while the next lemma yields $\P(\mathsf{E}^\comp)\leq 2e^{-cN^r}$ with $r=(\al+p-1)/(\al+p)< 1 $ being strictly smaller than one. The proof of this result closely mimics the argument in the proof of Lemma~\ref{lem:summand_bernstein_bound} and may be found in \cite[Lemma D.26, p.~264--265]{nelsen2024statistical}
\begin{lemma}[good set without independent KL coefficients]\label{lem:good_set_wo}
    Let $\mu=\gamma^2/N$ and $\cC$, $\cC_\mu$, and $\widehat{\cC}$ be as in \eqref{eqn:app_defn_C_hat} and \eqref{eqn:app_defn_Ct_C_Creg}. Let Assumption~\ref{ass:data_end_to_end} hold. Suppose that the hypotheses of Assumption~\ref{ass:data_input_kl_expand} hold, but instead of the requirement that the normalized KL expansion coefficients $\{z_j\}_{j\in\N}$~\eqref{eqn:data_input_kl_expand} of $u\sim\nu$ are independent, assume now that the $\{z_j\}_{j\in\N}$ are only pairwise uncorrelated. Then there exist constants $c\in(0,1)$ and $N_0\geq 1$ (depending only on $ \nu $, $ \Lambda $, and $ \gamma^2 $) such that if $N\geq N_0$, then the event
    \begin{equation}\label{eqn:normalized_operator_norm_wo}
        \sfE\defeq \left\{\norm[\big]{\cC_\mu^{-1/2}(\widehat{\cC}-\cC)\cC_\mu^{-1/2}}_{\cL(H)}\leq 1/2 \right\}
    \end{equation}
    satisfies $\P(\sfE)\geq 1-2\exp(-cN^{\frac{\al+p-1}{\al+p}})$, where $\al$ and $p$ are as in \ref{item:ass_ee_data} and \ref{item:ass_ee_prior}.
\end{lemma}

Following \cite[Section 5.2, p. 350]{caponnetto2007optimal} and 
\cite[Section 6.2, p. 26]{fischer2020sobolev}, we recall the following identity for regularized inverses of linear operators. This result is especially useful in conjunction with the cyclic property of the trace for deriving trace bounds.
\begin{lemma}[identity for regularized inverses]\label{lem:inverse_reg_neumann_identity}
    Let $\cH$ be a separable Hilbert space. For any $\lambda>0$ and any symmetric positive-semidefinite bounded linear operators $A\in\cL(\cH)$ and $B\in\cL(\cH)$, let $A_\lambda\defeq A+\lambda\Id_\cH$ and $B_\lambda\defeq B+\lambda\Id_\cH$. It holds that
    \begin{equation}\label{eqn:inverse_reg_neumann_identity}
        A_\lambda^{-1}=B_\lambda^{-1/2}\bigl(\Id_{\cH} - B_\lambda^{-1/2}(B-A)B_\lambda^{-1/2}\bigr)^{-1}B_\lambda^{-1/2}\,.
    \end{equation}
\end{lemma}
\begin{proof}
    Write $\Id\defeq \Id_\cH$. A direct calculation shows that
    \begin{align*}
    A+\lambda\Id &= A- B + B +\lambda\Id\\
    &=(B +\lambda\Id)^{1/2}(\Id) (B +\lambda\Id)^{1/2} - (B-A)\\
    & = (B +\lambda\Id)^{1/2}\bigl(\Id - (B +\lambda\Id)^{-1/2}(B-A)(B +\lambda\Id)^{-1/2}\bigr)(B +\lambda\Id)^{1/2}\,.
    \end{align*}
    Inverting the both sides of the equality yields the assertion.
\end{proof}


\section*{Acknowledgments}
The first author is supported by the high-performance computing platform of Peking University. The second author acknowledges support from the National Science Foundation Graduate
Research Fellowship Program under award number DGE-1745301 and from the Amazon/Caltech
AI4Science Fellowship, and partial support from the Air Force Office of Scientific Research under
MURI award number FA9550-20-1-0358 (Machine Learning and Physics-Based Modeling and
Simulation). The third author is supported by the Department of Energy Computational
Science Graduate Fellowship under award number DE-SC00211. The second and third authors are also grateful for partial support from the Department of Defense
Vannevar Bush Faculty Fellowship held by Andrew M. Stuart under Office of Naval Research award
number N00014-22-1-2790. The computations presented in this paper were partially conducted on the Resnick High Performance Computing Center, a facility supported by the Resnick Sustainability Institute at the California Institute of Technology.
The authors thank Kaushik Bhattacharya for useful discussions about learning functionals, Andrew Stuart for helpful remarks about the universal approximation theory, and Zachary Morrow for providing the code for the advection--diffusion equation solver.


\bibliographystyle{abbrv}
\bibliography{references}

\end{document}